\documentclass[twoside]{article}

\usepackage{eqnarray,amsmath}
\usepackage[utf8]{inputenc} % allow utf-8 input
\usepackage[T1]{fontenc}    % use 8-bit T1 fonts
\usepackage{booktabs}       % professional-quality tables
\usepackage{amsfonts}       % blackboard math symbols
\usepackage{nicefrac}       % compact symbols for 1/2, etc.
\usepackage{microtype}      % microtypography

\usepackage{subfigure}
\usepackage{epsf}
\usepackage{epsfig}
\usepackage{fancyhdr}
\usepackage{graphics}
\usepackage{graphicx}

\usepackage{url}% for url's in bib
\usepackage[colorlinks,linkcolor=black,citecolor=blue, pagebackref=true]{hyperref}
\renewcommand*{\backrefalt}[4]{%
    \ifcase #1 \footnotesize{(Not cited.)}%
    \or        \footnotesize{(Cited on page~#2.)}%
    \else      \footnotesize{(Cited on pages~#2.)}%
    \fi}
\usepackage{color}

\usepackage{amsthm}
\usepackage{amsmath}
\usepackage{amssymb,bbm}
\usepackage{caption}
\usepackage{algorithmic}
\usepackage{algorithm}
\usepackage{textcomp}
\usepackage{siunitx}
\usepackage{wrapfig}
\usepackage{algorithmic}
\usepackage{algorithm}
\usepackage{multirow}
\usepackage{multicol}% colors

\newtheorem{assumption}{Assumption}
\newtheorem{lemma}{Lemma}
\newtheorem{example}{Example}
\newtheorem{theorem}{Theorem}
\newtheorem{proposition}{Proposition}
\newtheorem{definition}{Definition}
\newcommand{\widgraph}[2]{\includegraphics[keepaspectratio,width=#1]{#2}}

\newcommand{\lambdan}{\lambda_n}

\newcommand{\Ocal}{\ensuremath{\mathcal{O}}}
\newcommand{\Otilde}{\widetilde{\mathcal{O}}}

\newcommand{\thetastar}{\theta^*}

\newcommand{\thetan}{\theta^n}

\newcommand{\daijn}{\Delta a^n_{ij}}
\newcommand{\dbijn}{\Delta b^n_{ij}}
\newcommand{\dsijn}{\Delta \sigma^n_{ij}}

\newcommand{\dbione}{\Delta b^n_{i1}}
\newcommand{\dsione}{\Delta \sigma^n_{i1}}

\newcommand{\ain}{a^n_i}
\newcommand{\bin}{b^n_i}
\newcommand{\sigmain}{\sigma^n_i}

\newcommand{\aj}{a^*_j}
\newcommand{\bj}{b^*_j}
\newcommand{\sj}{\sigma^*_j}

\newcommand{\apj}{a'_j}
\newcommand{\bpj}{b'_j}
\newcommand{\spj}{\sigma'_j}

\newcommand{\bGstar}{\overline{G}_*(\lambda)}

\newcommand{\Gstar}{G_*}

\newcommand{\kstar}{k_*}
\newcommand{\kzero}{k_0}

\newcommand{\kbar}{\bar{k}}

\newcommand{\lambdastar}{\lambda^*}
\newcommand{\lambdaGstar}{\lambda^*G_*}
\newcommand{\lambdaGn}{\lambda_nG_n}
\newcommand{\ones}{\mathbf{1}}

\newcommand{\parenth}[1]{\left( #1 \right)}
\newcommand{\abss}[1]{\left| #1 \right |}

\newcommand{\brj}{\Bar{r}(|\mathcal{A}_j|)}
\newcommand{\brone}{\Bar{r}(|\mathcal{A}_{1}|)}

\newcommand{\brbj}{\Bar{r}(|\mathcal{B}_j|)}
\newcommand{\brcj}{\Bar{r}(|\mathcal{C}_j|)}

\newcommand{\brcone}{\Bar{r}(|\mathcal{C}_1|)}

\newcommand{\dint}{\mathrm{d}}
\newcommand{\Prob}{\ensuremath{{\mathbb{P}}}}
\newcommand{\zerod}{\mathbf{0}_d}

\DeclareMathOperator*{\argmax}{arg\,max}

\usepackage[accepted]{aistats2024}
\usepackage[utf8]{inputenc} % allow utf-8 input
\usepackage[T1]{fontenc}    % use 8-bit T1 fonts
\usepackage{hyperref}       % hyperlinks
\usepackage{url}            % simple URL typesetting
\usepackage{booktabs}       % professional-quality tables
\usepackage{amsfonts}       % blackboard math symbols
\usepackage{nicefrac}       % compact symbols for 1/2, etc.
\usepackage{microtype}      % microtypography
\usepackage{xcolor}         % colors 
\usepackage{ragged2e}
\begin{document}

\runningtitle{On Parameter Estimation in Deviated Gaussian Mixture of Experts}


\twocolumn[

\aistatstitle{On Parameter Estimation in \\ Deviated Gaussian Mixture of Experts}

\aistatsauthor{ Huy Nguyen$^\dagger$ \And Khai Nguyen$^\dagger$ \And  Nhat Ho$^\dagger$ }

\aistatsaddress{Department of Statistics and Data Sciences, The University of Texas at Austin$^\dagger$} ]

\begin{abstract}
  We consider the parameter estimation problem in the \emph{deviated Gaussian mixture of experts} in which the data are generated from $(1 - \lambda^{\ast}) g_0(Y| X)+ \lambda^{\ast} \sum_{i = 1}^{k_{\ast}} p_{i}^{\ast} f(Y|(a_{i}^{\ast})^{\top}X+b_i^{\ast},\sigma_{i}^{\ast})$, where $X, Y$ are respectively a covariate vector and a response variable, $g_{0}(Y|X)$ is a known function, $\lambda^{\ast} \in [0, 1]$ is true but unknown mixing proportion, and $(p_{i}^{\ast}, a_{i}^{\ast}, b_{i}^{\ast}, \sigma_{i}^{\ast})$ for $1 \leq i \leq k^{\ast}$ are unknown parameters of the Gaussian mixture of experts. This problem arises from the goodness-of-fit test when we would like to test whether the data are generated from $g_{0}(Y|X)$ (null hypothesis) or they are generated from the whole mixture (alternative hypothesis). Based on the algebraic structure of the expert functions and the distinguishability between $g_0$ and the mixture part, we construct novel Voronoi-based loss functions to capture the convergence rates of maximum likelihood estimation (MLE) for our models. We further demonstrate that our proposed loss functions characterize the local convergence rates of parameter estimation more accurately than the generalized Wasserstein, a loss function being commonly used for estimating parameters in the Gaussian mixture of experts.
\end{abstract}

\section{INTRODUCTION}
Assume that $(X_1,Y_1),\ldots,(X_n,Y_n) \in \mathcal{X} \times \mathcal{Y} \subset \mathbb{R}^{d} \times \mathbb{R}$ are i.i.d samples from a joint distribution with density function $p_{\lambda^{*}, G_{*}}(X,Y):=p_{\lambda^{*}, G_{*}}(Y|X)\Bar{f}(X)$, where $\Bar{f}(X)$ is a prior density function of the explanatory variable $X$ and the conditional density function $p_{\lambda^{*}, G_{*}}(Y|X)$ is a \emph{deviated Gaussian mixture of experts} of order $k_*$, which takes the following form:
\begin{align}
    p_{\lambda^*,G_*}(Y|X)&:=(1 - \lambda^{*}) g_0(Y| X)  \nonumber\\& + \lambda^{*} \sum_{i = 1}^{k_{*}} p_{i}^{*} f(Y|(a_{i}^{*})^{\top}X+b_i, \sigma_{i}^{*}). \label{def:deviated_mixture}
\end{align}
where $f(\cdot|\mu, \sigma)$ denotes the univariate Gaussian density function with mean $\mu$ and variance $\sigma$. Here, $g_{0}(Y|X)$ is a known function and $p_{G_*}(Y|X):=\sum_{i=1}^{k_*}f(Y|(a_{i}^{*})^{\top}X+b_i, \sigma_{i}^{*})$ denotes the mixture of experts part with respect to $G_*$. Next, $\lambda^{*}\in[0,1]$ represents a true mixing proportion, whereas $G_{*}: = \sum_{i = 1}^{k_{*}} p_{i}^{*} \delta_{(a_{i}^{*}, b_{i}^{*},\sigma_{i}^{*})}$ is a true but unknown \emph{mixing measure}, that is, a linear combination of Dirac measures $\delta$ associated with positive weights $(p^*_i)_{i=1}^{k_*}$ which sum up to one, i.e., $\sum_{i=1}^{k_*}p^*_i=1$. Additionally, $(a^*_i,b^*_i,\sigma^*_i)\in\Theta\subset\mathbb{R}^d\times\mathbb{R}\times\mathbb{R}_+$, for all $1\leq i\leq k_*$, are called atoms or components of the true mixing measure $G_*$. Meanwhile, $h_1(X,a,b):=a^{\top}X+b$ and $h_2(X,\sigma):=\sigma$ are referred to as mean and variance expert functions. 

\textbf{Universal assumptions.} For the sake of theory, we assume the distribution of $X$ to be continuous so that the deviated Gaussian mixture of experts is identifiable (see Proposition~\ref{prop:identifiable}). Moreover, we also assume that the parameter space $\Theta$ is compact and the covariate space $\mathcal{X}$ is bounded in order to guarantee the convergence of parameter estimation. Finally, we let $(a^*_1,b^*_1,\sigma^*_1),\ldots,(a^*_{k_*},b^*_{k_*},\sigma^*_{k_*})$ be pairwise distinct to ensure the difference of Gaussian experts.

% Furthermore, $h_{1}:\mathcal{X}\times\Theta_1\to\mathbb{R}$ and $h_{2}:\mathcal{X}\times\Theta_2\to\mathbb{R}_+$ are expert functions where $\Theta_{i}$ are compact subsets of $\mathbb{R}^{q_{i}}$ for given $q_{1}, q_{2} \geq 1$. 

The deviated Gaussian mixture of experts~\eqref{def:deviated_mixture} arises from the goodness-of-fit test \cite{jitkrittum2020kernel,delBarrio-etal-99,hunter2008social} when the null hypothesis says that the data are generated from the known joint distribution $g_{0}(Y|X) \bar{f}(X)$ while the alternative hypothesis corresponds to the assumption that the data indeed follow the joint distribution $p_{\lambda^{*} ,G_{*}}(X,Y)$. 
Several settings of this testing problem had been considered in the literature; namely the problem of detection of sparse homogeneous mixtures~\cite{Jin_homogeneous, Jin_confidence, Wu_detection, Verzelen_feature}, the problem of testing the number of components~\cite{Chen_modified, Kasahara_nonparametric, Kasahara_number_components}, and multiple testing problems~\cite{Patra_estimation, Deb_two_component}. Moreover, the deviated Gaussian mixture of experts is also a generalization of the Gaussian mixture of experts~\cite{Jacob_Jordan-1991, Jordan-1994, Xu_Jordan-1995}, which have been used in various fields, namely speech recognition~\cite{Jacobs-1996,50136,You_Speech_MoE_2}, multi-task learning~\cite{liang_m3vit_2022,ma_modeling_2018,hazimeh2021dselect}, computer vision~\cite{puigcerver_scalable_2021,Lathuiliere2017,xia2022image}, and natural language processing~\cite{Eigen_learning_2014, Quoc-conf-2017, Ashok_breaking_icml, Ashok_learning_nips,pmlr-v162-du22c,zuo2023moebert}.

\textbf{Maximum likelihood estimation.} An important application of the deviated Gaussian mixture of experts to the hypothesis testing problem is parameter estimation, namely, the problem of estimating unknown mixing proportion $\lambda^{*}$ and mixing measure $G_{*}$. It is worth noting that the number of experts $k_*$ is also unknown in practice. Therefore, we fit the true model~\eqref{def:deviated_mixture} with a deviated Gaussian mixture of $k$ experts, where $k>k_*$, and then use the maximum likelihood estimation (MLE) method to find the estimates of $\lambdastar$ and $G_*$ as follows:
\begin{align}
\label{eq:MLE_estimation}
(\hat{\lambda}_n,\widehat{G}_n)\in\argmax_{(\lambda,G)\in[0,1]\times\Ocal_{k}(\Theta)}\sum_{i=1}^{n}\log(p_{\lambda, G}(Y_i|X_i)).
\end{align}
Here, we denote $\Ocal_{k}(\Theta):=\{G = \sum_{i=1}^{k'}p_i\delta_{(a_{i},b_{i},\sigma_{i})}:1\leq k'\leq k,\  (a_{i},b_{i},\sigma_{i})\in\Theta\}$ as the set of all discrete probability measures with at most $k$ components.
% \begin{align*}
%     p_{\lambda, G}(Y|X) &:= (1 - \lambda) g_0(Y| X) \nonumber\\&\quad + \lambda \sum_{i = 1}^{k} p_{i} f(Y|a_i^{\top}X+b_i,\sigma_i),
% \end{align*}
% for any $\lambda\in[0,1]$ and $G = \sum_{i=1}^{k}p_i\delta_{(a_{i},b_{i},\sigma_{i})}\in\Ocal_{k}(\Theta)$. 

\textbf{Challenge discussion.} When $\lambda^{*} = 1$ is known, the conditional density $p_{\lambdastar,G_*}(Y|X)$ reduces to the mixture part $p_{G_*}(Y|X)$. Thus, the problem of estimating $\widehat{G}_{n}$ becomes a parameter estimation problem in Gaussian mixture of experts, which had been studied in Theorem 2 of \cite{ho2022gaussian}. Ho et al. \cite{ho2022gaussian} demonstrated that the convergence rates of parameter estimation in the Gaussian mixture of experts were determined by the solvability of a system of polynomial equations induced the algebraic structures between the expert functions. These convergence rates ranged from order $\Otilde(n^{-1/4})$ to order $\Otilde(n^{-1/2r})$ for some $r\geq 4$.

However, when $\lambda^{*} \in [0,1]$ is unknown, the theoretical understanding of the MLE $(\hat{\lambda}_n,\widehat{G}_n)$ in the deviated Gaussian mixture of experts becomes more challenging than those in the standard Gaussian mixture of experts. The main challenge comes from the interaction between the known function $g_{0}(Y|X)$ and the mixture part $p_{G_*}(Y|X)$ with respect to the mixing measure $G_*$ via some partial differential equations (PDEs). This interaction influences not only the identifiability of the model but also the convergence rate of the MLE.

Another issue comes from the suboptimality of the generalized Wasserstein loss function \cite{Villani-03,Villani-09} used in learning parameters. The idea of leveraging that loss function in analyzing the convergence behavior of MLE in mixture models was initialized by \cite{nguyen2016latentmixing}, and then extended to mixture of experts by \cite{ho2022gaussian}. An important property of this divergence is that the convergence of the MLE $\widehat{G}_n$ is able to imply those of its atoms. For example, it can be seen from Theorem 1 in \cite{ho2022gaussian} that the convergence rate $\Otilde(n^{-1/4})$ of $\widehat{G}_n$ to $G_*$ under the generalized Wasserstein indicates that the rates of estimating individual components are also $\Otilde(n^{-1/4})$. On the other hand, the generalized Wasserstein are unable to capture those rates accurately. In particular, while the estimation rates for those components should vary with the number of fitted components approximating them, that loss function always leads to the same rates. 

%For that reason, our work aims to propose novel loss functions that are able to precisely capture different convergence rates of the parameter estimation in the deviated Gaussian mixture of experts. 
\begin{table*}[!ht]
\centering
\begin{tabular}{ | c | c| m{3em} | m{4em} |c|} 
\hline
\multirow{2}{4em}{\textbf{Setting}} & \textbf{Bound of $k$} & \hspace{.2em} \textbf{Loss} &\textbf{Exact-fitted $a^*_j,b^*_j,\sigma^*_j$} & \textbf{Over-fitted $a^*_j,b^*_j,\sigma^*_j$ }\\
\hline 
\scriptsize Distinguishable & $k\geq k_*$ &\hspace{0.5em} $D_1$ &\hspace{.5em} $n^{-1/2}$ & $n^{-1/4},n^{-1/2\bar{r}({|\mathcal{A}_j|)}},n^{-1/\bar{r}({|\mathcal{A}_j|)}}$ \\
\hline
\multirow{2}{4em}{\hspace{-1em} \scriptsize Non-distinguishable} & $k\geq k_*+k_0-\kbar$, $\hat{\lambda}_n>\lambdastar$ & \multirow{2}{4em}{\hspace{.5em} $D_2$} & \multirow{2}{4em}{\hspace{.5em} $n^{-1/2}$} & $n^{-1/4},n^{-1/2\bar{r}({|\mathcal{B}_j|}}),n^{-1/\bar{r}({|\mathcal{B}_j|})}$\\\cline{2-2}\cline{5-5}
& Otherwise & &  & $n^{-1/4},n^{-1/2\bar{r}({|\mathcal{A}_j|})},n^{-1/\bar{r}({|\mathcal{A}_j|})}$\\
\hline
\scriptsize Theorem 2 \cite{ho2022gaussian} & $k\geq k_*$ & \hspace{0.9em}$\widetilde{W}$ & \hspace{.5em} $n^{-1/4}$ &$n^{-1/4},n^{-1/2\bar{r}(k-k_*+1)},n^{-1/\bar{r}(k-k_*+1)}$\\
\hline
\end{tabular}
\caption{Summary of parameter estimation rates in the (deviated) Gaussian mixture of experts. Here, exact-fitted parameters are those approximated by one fitted component, while their over-fitted counterparts are approached by at least two fitted components. Additionally, the value of function $\bar{r}(\cdot)\geq 4$ is determined by the solvability of the system of polynomial equations \eqref{definition:polynomial_equation}. Meanwhile, the cardinalities of Voronoi cells $\mathcal{A}_j$ and $\mathcal{B}_j$, which are respectively defined in equations~\eqref{eq:Voronoi_cells_A} and \eqref{eq:Voronoi_cells_B}, indicate the number of components fitting parameters $a^*_j,b^*_j,\sigma^*_j$. Lastly, the notation $\widetilde{W}$ stands for the generalized Wasserstein loss function used in \cite{ho2022gaussian}.}
\label{table:parameter_rates}
\end{table*}

\textbf{Contribution.} In the paper, we first establish the parametric convergence rate of density estimation $p_{\hat{\lambda}_n,\widehat{G}_n}$ to the true density $p_{\lambdastar,G_*}$ of order $\Otilde(n^{-1/2})$ under the Total Variation distance $V$. Next, to address the above challenges of the parameter estimation problem in the deviated Gaussian mixture of experts, we first develop a distinguishability condition between the function $g_{0}(Y|X)$ and the mixture part $p_{G_*}(Y|X)$ in the deviated Gaussian mixture of experts in order to isolate the effect of function $g_0$ on the convergence behaviors of parameter estimation of $p_{G_*}$. Then, we conduct the convergence analysis of parameter estimation under \emph{distinguishable settings}, namely when the distinguishability condition holds true, and \emph{non-distinguishable settings}, i.e. when that condition does not hold. In each scenario, we construct a novel Voronoi loss function to precisely capture distinct convergence rates of parameter estimation in the deviated Gaussian mixture of experts (see also Table~\ref{table:parameter_rates}). 
Our theory can be summarized as follows: 

\textbf{1. Distinguishable settings:} When the distinguishability condition holds, there is no impact of the function $g_0$ on the mixture of experts part $p_{G_*}$. Therefore, we design a novel Voronoi loss function $D_1((\hat{\lambda}_n,\widehat{G}_n),(\lambdastar,G_*))$ in equation~\eqref{eq:D_distinguishable_dependent}, and then demonstrate that it is lower bounded by the Total Variation distance $V(p_{\hat{\lambda}_n,\widehat{G}_n},p_{\lambdastar,G_*})$, which vanishes at the rate of order $\Otilde(n^{-1/2})$. It follows from this bound that the
estimation rate for $(\lambdastar,G_*)$ is of order $\Otilde(n^{-1/2})$. Moreover, parameters $(a^*_i,b^*_i,\sigma^*_i)$ which are fitted by exactly one component enjoy the same estimation rate of order $\Otilde(n^{-1/2})$. By contrast, if $(a^*_i,b^*_i,\sigma^*_i)$ are approached by more than one component, then the rates for estimating $a^*_i$ become slower at $\Otilde(n^{-1/4})$, while those for $b^*_i$ and $\sigma^*_i$ are of orders $\Otilde(n^{-1/2\bar{r}_i})$ and $\Otilde(n^{-1/\bar{r}_i})$, respectively, where $\bar{r}_i\geq 4$ is determined by the solvability of the system of polynomial equations defined in equation~\eqref{definition:polynomial_equation}.

\textbf{2. Non-distinguishable settings:} When the distinguishability condition fails, we consider the function $g_0$ as a Gaussian mixture of $k_0$ experts, where $1\leq k_0\leq k_*$, whose parameters interact with those of the mixture part $p_{G_*}$. Notably, the convergence behaviors of parameter estimation in the deviated Gaussian mixture of experts strictly depend on the interaction level determined by the number of overlapped components $\kbar$ that $g_0$ and $p_{G_*}$ share. Therefore, we propose a novel Voronoi loss function $D_2((\hat{\lambda}_n,\widehat{G}_n),(\lambdastar,G_*))$ in equation~\eqref{eq:D_partial_dependent} to capture this property, and then derive the Total Variation lower bound $D_2((\hat{\lambda}_n,\widehat{G}_n),(\lambdastar,G_*))\lesssim V(p_{\hat{\lambda}_n,\widehat{G}_n},p_{\lambdastar,G_*})=\Otilde(n^{-1/2})$. Consequently, the rates for estimating $(\lambdastar,G_*)$ and exact-fitted parameters $a^*_j,b^*_j,\sigma^*_j$ are of order $\Otilde(n^{-1/2})$. On the other hand, for over-fitted parameters $a^*_j,b^*_j,\sigma^*_j$, while the estimation rate for $a^*_j$ remains unchanged of order $\Otilde(n^{-1/4})$, those for $b^*_j,\sigma^*_j$ not only depend on the solvability of the system~\eqref{definition:polynomial_equation} but also vary with the relation between $k$ and $k_*+k_0-\kbar$.

% Therefore, we take into account two different regimes which we refer to as (i) \emph{partial overlap} ($\kbar<k_0$) and (ii) \emph{full overlap} ($\kbar=k_0$). Then, we respectively solve the identifiability equation under these regimes and design Voronoi loss functions $D_2$ and $D_3$ in equations~\eqref{eq:D_partial_dependent} and \eqref{eq:D_full_dependent} based on the solutions of that equation for the sake of convergence analysis of the MLE.
%conduct the convergence analysis of the MLE based on them as in the distinguishable setting.

\textbf{Organization.} 
The paper is organized as follows. Firstly, we introduce a novel distinguishability condition and a notion of Voronoi cells as well as establish the density estimation rate in Section~\ref{sec:problem_setup}. In Section~\ref{sec:distinguishable}, we analyze the convergence behavior of parameter estimation under both the distinguishable and non-distinguishable settings and provide a proof sketch for main results. 
%Next, under the non-distinguishable settings, we conduct similar convergence analysis for the partial overlap regime and the full overlap regime in Section~\ref{sec:partial_overlap} and Section~\ref{appendix:full_overlap}, respectively. 
Then, we conduct a simulation study in Section~\ref{sec:experiments} to empirically verify our theoretical results before concluding the paper in Section~\ref{sec:conclusion}. Additional results and detailed proofs are all deferred to the supplementary material.

\textbf{Notation.} Let $[n]$ stand for the set $\{1,2,\ldots,n\}$ for any $n\in\mathbb{N}$. Given two sequences $\{s_n\}$ and $\{t_n\}$, we write $s_n \lesssim t_n$ or $s_n=\Ocal(t_n)$ if there exists a constant $C>0$ independent of $n$ such that $s_n \leq C t_n$ for all $n\in\mathbb{N}$ (similar for $s_n \gtrsim t_n$). Next, the notation $s_n=\Otilde(t_n)$ indicates that the previous bound occurs up to some logarithmic factor. Let $\|{\cdot}\|_p$ represents for the usual $p$-norm in $\mathbb{R}^d$ with a convention that $\|\cdot\|$ being the $2$-norm. 
%Additionally, for any set $A$, we denote $|A|$ as its cardinality. 
Finally, for any two probability density functions $p$ and $q$ (with respect to the Lebesgue measure $\mu$), we define the Total Variation distance between them as $V(p,q):=\frac{1}{2}\int|p-q|d\mu$, while the Hellinger distance is given by $h(p,q):=\Big(\frac{1}{2}\int(\sqrt{p}-\sqrt{q})^2d\mu\Big)^{1/2}$. 
\vspace{-0.3 em}
\section{BACKGROUND}
\label{sec:problem_setup}
\vspace{-0.3 em}
In this section, we first introduce a distinguishability condition, and then validate the identifiability of the deviated Gaussian mixture of experts as well as characterize the density estimation rate. 
%Lastly, we provide the definition of a Voronoi cell which is essential to construct Voronoi loss functions in Section~\ref{sec:parameter_rates}.

%%%%%%%%%%%%%%%%%%%%%%%%%%%%

% For ease of the ensuing presentation, we denote $h_1(X,a,b):=a^{\top}X+b$ and $h_2(X,\sigma):=\sigma$ as the mean and variance expert functions used in the paper.
\vspace{-0.3 em}
Recall that $h_1$ and $h_2$ are mean and variance expert functions in the true model~\eqref{def:deviated_mixture}. Then, we begin this section with the following distinguishability condition between the function $g_0$ and the mixture part $p_{G_*}$:
\begin{definition}[Distinguishability Condition]
\label{definition:distinguishability}
We say that $p_{G_*}$ is distinguishable from $g_0$ with respect to vector $r=(r_1,\ldots,r_{k_*})\in\mathbb{N}^{k_*}$ if the following holds: assume that there exist real coefficients $\alpha^{(0)}$ and $\alpha^{(i)}_{\ell}$, for $i\in[k_*]$ and $0\leq \ell\leq r_i$ that satisfy
\begin{align*}
\sum_{i=1}^{k_*}\sum_{\ell=0}^{r_i} \alpha^{(i)}_{\ell}\cdot\dfrac{\partial^{\ell}f}{\partial h_1^{\ell}}(Y|(a^*_i)^{\top}X+b^*_i,&~\sigma^*_i)\\
&+\alpha^{(0)}g_0(Y|X)=0,
\end{align*}
for almost surely $(X,Y)$, then $\alpha^{(0)}=\alpha^{(i)}_{\ell}=0$ for any $i\in[k_*]$ and $0\leq \ell\leq r_i$.
\end{definition}
\vspace{-.6em}
For better understanding, we provide below a scenario when $p_{G_*}$ is distinguishable from $g_0$.
\begin{example}
\label{eq:example}
Let $G_0=\sum_{i=1}^{k_0}p^0_i\delta_{(\theta^0_{1i},\theta^0_{2i})}\in\mathcal{E}_{k_0}(\Theta):=\Ocal_{k_0}(\Theta)\setminus\Ocal_{k_0-1}(\Theta)$, where $k_0\in\mathbb{N}$. If we set
\begin{align}
    \label{eq:g0_form}
    g_0(Y|X)&=p_{G_0}(Y|X) \nonumber\\:&=\sum_{i=1}^{k_0}p^0_if(Y|(a^0_i)^{\top}X+b^0_i,\sigma^0_i),
\end{align}
then $p_{G_*}$ is distinguishable from $g_0$ whenever $k_0>\kstar$.
\end{example}
\vspace{-.5em}
In high level, the purpose of the distinguishability condition is to control the interaction level between the function $g_0$ and the mixture part $p_{G_*}$. From the perspective of the parameter estimation problem, if there is no effect of the function $g_{0}$ on the mixture $p_{G_{\ast}}$, then the convergence behaviors of parameter estimation in the deviated Gaussian mixture of experts will be similar to those in the standard Gaussian mixture of experts previously studied in \cite{ho2022gaussian}. On the other hand, when the distinguishability condition fails, i.e., there are interactions between $g_{0}$ and $p_{G_{\ast}}$, the parameter estimation rates will strictly depend on the interaction level among these two functions. In our paper, we illustrate that point by considering $g_{0}$ as a Gaussian mixture of $k_0$ expert given in equation~\eqref{eq:g0_form}, where $1\leq k_{0}\leq \kstar$. This choice of function $g_0$ allows us to determine the level of the interaction between $g_{0}$ and $p_{G_{\ast}}$ explicitly via the number of overlapped components that these functions share, which will be discussed further in Section~\ref{sec:non-distinguishable}.

Subsequently, we figure out in the following proposition that if the distinguishability condition is satisfied, then the deviated Gaussian mixture of experts in equation~\eqref{def:deviated_mixture} is identifiable.
\begin{proposition}[Identifiability]
    \label{prop:identifiable}
    Let $G,G'$ be two mixing measures in $\Ocal_{k}(\Theta)$ and $\lambda,\lambda'$ be two mixing proportions in $[0,1]$. Assume that $p_{G_*}$ is distinguishable from $g_0$, then if the identifiability equation $p_{\lambda, G}(X,Y)=p_{\lambda', G'}(X,Y)$ holds for almost surely $(X,Y)\in\mathcal{X}\times\mathcal{Y}$, then we achieve that $(\lambda, G)\equiv(\lambda', G')$.
\end{proposition}
Proof of Proposition~\ref{prop:identifiable} is in Appendix~\ref{appendix:identifiable}. Given that $p_{G_*}$ is distinguishable from $h_0$, this result ensures the convergence of the MLE $(\hat{\lambda}_n,\widehat{G}_n)$ to the true pair of mixing proportion and mixing measure $(\lambdastar,G_*)$ when the density estimation $p_{\hat{\lambda}_n,\widehat{G}_n}(X,Y)$ converges to the true density $p_{\lambdastar,G_*}(X,Y)$ for almost surely $(X,Y)$. Thus, it is natural to explore the density estimation rate in the following proposition:

\begin{proposition}[Density estimation rate]
\label{prop:mle_estimation}
Suppose that the function $g_0$ is bounded with tail $\mathbb{E}_{X}[-\log g_0(Y|X)]\gtrsim Y^q$ for almost surely $Y\in\mathcal{Y}$ for some $q>0$. Then, the following inequality holds true:
\begin{align*}
    \mathbb{P}\Big(V(p_{\widehat{\lambda}_n, \widehat{G}_n},p_{\lambda^*, G_*})>C\sqrt{\log (n)/n}\Big)\lesssim n^{-c},
\end{align*}
where $C>0$ is a constant that depends on $g_0$, $\lambdastar$, $G_*$ and $\Theta$, while the constant $c>0$ depends only on $\Theta$.
\end{proposition}
Proof of Proposition~\ref{prop:mle_estimation} is in Appendix~\ref{appendix:prop:mle_estimation}. The above bound indicates that the density estimation $p_{\widehat{\lambda}_n, \widehat{G}_n}$ converges to the true density $p_{\lambda^*, G_*}$ under the Total Variation distance at the parametric rate of order $\Otilde(n^{-1/2})$. In order to leverage this result, we assume that the function $g_0$ is bounded with tail $\mathbb{E}_{X}[-\log g_0(Y|X)]\gtrsim Y^q$ for almost surely $Y\in\mathcal{Y}$ for some $q>0$ throughout the paper unless stating otherwise.

% \textbf{Voronoi cells.} In order to construct such loss functions, we define for each true component $\theta^*_j:=(\aj,\bj,\sj)$ an index set known as a Voronoi cell to quantify the number of fitted components converging to that true component. In particular, given a mixing measure $G=\sum_{i=1}^{k}p_i\delta_{(a_i,b_i,\sigma_i)}$ in $\Ocal_{k}(\Theta)$, its corresponding Voronoi cells are defined as follows:
% \begin{align}
%     \label{eq:Voronoi_cells_A}
%     \mathcal{A}_j &\equiv\mathcal{A}_j(G) \nonumber \\&: =\{~i\in[k]:\|\theta_i-\theta^*_j\|\leq \|\theta_i-\theta^0_{\ell}\|, \ \forall \ell\neq j\},
% \end{align}
% where $\theta:=(a_i,b_i,\sigma_i)$ for any $j\in[k_*]$. An instance of Voronoi cells is illustrated in Figure~\ref{fig:Voronoi_cells}. 
\begin{figure*}[!ht]
    \centering
    \includegraphics[scale =.08]{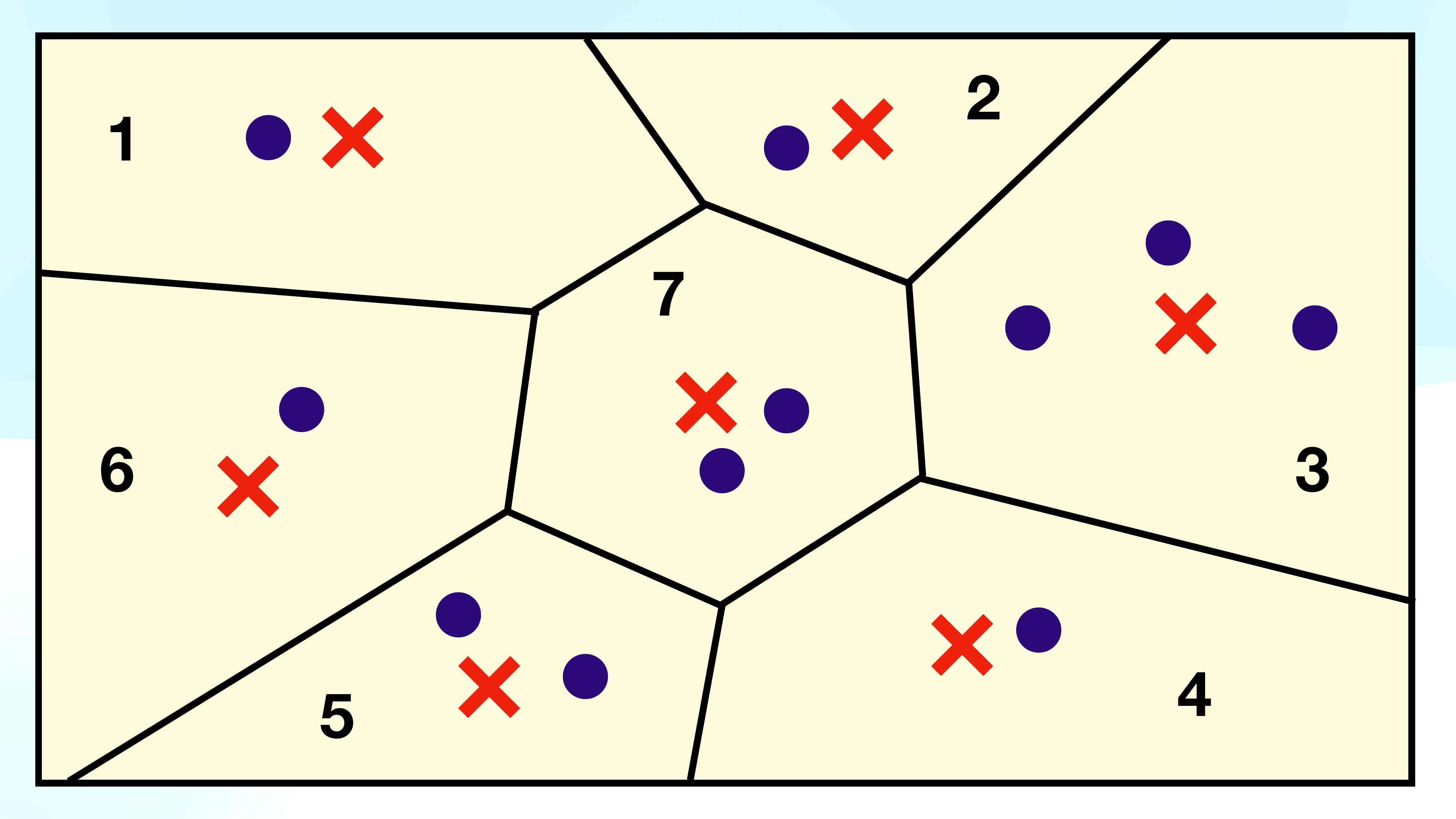}
    \caption{Illustration of the Voronoi cells generated by the components of $G_{\ast}$ (red crosses) and the fitted components of the MLE $\widehat{G}_n$ (blue points). 
    %As mentioned in Section~\ref{sec:problem_setup}, the cardinality of each Voronoi cell is exactly the number of fitted components that approximate the true component generating that cell. 
    Under the distinguishable settings, Theorem~\ref{theorem:distinguishable_dependent} indicates that the rates for estimating true components $(\aj,\bj,\sj)$ in cells $1,2,4,6$, which are fitted by one component, are of order $\Otilde(n^{-1/2})$. Meanwhile, those for true components $(a^*_3,b^*_3,\sigma^*_3)$ in cell 3, which are fitted by three components, are drastically slow at $\Otilde(n^{-1/4})$, $\Otilde(n^{-1/12})$ and $\Otilde(n^{-1/6})$, respectively.}	
    \label{fig:Voronoi_cells}
    \vspace{-1 em}
\end{figure*}

%%%%%%%%%%%%%%%%%%%%%%%%%%%%%%%%%%%%

%\vspace{-0.6 em}
\section{CONVERGENCE RATES OF PARAMETER ESTIMATION}
%\vspace{-0.3 em}
\label{sec:parameter_rates}
In this section, we aim to establish the convergence rates of maximum likelihood estimation in the deviated Gaussian mixture of experts under both the distinguishable and non-distinguishable settings.
%\vspace{-0.6 em}
\subsection{Distinguishable Settings}
%\vspace{-0.3 em}
\label{sec:distinguishable}
Under this setting, the mixture part $p_{G_*}$ is distinguishable from the function $g_0$ w.r.t vector $r=(r_1,\ldots,r_{k_*})$ that we will choose later. In other words, there is no interaction between $p_{G_*}$ and $g_0$, and the following set is linearly independent for almost surely $X$:
\begin{align}
    \label{eq:independent_set}
    &\left\{\dfrac{\partial^{\ell}f}{\partial h_1^{\ell}}(Y|(a^*_j)^{\top}X+b^*_j,\sigma^*_j), \ g_0(Y|X):\right. \nonumber \\ 
    &\hspace{4cm}\left. j\in[k_*], \ 0\leq \ell\leq r_j\right\}.
\end{align}
Given the parametric density estimation rate $V(p_{\hat{\lambda}_n,\widehat{G}_n},p_{\lambdastar,G_*})=\Otilde(n^{-1/2})$ in Proposition~\ref{prop:mle_estimation}, our main goal is to establish the Total Variation lower bound $V(p_{\hat{\lambda}_n,\widehat{G}_n},p_{\lambdastar,G_*})\gtrsim D_1((\hat{\lambda}_n,\widehat{G}_n),(\lambdastar,G_*))$, where $D_1$ will be defined in equation~\eqref{eq:D_distinguishable_dependent}, in order to achieve the parametric convergence rate of the MLE $D_1((\hat{\lambda}_n,\widehat{G}_n),(\lambdastar,G_*))=\Otilde(n^{-1/2})$. For that purpose, we first rewrite the density discrepancy $p_{\hat{\lambda}_n,\widehat{G}_n}(X,Y)-p_{\lambdastar,G_*}(X,Y)$ in terms of $f(Y|(\hat{a}^n_i)^{\top}X+\hat{b}^n_i,\hat{\sigma}^n_i)-f(Y|(a^*_j)^{\top}X+b^*_j,\sigma^*_j)$, where $(\hat{a}^n_i,\hat{b}^n_i,\hat{\sigma}^n_i)$ is a component of $\widehat{G}_n$. Next, we apply a Taylor expansion to the function $f(Y|(\hat{a}^n_i)^{\top}X+\hat{b}^n_i,\hat{\sigma}^n_i)$ about the point $(a^*_j,b^*_j,\sigma^*_j)$ to decompose the density discrepancy into a linear combination of linearly independent elements associated with coefficients involving the parameter discrepancies, namely $\hat{a}^n_i-a^*_j$, $\hat{b}^n_i-b^*_j$ and $\hat{\sigma}^n_i-\sigma^*_j$. As a result, when $p_{\hat{\lambda}_n,\widehat{G}_n}$ converges to $p_{\lambdastar,G_*}$, the previous parameter discrepancies also go to zero and we achieved our desired estimation rates. Note that such decomposition cannot be done if the set in equation~\eqref{eq:independent_set} is linearly dependent. However, we observe an interaction among parameters of the Gaussian density $f$ via the following partial differential equation (PDE):
\begin{align}
    \label{eq:PDE_interaction}
    \dfrac{\partial^2f}{\partial b^2}=2~\dfrac{\partial f}{\partial \sigma}.
\end{align}
This interaction induces a lot of linearly dependent derivative terms in the previous Taylor expansion. Thus, we have to group these terms together by adding their coefficients. Consequently, when the resulting coefficients tend to zero, we arrive at a system of polynomial equations which was previously studied in \cite{ho2016weakly}. 

\textbf{System of polynomial equations.} Let $\Bar{r}(m)$ be the smallest natural number $r$ such that the following system of polynomial equations does not admit any non-trivial solutions for the unknown variables: $(s_l,t_{1l},t_{2l})_{l=1}^{m}\subseteq\mathbb{R}^3$
\begin{align}
    \label{definition:polynomial_equation}
    \sum_{l=1}^{m}\sum_{\substack{n_1,n_2\in\mathbb{N}\\ n_1+2n_2=\beta}}\dfrac{s_l^2~t_{1l}^{n_1}~t_{2l}^{n_2}}{n_1!n_2!}=0, \quad \beta=1,2,\ldots,r,
\end{align}
A solution to the above system is regarded as non-trivial if all variables $s_{l}$ are non-zero, whereas at least one of the $t_{1l}$ is different from zero. As shown in [Proposition 2.1, \cite{ho2016weakly}], we have $\Bar{r}(2)=4$, $\Bar{r}(3)=6$ and $\Bar{r}(m)\geq 7$ when $m\geq 4$. 

\textbf{Voronoi loss function:} 
Intuitively, true parameters $a^*_j,b^*_j,\sigma^*_j$ which are fitted by one component should admit faster estimation rates than those approximated by more than one component. To capture this convergence behavior of parameter estimation, let us introduce a class of Voronoi cells $\mathcal{A}_j\equiv\mathcal{A}_j(G)$ w.r.t an arbitrary mixing measure $G$, which are generated by the components $\theta^*_j:=(a^*_j,b^*_j,\sigma^*_j)$ of $G_*$ as follows: 
\begin{align}
    \label{eq:Voronoi_cells_A}
    \mathcal{A}_j: =\{i\in[k]:\|\theta_i-\theta^*_j\|\leq \|\theta_i-\theta^*_{\ell}\|, \ \forall \ell\neq j\},
\end{align}
where $\theta_i:=(a_i,b_i,\sigma_i)$ for any $i\in[k]$. Notably, the cardinality of Voronoi cell $\mathcal{A}_j$ is exactly the number of components fitting $\theta^*_j$. An instance of Voronoi cells is illustrated in Figure~\ref{fig:Voronoi_cells}. Based on those cells, the Voronoi loss function used for this setting is defined as
\begin{align}   
    \label{eq:D_distinguishable_dependent}
    &D_1((\lambda, G),(\lambdastar, G_*)):=|\lambda-\lambdastar|+(\lambda+\lambdastar)\nonumber\\
    &\times  \Big[\sum_{j:|\mathcal{A}_j|=1}\sum_{i\in\mathcal{A}_j}p_i\Big(\|\Delta a_{ij}\|+|\Delta b_{ij}|+| \Delta \sigma_{ij}|\Big)\nonumber \\
    &+\sum_{j:|\mathcal{A}_j|>1}\sum_{i\in\mathcal{A}_j}p_i\Big(\|\Delta a_{ij}\|^2+|\Delta b_{ij}|^{\brj}\nonumber  \\
    &\quad +|\Delta \sigma_{ij}|^{\brj/2}\Big)+\sum_{j=1}^{k_*}\big|\sum_{i\in\mathcal{A}_j}\lambda p_i-\lambdastar p^*_j\big|\Big].
\end{align}
where $\Delta a_{ij}:=a_i-\aj$, $\Delta b_{ij}:=b_i-\bj$ and $\Delta\sigma_{ij}:=\sigma_i-\sj$. It is obvious that $D_1((\lambda, G),(\lambdastar, G_*))=0$ if and only if $(\lambda, G)=(\lambdastar, G_*)$. Now, we derive the convergence rate of the MLE $(\hat{\lambda}_n,\widehat{G}_n)$ under the challenging scenario when $\lambdastar\in(0,1]$ in Theorem~\ref{theorem:distinguishable_dependent}, while a discussion on the scenario when $\lambdastar=0$ is relegated to Appendix~\ref{appendix:additional_results} due to the space limit.
\begin{theorem}
\label{theorem:distinguishable_dependent}
Assume that the distinguishability condition holds and $\lambdastar\in(0,1]$ is unknown. Then, we achieve the Total Variation lower bound $V(p_{\lambda,G},p_{\lambdastar,G_*})\gtrsim D_1((\lambda,G),(\lambdastar,G_*))$ for any $(\lambda,G)\in[0,1]\times\Ocal_k(\Theta)$. This bound together with Proposition~\ref{prop:mle_estimation} imply that 
\begin{align*}
    \mathbb{P}\Big(D_1((\hat{\lambda}_n,\widehat{G}_n),(\lambdastar, G_*)) &>C_{1}\sqrt{\log (n)/n}\Big) \lesssim n^{-c_1},
\end{align*}
where $C_{1}>0$ is a constant depending on $g_0,\lambdastar,G_*,\Theta$, while the constant $c_{1}>0$ depends only on $\Theta$.
\end{theorem}
Proof of Theorem~\ref{theorem:distinguishable_dependent} is in Appendix~\ref{appendix:distinguishable_dependent}. 
%When $\lambdastar=0$, the MLE $\hat{\lambda}_n$ converges to $\lambdastar$ at a rate of order $\Otilde(n^{-1/2})$. 
When $\lambdastar\in(0,1]$, the result that $D_1((\hat{\lambda}_n,\widehat{G}_n),(\lambdastar, G_*))$ vanishes at a rate of order $\Otilde(n^{-1/2})$ implies the following observations (which are illustrated in Figure~\ref{fig:Voronoi_cells} as well): 

\textbf{(i)} Firstly, for any $j\in[k_*]$ such that $|\mathcal{A}^n_j|=1$, where $\mathcal{A}^n_j=\mathcal{A}_j(\widehat{G}_n)$, it follows that all the true parameters $\aj,\bj,\sj$, which are fitted by a single component, share the same parametric rate of order $\Otilde(n^{-1/2})$. On the other hand, [Theorem 2, \cite{ho2022gaussian}], which used the generalized Wasserstein as a loss function, indicated that the rates for estimating those parameters were of orders $\Otilde(n^{-1/4})$, $\Otilde(n^{-1/2\bar{r}(k-k_*+1)})$ and $\Otilde(n^{-1/\bar{r}(k-k_*+1)})$, respectively. When $k-k_*+1=3$, these rates become $\Otilde(n^{-1/4})$, $\Otilde(n^{-1/12})$ and $\Otilde(n^{-1/6})$, which are substantially slower than our parametric rate. This highlights the benefits of using the Voronoi loss function over the generalized Wasserstein in the convergence analysis of the MLE.

\textbf{(ii)} Secondly, for any $j\in[k_*]$ such that $|\mathcal{A}^n_j|>1$, the rates for estimating true parameters $\aj,\bj,\sj$, which are fitted by more than one component, are not uniform. More specifically, the estimation rates for $\bj$ and $\sj$ are significantly slow, standing at orders $\Otilde(n^{-1/2\Bar{r}(|\mathcal{A}^n_j|)})$ and $\Otilde(n^{-1/\Bar{r}(|\mathcal{A}^n_j|)})$, respectively. This is due to the interaction between them via the PDE in equation~\eqref{eq:PDE_interaction}. By contrast, since $\aj$ does not interact with those parameters, their estimation rates are much faster of order $\Otilde(n^{-1/4})$. 

\textbf{(iii)} Finally, we point out a scenario when true parameters $b^*_j,\sigma^*_j$ attain the slowest estimation rates. In particular, assume that the MLE $\widehat{G}_n$ has $\hat{k}_n$ components. When $\widehat{G}_n$ converges to $G_*$, each Voronoi cell $\mathcal{A}^n_j$ contains at least one element for any $j\in[k_*]$, which implies that $|\mathcal{A}^n_j|\leq \hat{k}_n-k_*+1$. The equality is achieved if, for example, $|\mathcal{A}^n_1|=\hat{k}_n-k_*+1$ and $|\mathcal{A}^n_j|=1$ for any $j\in[k_*]\setminus\{1\}$. Then, the rates for estimating $b^*_1,\sigma^*_1$ reach the slowest orders of $\Otilde(n^{-1/2\Bar{r}(|\mathcal{A}^n_j|)})$ and $\Otilde(n^{-1/\Bar{r}(|\mathcal{A}^n_j|)})$, respectively, which match those of their counterparts in [Theorem 2, \cite{ho2022gaussian}]. Conversely, we achieve the fastest estimation rates for other parameters $b^*_j,\sigma^*_j$, which are of order $\Otilde(n^{-1/2})$. 

% Then, true parameters $a^*_1,b^*_1,\sigma^*_1$ enjoy the same estimation rates as their counterparts in [Theorem 2, \cite{ho2022gaussian}], which are $\Otilde(n^{-1/4})$, $\Otilde(n^{-1/2\bar{r}(k-k_*+1)})$ and $\Otilde(n^{-1/\bar{r}(k-k_*+1)})$, respectively. Conversely, the rates for estimating other parameters $a^*_j,b^*_j,\sigma^*_j$ It follows from the first observation that those $k_*-1$ fitted components admit the same fast convergence rate of order $\Otilde(n^{-1/2})$. 
%On the other hand, Theorem 2 in \cite{ho2022gaussian}, which uses the generalized Wasserstein as their loss function, indicates that the convergence rates of those fitted components are no better than $\Otilde(n^{-1/4})$. This highlights the benefits of using the Voronoi loss functions over the generalized Wasserstein in conducting the convergence analysis of the MLE.

%
%\vspace{-0.3 em}
\subsection{Non-distinguishable Settings}
\label{sec:non-distinguishable}
%\vspace{-0.3 em}
%In this section, we visit the non-distinguishable settings, which will be separated into two distinct regimes. Then, we characterize the convergence rates for parameter estimation for these regimes in Section~\ref{sec:partial_overlap} and Section~\ref{appendix:full_overlap}, respectively.
Under this setting, the mixture part $p_{G_*}$ is not distinguishable from the function $g_0$, that is, the set in equation~\eqref{eq:independent_set} is no longer linearly independent for almost surely $X$. There are several scenarios under which this phenomenon occurs, and one of them is when $g_0$ being a Gaussian mixture of $k_0$ experts as follows:
\begin{align}
\label{eq:g_0_partial_form}
g_0(Y|X)&=p_{G_0}(Y|X) \nonumber\\&:=\sum_{j=1}^{k_0}p^0_jf(Y|(a^0_j)^{\top}X+b^0_j,\sigma^0_j),
\end{align}
where $G_0:=\sum_{i=1}^{k_0}p^0_i\delta_{(a^0_i,b^0_i,\sigma^0_i)}\in\mathcal{E}_{k_0}(\Theta)$ with $k_0\in[k_*]$ such that $G_0$ and $G_*$ share some common atoms. Without loss of generality, we assume that $(a^0_i,b^0_i,\sigma^0_i)=(\aj,\bj,\sj)$ for any $j\in[\kbar]$, where $\kbar\in[k_0]$. We consider this choice of function $g_0$ as we can control the level of the interaction between $g_{0}$ and $p_{G_{\ast}}$ explicitly via the number of overlapped components $\kbar$ of two mixing measures $G_0$ and $G_*$. In particular, we consider two separate regimes of the value of $\kbar$. The first one is when $1\leq \Bar{k}<k_0$, which is referred to as the \emph{partial overlap} regime, and the second one is when $\Bar{k}=k_0$, which is termed the \emph{full overlap} regime. Due to the space limit, we will present only results for the partial overlap regime in this section, while those for the full overlap regime are deferred to Appendix~\ref{appendix:full_overlap}. Furthermore, similar to Section~\ref{sec:distinguishable}, we will also focus on only the scenario when $\lambdastar\in(0,1]$, and relegate the details for the scenario when $\lambdastar=0$ to Appendix~\ref{appendix:additional_results}.
% it is worth noting that there are two different settings of $\lambdastar$, which are $\lambdastar=0$ and $\lambdastar\in(0,1]$. However, we will concentrate only on the latter setting in this section, which is much more challenging, while details of the former can be found in Appendix~\ref{appendix:additional_results}. Additionally, we assume without loss of generality (WLOG) that $(\azj,\bzj,\szj)=(\aj,\bj,\sj)$ for all $j\in[\Bar{k}]$.
%\vspace{-0.3 em}

\textbf{Partial overlap.} 
Given the formulation of function $g_0$ in equation~\eqref{eq:g_0_partial_form}, the deviated Gaussian mixture of experts is no longer identifiable, that is, the equation $p_{\lambda,G}(X,Y)=p_{\lambdastar,G_*}(X,Y)$ for almost surely $(X,Y)$ does not merely lead to $(\lambda,G)\equiv (\lambdastar,G_*)$ anymore, which causes a significant issue compared to the distinguishable settings. Therefore, it is necessary to find all the solutions $(\lambda,G)$ of the equation $p_{\lambda,G}(X,Y)=p_{\lambdastar,G_*}(X,Y)$ for almost surely $(X,Y)$. For that purpose, let us consider a new mixing measure. In particular, for any mixing proportion $\lambda>\lambdastar$, we define
\begin{align}
\label{eq:G_bar_definition}
    \overline{G}_*(\lambda):=\left(1-\dfrac{\lambdastar}{\lambda}\right)G_0 + \dfrac{\lambdastar}{\lambda}G_*,
\end{align}
as a mixing measure having a total of $k_*+k_0-\Bar{k}$ components in $\Theta$. Note that the previous equation can be rewritten as $\lambda[p_{G}(X,Y)-p_{\overline{G}_*(\lambda)}(X,Y)]=0$ for almost surely $(X,Y)$. Thus, when $k\geq k_*+k_0-\Bar{k}$ and $\lambda>\lambdastar$, we admit $(\lambda, G)\equiv(\lambda, \overline{G}_*(\lambda))$ as a solution. On the other hand, when either $k< k_*+k_0-\kbar$ or $\lambda\leq\lambdastar$, we obtain an obvious solution $(\lambda,G)\equiv(\lambdastar,G_*)$.
For those reasons, we have to design a Voronoi loss function which is able to capture the aforementioned solutions.

\textbf{Voronoi loss function.} To facilitate our presentation, we assume that $\overline{G}_*(\lambda)=\sum_{j=1}^{k_*+k_0-\Bar{k}}p'_j(\lambda)\delta_{(a'_{j},b'_{j},\sigma'_{j})}$. When $k\geq k_*+k_0-\kbar$ and $\lambda>\lambdastar$, we introduce another set of Voronoi cells $\mathcal{B}_j\equiv\mathcal{B}_j(G)$ w.r.t an arbitrary mixing measure $G\in\mathcal{O}_k(\Theta)$ generated by the support of $\overline{G}_*(\lambda)$, denoted by $\theta'_j:=(a'_j,b'_j,\sigma'_j)$, as follows:
\begin{align}
    \label{eq:Voronoi_cells_B}
    \mathcal{B}_j=\{i\in[k]:\|\theta_i-\theta'_j\|\leq\|\theta_i-\theta'_{\ell}\|,\ \forall \ell\neq j\},
\end{align}
for any $j\in[\kstar+k_0-\kbar]$, where $\theta_i:=(a_i,b_i,\sigma_i)$. Let us denote $\Delta' a_{ij}:=a_i-a'_j$, $\Delta' b_{ij}:=b_i-b'_j$ and $\Delta' \sigma_{ij}:=\sigma_i-\sigma'_j$, then the discrepancy between two mixing measures $G$ and $\overline{G}_*(\lambda)$ can be characterized by
\begin{align*}
    &D_3(G,\overline{G}_*(\lambda)) \nonumber\\&:=\sum_{j:|\mathcal{B}_j|=1}\sum_{i\in \mathcal{B}_j}p_i\left(\|\Delta' a_{ij}\|+|\Delta' b_{ij}|+|\Delta' \sigma_{ij}|\right)\nonumber\\
    &+\sum_{j:|\mathcal{B}_j|>1}\sum_{i\in \mathcal{B}_j}p_i\left(\|\Delta' a_{ij}\|^2+|\Delta' b_{ij}|^{\brbj} \right.\nonumber\\&+\left.|\Delta' \sigma_{ij}|^{\brbj/2}\right)+\sum_{j=1}^{k_*+k_0-\Bar{k}}\left|\sum_{i\in\mathcal{B}_j} p_i- p'_j(\lambda)\right|.
\end{align*}
When either $k<k_*+k_0-\kbar$ or $\lambda\leq\lambdastar$, we reuse the loss function $D_1$ in equation~\eqref{eq:D_distinguishable_dependent} to capture the solution $(\lambda,G)\equiv(\lambdastar,G_*)$. Thus, we combine the two loss functions $D_1$ and $D_3$ together to construct the following general Voronoi loss function used for any settings of $k$ and $\lambda$ under the partial overlap regime: 
\begin{align}
    \label{eq:D_partial_dependent}
    &D_2((\lambda, G),(\lambdastar, G_*))\nonumber\\
    &:=\begin{cases}
        D_1((\lambda,G),(\lambda^*,G_*)), \hspace{2.4em}\forall k< k_*+k_0-\bar{k};\\
        \textbf{}\\
        \ones_{\{\lambda\leq\lambdastar\}}D_1((\lambda,G),(\lambdastar, \Gstar)) \\+\ones_{\{\lambda>\lambdastar\}}D_3({G},\bGstar), \hspace{1em}\forall k\geq \kstar+\kzero-\kbar.
    \end{cases}
\end{align}
Now, we are ready to present the main result for the partial overlap regime in the following theorem.
\begin{theorem}
\label{theorem:partial_dependent}
Assume that $\lambda^*\in(0,1]$ is unknown, and let $g_{0}$ take the form in equation~\eqref{eq:g_0_partial_form} with $1\leq\bar{k}<k_0$. Then, we obtain that $V(p_{\lambda,G},p_{\lambdastar,G_*})\gtrsim D_2((\lambda,G),(\lambdastar,G_*))$ for any $(\lambda,G)\in[0,1]\times\Ocal_k(\Theta)$. This bound together with Proposition~\ref{prop:mle_estimation} indicate that  
\begin{align*}
    \mathbb{P}(D_2((\hat{\lambda}_n, \widehat{G}_n),(\lambdastar,\Gstar))&>C_2\sqrt{\log (n)/n}) \lesssim n^{-c_2},
\end{align*}
where $C_2>0$ is a constant depending on $g_0,\lambdastar,\Gstar,\Theta$, while the constant $c_2>0$ depends only on $\Theta$.
\end{theorem}
Proof of Theorem~\ref{theorem:partial_dependent} is in Appendix~\ref{appendix:partial_dependent}. A few comments regarding Theorem~\ref{theorem:partial_dependent} are in order. Firstly, when either $k< \kstar+k_0-\kbar$ or $\hat{\lambda}_n\leq\lambdastar$, the loss function $D_2$ reduces to $D_1$. Therefore, the convergence rates of parameter estimation remain the same as those in Theorem~\ref{theorem:distinguishable_dependent}. Secondly, when $k\geq k_*+k_0-\Bar{k}$ and $\hat{\lambda}_n>\lambdastar$, the loss function $D_2$ turns into $D_3$. As a result, for true components $(a'_j,b'_j,\sigma'_j)$ which are approximated by more than one fitted components, the rates of estimating $b'_j$ and $\sigma'_j$ are reported to be $\Otilde(n^{-1/2\bar{r}(|\mathcal{B}^n_j|)})$ and $\Otilde(n^{-1/\bar{r}(|\mathcal{B}^n_j|)})$, respectively, where $\mathcal{B}^n_j:=\mathcal{B}_j(\widehat{G}_n)$. Meanwhile, those for $a'_j$ are of order $\Otilde(n^{-1/4})$. However, for true components $(a'_j,b'_j,\sigma'_j)$ approximated by a single fitted component, their estimation rates are uniform, standing at $\Otilde(n^{-1/2})$. This again confirms the ability to capture distinct estimation rates accurately of the proposed Voronoi loss functions in comparison with the generalized Wasserstein.
%\vspace{-0.3 em}

\subsection{Proof Sketch}
\label{sec:proof_sketch}
In this section, we provide a generic proof sketch for Theorems~\ref{theorem:distinguishable_dependent} and \ref{theorem:partial_dependent}, and distinguish our proof techniques from those used in the most related work \cite{ho2022gaussian}. More details of these proofs are deferred to Appendix~\ref{appendix:missing_proofs}. For simplicity, the metric $\mathcal{D}$ used in this sketch is implicitly understood as one among the Voronoi loss functions $D_1$ and $D_2$ in Sections~\ref{sec:distinguishable} and~\ref{sec:non-distinguishable}. Generally, our goal is to establish the Total Variation lower bound $V(p_{\lambda,G},p_{\lambdastar,\Gstar})\gtrsim \mathcal{D}((\lambda,G),(\lambdastar,\Gstar))$ for any $(\lambda,G)\in[0,1]\times\Ocal_{k}(\Theta)$, which together with Proposition~\ref{prop:mle_estimation} give us our desired conclusions in those theorems. 

\begin{figure*}[t]
\begin{center}
    
  \begin{tabular}{cc}
  \widgraph{0.38\textwidth}{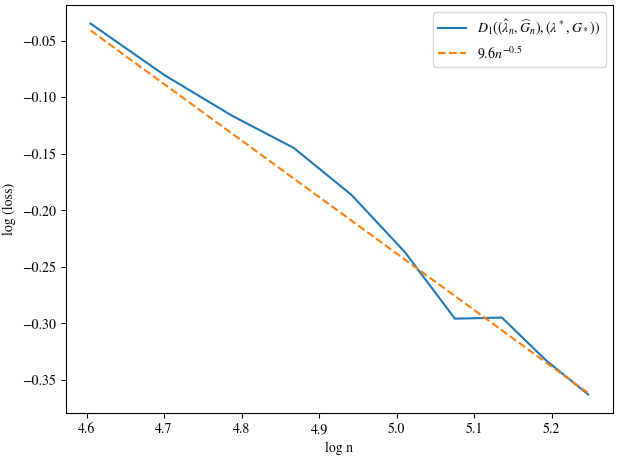} 
&
\widgraph{0.38\textwidth}{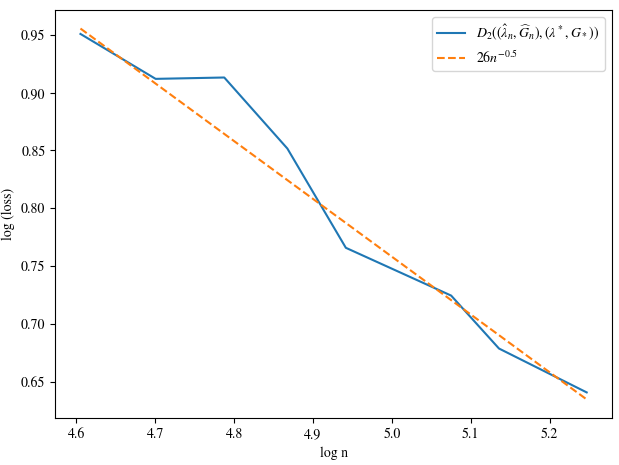} \\
(a) Distinguishable setting &(b) Non-distinguishable setting
% &
% \widgraph{0.26\textwidth}{CelebAHQ/fid_celebahq.pdf} 
  \end{tabular}
  \end{center}
  \vskip -0.1in
  \caption{
  \footnotesize{Convergence rate of the maximum likelihood estimation $(\hat{\lambda}_n,\widehat{G}_n)$ under the Voronoi loss functions. 
}
} 
  \label{fig:dis_estimation}
   \vskip -0.2in
\end{figure*}

\textbf{Local inequality.} First, we prove the following local bound by contradiction in three main steps:
\begin{align}
    \label{eq:local_sketch}
    &\lim_{\varepsilon\to 0}\inf_{\substack{(\lambda,G)\in[0,1]\times\Ocal_{k}(\Theta):\\ \mathcal{D}((\lambda,G),(\lambdastar,\Gstar))\leq\varepsilon}} \frac{V(p_{\lambda,G},p_{\lambdastar,\Gstar})}{\mathcal{D}((\lambda,G),(\lambdastar,\Gstar))}>0.
\end{align}
\textbf{Step 1.} Assume that the local inequality does not hold, then there exists a sequence $(\lambdan,G_n)$ such that both $\mathcal{D}_n:=\mathcal{D}((\lambda_n,G_n),(\lambdastar,\Gstar))$ and $V(p_{\lambdan,G_n},p_{\lambdastar,\Gstar})/\mathcal{D}_n$ vanish as $n\to\infty$. Different from \cite{ho2022gaussian}, we need to invoke the Taylor expansion twice in this step due to the sophisticated structure of our metric $\mathcal{D}$. Firstly, for each $j\in[k_*]:|\mathcal{A}_j|=1$, we apply the first-order Taylor expansion to the quantity $U_n:=[p_{\lambdan,G_n}(X,Y)-p_{\lambdastar,G_*}(X,Y)]/\mathcal{D}_n$, whereas for each $j\in[k_*]:|\mathcal{A}_j|>1$, we use the Taylor expansion up to order $r_j$ that will be chosen later in Step 2. Then, we show that $U_n$ can be written as a linear combination of elements of some linearly independent set $\mathcal{H}$.

\textbf{Step 2.} We attempt to show by contradiction that at least one among the coefficients in that combination does not converge to zero. Assume that all of them go to zero. Then, by some algebraic transformations of those limits, we arrive at the following system of polynomial equations:
\begin{align*}
    \sum_{i\in\mathcal{A}_j}\sum_{\substack{n_1,n_2\in\mathbb{N},\\ n_1+2n_2=\beta}}\dfrac{s^2_{i}~t^{n_1}_{1i}~t^{n_2}_{2i}}{n_1!n_2!}=0, \quad \beta=1,2,\ldots,r_j.
\end{align*}
By construction, this system necessarily has at least one non-trivial solution. Therefore, in order to point out a contradiction, we set $r_j=\Bar{r}(|\mathcal{A}_j|)$ so that the above system does not have any non-trivial solutions.

\textbf{Step 3.} On the other hand, by means of Fatou's lemma and the limit $V(p_{\lambdan,G_n},p_{\lambdastar,\Gstar})/\mathcal{D}_n\to 0$, we show that $U_n\to 0$ for almost surely $(X,Y)$. Since $\mathcal{H}$ is a linearly independent set, we deduce that all the coefficients of elements of $\mathcal{H}$ in the representation of $U_n$ vanish as $n\to\infty$, which contradicts the result in Step 2. Hence, we obtain the local inequality in equation~\eqref{eq:local_sketch}. 

\textbf{Global inequality.} Given the above local inequality, it suffices to prove the following global bound:
\begin{align}
    \label{eq:global_sketch}
    &\inf_{\substack{(\lambda,G)\in[0,1]\times\Ocal_{k}(\Theta):\\ \mathcal{D}((\lambda,G),(\lambdastar,\Gstar))>\varepsilon}}\frac{V(p_{\lambda,G},p_{\lambdastar,\Gstar})}{\mathcal{D}((\lambda,G),(\lambdastar,\Gstar))}>0.
\end{align}
A key step for establishing this bound is to solve the equation $p_{\lambda,G}(X,Y)=p_{\lambdastar,G_*}(X,Y)$ for almost surely $(X,Y)$. If $p_{G_*}$ is distinguishable from $g_0$, then this equation has the unique solution $(\lambda,G)\equiv(\lambdastar,G_*)$. However, this property does not hold when the distinguishability condition fails, which induces a huge challenge compared to \cite{ho2022gaussian}. Thus, we solve the previous equation under the partial overlap regime in Section~\ref{sec:distinguishable}. Hence, the proof sketch is completed.

\section{SIMULATION STUDY}
\label{sec:experiments}
In this section, we carry out a simulation study to empirically verify our theoretical results regarding the convergence rate of the MLE $(\hat{\lambda}_n,\widehat{G}_n)$ under both the distinguishable and non-distinguishable settings.

\textbf{Distinguishable setting.} We first generate the covariates $X_1,\ldots, X_n \overset{i.i.d}{\sim} \mathcal{N}(0,1)$, and then set
\begin{align*}
    g_0(Y|X)=\sum_{j=1}^{2}p_j^0 f(Y|a^0_{j}X +b^0_{j},\sigma^0_{j}),
\end{align*}
where $(a^0_1,b^0_1,\sigma^0_1)=(.2,.1,.01)$, $(a^0_2,b^0_2,\sigma^0_2)=(.1,0,.01)$ and  
%$a^0_{1}=0.2,b^0_1=0.1$, $\sigma^0_1=0.01$, $a^0_{2}=0.1,b_2^0=0$, $\sigma^0_{2}=0.01$, and
$p^0_{1}=p^0_{2}=\frac{1}{2}$. Next, we consider the following true conditional density function: $$p_{\lambda^*,G_*}(Y|X):=(1 - \lambda^{*}) g_0(Y| X) + \lambda^{*} f(Y|a^*X+b^*,\sigma^*),$$ in which $\lambda^*= 0.5$ and $(a^*,b^*,\sigma^*)=(1,1,1)$. Here, we have $k_0=2>1=k_*$, therefore, the distinguishability condition is satisfied according to Example~\ref{eq:example}. Subsequently, we draw a sample $Y_1,Y_2\ldots,Y_n$ of size $n\in \{100,110,120,\ldots,200\}$ from $p_{\lambda^*,G_*}(Y|X)$. Then, we overfit the true model by a deviated Gaussian mixture of $k=2$ experts, and run the EM algorithm \cite{dempster_maximum_1977} in 1000 iterations to find the estimators $\hat{\lambda}_n$, $(\hat{p}^n_1,\hat{a}^{n}_1,\hat{b}^{n}_1,\hat{\sigma}^{n}_1)$ and $(\hat{p}^n_2,\hat{a}^{n}_2,\hat{b}^{n}_2,\hat{\sigma}^{n}_2)$. Finally, we compute the discrepancy $D_1((\hat{\lambda}_n,\widehat{G}_n),(\lambdastar,G_*))$, where $\widehat{G}_n=\sum_{i=1}^{2}\hat{p}^n_i\delta_{(\hat{a}^n_i,\hat{b}^n_i,\hat{\sigma}^n_i)}$ and $G_*=\delta_{(a^*,b^*,\sigma^*)}$,
% $$D_1(\lambda G,\lambdastar, G_*|n)=|\hat{\lambda}_n-\lambdastar|+(\hat{\lambda}_n+\lambdastar)\times[|(\hat{a}_{1,n}-a_{1}^*)|+|(\hat{b}_{1,n}-b_1^*)^{(2)}|+|\hat{\sigma}_{1,n}-\sigma_1^*|+|\hat{\lambda}_n-\lambdastar|].$$
and plot its values in Figure~\ref{fig:dis_estimation}(a). From the figure, it is obvious that the convergence rate of the MLE $(\hat{\lambda}_n,\widehat{G}_n)$ under the Voronoi loss $D_1$ is at the order of $\Otilde(n^{-1/2})$, which is consistent with our theoretical result in Theorem~\ref{theorem:distinguishable_dependent}.

\textbf{Non-distinguishable setting.} For this setting, we also generate $X_1,\ldots, X_n \overset{i.i.d}{\sim} \mathcal{N}(0,1)$, but set
\begin{align*}
    g_0(Y|X)=\sum_{j=1}^{3}p_j^0 f(Y|a^0_{j}X +b^0_{j},\sigma^0_{j}),
\end{align*}
with $(a^0_1,b^0_1,\sigma^0_1)=(.2,.1,.01)$, $(a^0_2,b^0_2,\sigma^0_2)=(.1,.1,.01)$, $(a^0_3,b^0_3,\sigma^0_3)=(.21,.11,.01215)$
%$a^0_{1}=0.2,b^0_1=0.1$, $\sigma^0_1=0.01$, $a^0_{2}=0.1,b_2^0=0.1$, $\sigma^0_{2}=0.01$,  $a^0_{3}=0.21,b_3^0=0.11$, $\sigma^0_{3}=0.01215$, 
and  $p^0_{1}=p^0_{2}=p_0^3=\frac{1}{3}$. Next, we consider the following true conditional density
\begin{align*}
    p_{\lambda^*,G_*}(Y|X):=&~(1 - \lambda^{*}) g_0(Y| X) \\
    &+ \lambda^{*} \sum_{j=1}^3 p_j^*f(Y|a^*_j X+b_1^*,\sigma_j^*),
\end{align*}
with $(a^*_1,b^*_1,\sigma^*_1)=(.2,.1,.01)$, $(a^*_2,b^*_2,\sigma^*_2)=(.1,.4,.25)$, $(a^*_3,b^*_3,\sigma^*_3)=(1.1,.3,.25)$
%$a^*_{1} =0.2, b^*_{1}=0.1,\sigma_1^* =0.01$, $a^*_{2} =0.1, b^*_{2}=0.4,\sigma_2^* =0.25$, $a^*_{3} =1.1, b^*_{3}=0.3,\sigma_3^* =0.25$, 
$p^*_{1}=p^*_{2}=p^*_{3}=\frac{1}{3}$ and $\lambda^*= 0.5$. Here, each of $G_0$ and $G_*$ has 3 components and they share one common component among them, specifically $(a^0_1,b^0_1,\sigma^0_1)=(a^*_1,b^*_1,\sigma^*_1)$. Thus, $p_{G_*}$ is not distinguishable from $g_0$, and this setting belongs to the partial overlap regime in Section~\ref{sec:non-distinguishable} with $\kbar=1$. We then sample $Y_1,\ldots,Y_n$ with $n\in \{100,110,120,\ldots,200\}$ from $p_{\lambda^*,G_*}(Y|X)$. Next, we overfit the true model by a deviated Gaussian of $k=4$ experts and use the EM algorithm to find the estimators $\hat{\lambda}_n$, $\hat{p}^n_{i}$, $\hat{a}^n_{i}$, $\hat{b}^n_{i}$ and $\hat{\sigma}^n_{i}$ for $i\in[4]$ in 1000 iterations. After that, we calculate the distance $D_2((\hat{\lambda}_n, \widehat{G}_n),(\lambdastar, G_*))$, and plot its values in Figure~\ref{fig:dis_estimation}(b). It can be seen from the figure that the convergence rate of the MLE $(\hat{\lambda}_n,\widehat{G}_n)$ is at the order of $\Otilde(n^{-1/2})$, which aligns with our claim in Theorem~\ref{theorem:partial_dependent}.
\section{CONCLUSION}
\label{sec:conclusion}
In this paper, we characterize the convergence behaviors of maximum likelihood estimation in the deviated Gaussian mixture of experts, which is motivated by the goodness-of-fit test. We first show that the convergence rate of density estimation to the true density is parametric on the sample size. Regarding the parameter estimation problem, we consider two separate settings based on the level of distinguishability between the known function $g_{0}(Y|X)$ and the mixture of experts part $p_{G_*}(Y|X)$ in the proposed model. In each setting, we design a novel Voronoi loss function to capture the interaction among the parameters of expert functions, and the distinguishability of $p_{G_*}$ from $g_0$. 
We then theoretically and empirically demonstrate that our proposed loss functions outperform the generalized Wasserstein studied in previous work in terms of precisely characterizing distinct parameter estimation rates, which are determined by the solvability of a system of polynomial equations.

% En route to establishing the convergence rates of the MLE and its parameters, we design two novel Voronoi loss functions, which outperform the generalized Wasserstein studied in previous work in precisely capturing the distinct parameter estimation rates. 
% In each setting, we provide a rigorous theoretical guarantee to demonstrate that by using those loss functions, there is a class of parameters enjoying significantly faster estimation rates than those in the previous work. 

\bibliographystyle{abbrv}
\bibliography{aistats_references}

\newpage
\onecolumn
\appendix

\centering
\textbf{\Large{Supplementary Material for 
``On Parameter Estimation in Deviated \\ \vspace{0.1em} Gaussian Mixture of Experts''}}

\justifying
\setlength{\parindent}{0pt}
% \begin{center}
% \textbf{\Large{Supplementary Materials for
% ``On Parameter Estimation in Deviated Gaussian Mixture of Experts''}}
% \end{center}
In this supplementary material, we first provide proofs for main results in Appendix~\ref{appendix:missing_proofs}, while leaving those for auxiliary results in Appendix~\ref{appendix:proof_auxiliary_results}. Then, we present the convergence behavior of parameter estimation under the full overlap regime in Appendix~\ref{appendix:full_overlap}. Finally, we study the parameter estimation problem in the deviated Gaussian mixture of experts when the true mixing proportion vanishes, i.e. $\lambdastar=0$, in Appendix~\ref{appendix:additional_results}. 
\section{PROOF OF MAIN RESULTS}
\label{appendix:missing_proofs}
In this appendix, we provide the proof of Theorem~\ref{theorem:distinguishable_dependent} in Appendix~\ref{appendix:distinguishable_dependent}, and then present that for Theorem~\ref{theorem:partial_dependent} in Appendix~\ref{appendix:partial_dependent}.

To begin with, let us define some essential notations that will be used in our arguments. For any vectors $u=(u_1,u_2,\ldots,u_d)\in\mathbb{R}^d$ and $q=(q_1,q_2,\ldots,q_d)\in\mathbb{N}^d$, we denote $u^{q}:=u_{1}^{q_1}u_{2}^{q_2}\ldots u_{d}^{q_d}$, $|u|:=u_1+u_2+\ldots+u_d$ and $q!:=q_1!q_2!\ldots q_d!$. Next, the two expert functions considered in this work are denoted by $h_1(X,a,b)=a^{\top}X+b$ and $h_2(X,\sigma)=\sigma$, for any $(a,b,\sigma)\in\Theta$ and $X\in\mathcal{X}$. Finally, for any set $A$, we denote $A^c$ as its complement.
%%%%%%%%%%%%%%%%%%%
\subsection{Proof of Theorem~\ref{theorem:distinguishable_dependent}}
\label{appendix:distinguishable_dependent}
According to the result in Proposition~\ref{prop:mle_estimation}, in order to reach the desired conclusion, we need to demonstrate that $V(p_{\hat{\lambda}_n, \widehat{G}_n},p_{\lambdastar, G_*})\gtrsim D_{1}((\hat{\lambda}_n, \widehat{G}_n),(\lambdastar, G_*))$, which follows from the following inequality:
\begin{align}
    \label{eq:original_inequality}
    \inf_{\lambda\in[0,1],G\in\Ocal_{k,\xi}(\Theta)}\dfrac{V(p_{\lambda, G},p_{\lambdastar, G_*})}{D_{1}((\lambda, G),(\lambdastar, G_*))}>0.
\end{align}
For that purpose, we split the above inequality into two parts, which we referred to local inequality and global inequality. Note that in the above infimum is subject to mixing measures in the set $\mathcal{O}_{k,\xi}(\Theta):=\{G = \sum_{i=1}^{k'}p_i\delta_{(a_{i},b_{i},\sigma_{i})}:1\leq k'\leq k, \ p_i\geq\xi,\  (a_{i},b_{i},\sigma_{i})\in\Theta\}$ for some $\xi>0$ for simplicity.

\textbf{Local inequality.} Firstly, we will show the following local inequality:
\begin{align}
    \label{eq:distinguishable_dependent_local}
    \lim_{\varepsilon\to 0}\inf_{\substack{\lambda\in[0,1],G\in\Ocal_{k,\xi}(\Theta):\\ D_{1}((\lambda, G),(\lambdastar, G_*))\leq\varepsilon}}\dfrac{V(p_{\lambda, G},p_{\lambdastar, G_*})}{D_{1}((\lambda, G),(\lambdastar, G_*))}>0.
\end{align}
Assume by contrary that the above claim does not hold, then there exist a sequence of mixing measures $G_n=\sum_{i=1}^{k_n}p^n_i\delta_{(\ain,\bin,\sigmain)}\in\Ocal_{k,\xi}(\Theta)$ and a sequence of mixing proportions $\lambdan\in[0,1]$ that satisfy
\begin{align*}
\begin{cases}
    D_{1n}:=D_{1}((\lambdan, G_n),(\lambdastar, G_*))\to 0,\\
    V(p_{\lambdan, G_n},p_{\lambdastar, G_*})/D_{1n}\to 0,
\end{cases}
\end{align*}
as $n\to\infty$. Next, the following Voronoi cells with respect to $\widehat{G}_n$ is defined as:
\begin{align*}
    \mathcal{A}_j^n=\mathcal{A}_j(G_n)=\{i\in[k_n]:\|\theta^n_i-\theta^*_j\|\leq\|\theta^n_i-\theta^*_{\ell}\|, \ \forall \ell\neq j\},\quad  \forall j\in[k_*], 
\end{align*}
where $\theta^n_i:=(\ain,\bin,\sigmain)$ and $\theta^*_j:=(\aj,\bj,\sj)$. Since $k_n\leq k$ for all $n$, we can find a subsequence of $G_n$ such that $k_n$ does not change with $n$. Then, by replacing $G_n$ with this subsequence, we may assume that $k_n=k$ for all $n\in\mathbb{N}$. Additionally, $\mathcal{A}_j=\mathcal{A}_j^n$ does not change with $n$ for all $j\in[k_*]$, either. As $D_{1n}\to 0$, we can represent $G_n=\sum_{j=1}^{\kstar}\sum_{i\in\mathcal{A}_j}p^n_i\delta_{(\ain,\bin,\sigmain)}$ such that $|\mathcal{A}_j|\geq 1$ for all $j\in[k_*]$ and $\sum_{j=1}^{k_*}|\mathcal{A}_j|=k$. Furthermore, it follows from the formulation of metric $D_1$ in equation~\eqref{eq:D_distinguishable_dependent} that $\lambdan\to\lambdastar$, $(\ain,\bin,\sigmain)\to(\aj,\bj,\sj)$ for any $i\in\mathcal{A}_j$ and $\sum_{i\in\mathcal{A}_j}p^n_i\to p^*_j$ for any $j\in[k_*]$ when $n$ tends to infinity.

\textbf{Step 1 - Application of Taylor expansion.} Now, we consider the following quantity:
\begin{align}
    &p_{\lambdan, G_n}(X,Y)-p_{\lambdastar, G_*}(X,Y)\nonumber\\
    &=\sum_{j:|\mathcal{A}_j|>1}\sum_{i\in \mathcal{A}_j}\lambdan p^n_i[f(Y|(\ain)^{\top}X+\bin,\sigmain)-f(Y|(\aj)^{\top}X+\bj,\sj)]\Bar{f}(X)\nonumber\\
    &+\sum_{j:|\mathcal{A}_j|=1}\sum_{i\in \mathcal{A}_j}\lambdan p^n_i[f(Y|(\ain)^{\top}X+\bin,\sigmain)-f(Y|(\aj)^{\top}X+\bj,\sj)]\Bar{f}(X)\nonumber\\
    &+\sum_{j=1}^{k_*}\left(\sum_{i\in\mathcal{A}_j}\lambdan p^n_i-\lambdastar p^*_j\right)f(Y|(\aj)^{\top}X+\bj,\sj)\Bar{f}(X)+(\lambdastar-\lambdan)g_0(Y|X)\Bar{f}(X)\nonumber\\
    \label{eq:density_decomposition}
    &:=A_{n,1}+A_{n,2}+B_n+C_n.
\end{align}
For the sake of presentation, let us denote $\Delta a^n_{ij}:=\ain-\aj$, $\Delta b^n_{ij}:=\bin-\bj$ and $\Delta\sigma^n_{ij}:=\sigmain-\sj$ for all $i\in[k_n]$ and $j\in[k_*]$. For each $j\in[k_*]$ such that $|\mathcal{A}_j|>1$, by applying the Taylor expansion up to order $\Bar{r}(|\mathcal{A}_j|)$, we can rewrite $f(Y|(\ain)^{\top}X+\bin,\sigmain)-f(Y|(\aj)^{\top}X+\bj,\sj)$ as
\begin{align*}
    &\sum_{|\alpha|=1}^{\Bar{r}(|\mathcal{A}_j|)}\frac{1}{\alpha!}(\daijn)^{\alpha_1}(\dbijn)^{\alpha_2}(\dsijn)^{\alpha_3}\frac{\partial^{|\alpha_1|+\alpha_2+\alpha_3}f}{\partial a^{\alpha_1}\partial b^{\alpha_2}\partial \sigma^{\alpha_3}}(Y|(\aj)^{\top}X+\bj,\sj) + R_{1ij}(Y|X)\\
    =&\sum_{|\alpha|=1}^{\Bar{r}(|\mathcal{A}_j|)}\frac{1}{\alpha!}(\daijn)^{\alpha_1}(\dbijn)^{\alpha_2}(\dsijn)^{\alpha_3}\frac{X^{\alpha_1}}{2^{\alpha_3}}\frac{\partial^{|\alpha_1|+\alpha_2+2\alpha_3}f}{\partial h_1^{|\alpha_1|+\alpha_2+2\alpha_3}}(Y|(\aj)^{\top}X+\bj,\sj) + R_{1ij}(Y|X),
\end{align*}
where $R_{1ij}(Y|X)$ is a Taylor remainder term such that $R_{1ij}(X,Y)/D_{1n}\to 0$, and the first equality comes from the following partial differential equation (PDE):
\begin{align*}
    \frac{\partial^{\alpha_3} f}{\partial h_2^{\alpha_3}}(Y|(\aj)^{\top}X+\bj,\sj)=\frac{1}{2^{\alpha_3}}\cdot\frac{\partial^{2\alpha_3} f}{\partial h_1^{2\alpha_3}}(Y|(\aj)^{\top}X+\bj,\sj).
\end{align*}
As a result, let $\ell=\alpha_2+2\alpha_3$, then $A_{n,1}$ can be represented as
\begin{align*}
    A_{n,1}=\sum_{j:|\mathcal{A}_j|>1}\sum_{|\alpha_1|=0}^{\brj}\sum_{\ell=0}^{2(\brj-|\alpha_1|)}E^n_{\alpha_1,\ell}(j)X^{\alpha_1}\cdot\frac{\partial^{|\alpha_1|+\ell}f}{\partial h_1^{|\alpha_1|+\ell}}(Y|(\aj)^{\top}X+\bj,\sj)\Bar{f}(X) 
    + R_1(X,Y),
\end{align*}
where $R_1(X,Y):=\sum_{j=1}^{k_*}\sum_{i\in\mathcal{A}_j}R_{1ij}(Y|X)\Bar{f}(X)$, which leads to $R_1(X,Y)/D_{1n}\to 0$ as $n\to\infty$. In addition, the coefficients $E^n_{\alpha_1,\ell}(j)$ in this representation are defined as
\begin{align*}
    E^n_{\alpha_1,\ell}(j):=\sum_{i\in\mathcal{A}_j}\sum_{\substack{\alpha_2+2\alpha_3=\ell \\ \alpha_2+\alpha_3\geq 1-|\alpha_1|}}\frac{\lambdan p^n_i}{2^{\alpha_3}\alpha!}\cdot(\daijn)^{\alpha_1}(\dbijn)^{\alpha_2}(\dsijn)^{\alpha_3},
\end{align*}
for any $j\in[k_*]$, $0\leq |\alpha_1|\leq \brj$ and $0\leq \ell\leq 2(\brj-|\alpha_1|)$.

Similarly, by means of Taylor expansion up to the first order, $A_{n,2}$ is decomposed as
\begin{align*}
    A_{n,2}=\sum_{j:|\mathcal{A}_j|=1}\sum_{|\alpha_1|=0}^{1}\sum_{\ell=0}^{2(1-|\alpha_1|)}E^n_{\alpha_1,\ell}(j)X^{\alpha_1}\cdot\frac{\partial^{|\alpha_1|+\ell}f}{\partial h_1^{|\alpha_1|+\ell}}(Y|(\aj)^{\top}X+\bj,\sj)\Bar{f}(X)+R_2(X,Y),
\end{align*}
where $R_2(X,Y)$ is a Taylor remainder term such that $R_2(X,Y)/D_{1n}\to 0$ as $n\to\infty$. 

Note that three terms $A_{n,1}$, $A_{n,2}$, $B_n$ and $C_n$ can be viewed as linear combinations of elements of the set $\mathcal{H}_1$ defined as
\begin{eqnarray}    
\label{eq:distinguishable_dependent_2}
\mathcal{H}_1 : = \Bigg\{X^{\alpha_1}\cdot\frac{\partial^{|\alpha_1|+\ell}f}{\partial h_1^{|\alpha_1|+\ell}}(Y|(\aj)^{\top}X+\bj,\sj)\Bar{f}(X), \ g_0(Y|X)\Bar{f}(X):\ j\in[k_*],\nonumber\\ 
0\leq |\alpha_1|\leq \brj, \ 0\leq \ell\leq 2(\brj-|\alpha_1|)\Bigg\}.
\end{eqnarray}
\textbf{Step 2 - Non-vanishing coefficients.} In this step, we will show by contradiction that not all the coefficients in the representations of $A_{n,1}/D_{1n}$, $A_{n,2}/D_{1n}$, $B_n/D_{1n}$ and $C_n/D_{1n}$ vanish as $n\to\infty$. In particular, assume that all of them vanish converge to zero. Given this hypothesis, it can be seen from the definitions of $C_n$ and $B_n$ in equation~\eqref{eq:density_decomposition} that
\begin{align}
\label{eq:Cn_Bn_vanish}
\frac{(\lambdastar-\lambdan)}{D_{1n}}\to 0, \qquad \frac{1}{D_{1n}}\cdot\sum_{j=1}^{k_*}\left|\sum_{i\in\mathcal{A}_j}\lambdan p^n_i-\lambdastar p^*_j\right|\to 0.
\end{align}
Regarding the coefficients in $A_{n,2}$, by considering the limits of $E^n_{\zerod,1}(j)/D_{1n}$ and $E^n_{\alpha_1,0}(j)/D_{1n}$ for $j\in[k_*]:|\mathcal{A}_j|=1$ and  $\alpha_1\in\{e_1,e_2,\ldots,e_d\}$, where $e_u:=(0,\ldots,0,\underbrace{1}_{\textit{u-th}},0,\ldots,0)$ being a one-hot vector in $\mathbb{R}^d$ for any $u\in[d]$, we obtain that
\begin{align*}
    \frac{1}{D_{1n}}\cdot\sum_{j:|\mathcal{A}_j|=1}\sum_{i\in\mathcal{A}_j}\lambdan p^n_i\Big(\|\daijn\|_1+|\dbijn|+|\dsijn|\Big)\to 0.
\end{align*}
Due to the topological equivalence of $1$-norm and $2$-norm, the above limit is equivalent to
\begin{align}
    \label{eq:An2_vanish}
    \frac{1}{D_{1n}}\cdot\sum_{j:|\mathcal{A}_j|=1}\sum_{i\in\mathcal{A}_j}\lambdan p^n_i\Big(\|\daijn\|+|\dbijn|+|\dsijn|\Big)\to 0.
\end{align}
Regarding the coefficients in $A_{n,1}$, it follows from the limits of $E^n_{\alpha_1,0}(j)/D_{1n}$ for any $\alpha_1\in\{2e_1,2e_2,\ldots,2e_d\}$ and $j\in[k_*]:|\mathcal{A}_j|>1$ that
\begin{align*}
    \frac{1}{D_{1n}}\cdot\sum_{j:|\mathcal{A}_j|>1}\sum_{i\in\mathcal{A}_j}\lambdan p^n_i\|\daijn\|^2\to 0.
\end{align*}
Putting the results in equations~\eqref{eq:Cn_Bn_vanish} and \eqref{eq:An2_vanish} together with the formulation of $D_{1n}$ that
\begin{align*}
    \frac{1}{D_{1n}}\cdot\sum_{j:|\mathcal{A}_j|>1}\sum_{i\in\mathcal{A}_j}\lambdan p^n_i\Big(|\dbijn|^{\brj}+|\dsijn|^{\brj/2}\Big)\to1.
\end{align*}
Therefore, we can find an index $j^*\in[k_*]:|\mathcal{A}_{j^*}|>1$ such that
\begin{align*}
    \frac{1}{D_{1n}}\cdot \sum_{i\in\mathcal{A}_{j^*}}\lambdan p^n_i\Big(|\Delta b^n_{ij^*}|^{\Bar{r}(|\mathcal{A}_{j^*}|)}+|\Delta \sigma^n_{ij^*}|^{\Bar{r}(|\mathcal{A}_{j^*}|)/2}\Big)\not\to 0.
\end{align*}
WLOG, we assume that $j^*=1$ throughout this proof. Moreover, since $E^n_{\zerod,\ell}(1)/D_{1n}\to 0$ as $n\to\infty$ for any $1\leq \ell\leq \brone$, we deduce that
\begin{align}
    \label{eq:ratio_distinguishable}
    \dfrac{\sum_{i\in\mathcal{A}_{1}}\sum_{\substack{\alpha_2+2\alpha_3=\ell}} p^n_i\cdot\dfrac{(\dbione)^{\alpha_2}(\dsione)^{\alpha_3}}{2^{\alpha_3}\alpha_2!\alpha_3!}}{\sum_{i\in\mathcal{A}_{1}} p^n_i\Big(|\dbione|^{\brone}+|\dsione|^{\brone/2}\Big)}\to0,
\end{align}
for any $1\leq \ell\leq \brone$.
Subsequently, we denote
\begin{align*}
    \overline{M}_n=\max\{|\dbione|, \ |\dsione|^{1/2}:i\in\mathcal{A}_{1}\}, \quad \overline{p}_n=\max_{i\in\mathcal{A}_{1}} p^n_{i}.
\end{align*}
Since the sequence $p^n_i/\overline{p}_n$ is bounded, we can substitute it by its subsequence which admits a non-negative limit $s_i^2=\lim_{n\to\infty}p^n_{i}/\overline{p}_n$. Furthermore, as $p^n_{i}\geq \xi>0$ for all $i\in\mathcal{A}_{1}$, at least one among the limit $s_i^2$ is equal to $1$. Similarly, let $(\dbione)/\overline{M_n}\to t_{1i}$ and $(\dsione)/(2\overline{M}_n^2)\to t_{2i}$ as $n\to\infty$ for any $i\in\mathcal{A}_{1}$. Then, at least one among $t_{1i}$ and $t_{2i}$ for $i\in\mathcal{A}_{1}$ is equal to either $1$ or $-1$.

Then, we divide both the numerator and the denominator of the ratio in equation~\eqref{eq:ratio_distinguishable} by $\overline{p}_n\overline{M}_n^{\ell}$, and obtain the following system of polynomial equations:
\begin{align*}
    \sum_{i\in\mathcal{A}_{1}}\sum_{\substack{\alpha_2+2\alpha_{3}=\ell}}\dfrac{s_i^2~t_{1i}^{\alpha_2}~t_{2i}^{\alpha_{3}}}{{\alpha_2}!~{\alpha_{3}}!}=0, \quad \forall \ell=1,2,\ldots,\brone.
\end{align*}
From the definition of $\brone$, this system does not have any non-trivial solutions, which contradicts to the aforementioned properties of $s_{i}$, $t_{1i}$ and $t_{2i}$. Consequently, not all the coefficients in the representations of $A_{n,1}/D_{1n}$, $A_{n,2}/D_{1n}$, $B_n/D_{1n}$ and $C_n/D_{1n}$ go to 0 as $n\to\infty$.

\textbf{Step 3 - Collapse of coefficients by Fatou's lemma.} In this step, we will point out a contradiction to the result in Step 2 by using Fatou's lemma. In particular, let us denote by $m_n$ the maximum of the absolute values of the coefficients in the representations of $A_{n,1}/D_{1n}$, $A_{n,2}/D_{1n}$, $B_n/D_{1n}$ and $C_n/D_{1n}$, i.e.
\begin{align*}
m_n=\max_{\substack{j\in[k_*], \ 0\leq|\alpha_1|\leq\Bar{r}(|\mathcal{A}_j|),\\ 0\leq\ell\leq 2(\brj-|\alpha_1|)}}\left\{\frac{|E^n_{\alpha_1,\ell}(j)|}{D_{1n}},\frac{|\lambdan-\lambdastar|}{D_{1n}}\right\},
\end{align*}
with a note that $E^n_{\zerod,0}(j):=\sum_{i\in\mathcal{A}_j}\lambdan p^n_i-\lambdastar p^*_j$. Since $|E^n_{\alpha_1,\ell}(j)|/(m_nD_{1n})$ and $|\lambdan-\lambdastar|/(m_nD_{1n})$ are bounded, we can replace them by their subsequences such that 
\begin{align*}
\frac{|E^n_{\alpha_1,\ell}(j)|}{m_n D_{1n}}\to\tau_{\alpha_1,\ell}(j), \qquad
        \frac{|\lambdan-\lambdastar|}{m_n D_{1n}}\to \tau,
\end{align*}
as $n\to\infty$ for all $j\in[k_*]$, $0\leq |\alpha_1|\leq\Bar{r}(|\mathcal{A}_j|)$ and $0\leq\ell\leq 2(\brj-|\alpha_1|)$. Here, at least one among $\tau_{\alpha_1,\ell}(j)$ and $\tau$ is different from zero. By applying the Fatou's lemma, we get
\begin{align*}
    0=\lim_{n\to\infty}\frac{2V(p_{\lambdan, G_n},p_{\lambdastar, G_*})}{m_nD_{1n}}\geq\int\liminf_{n\to\infty}\dfrac{|p_{\lambdan, G_n}(X,Y)-p_{\lambdastar, G_*}(X,Y)|}{m_nD_{1n}}\dint (X,Y)\geq 0,
\end{align*}
which implies that 
\begin{align*}
    \dfrac{|p_{\lambdan, G_n}(X,Y)-p_{\lambdastar, G_*}(X,Y)|}{m_nD_{1n}}\to 0,
\end{align*}
for almost surely $(X,Y)$. Recall that the left hand side in the above equation converges to
\begin{align*}
    \sum_{j,\alpha_1,\ell}\tau_{\alpha_1,\ell}(j)X^{\alpha_1}\cdot\frac{\partial^{|\alpha_1|+\ell}f}{\partial h_1^{|\alpha_1|+\ell}}(Y|(\aj)^{\top}X+\bj,\sj)\Bar{f}(X)+\tau g_{0}(Y|X)\Bar{f}(X),
\end{align*}
where the ranges of $(j,\alpha_1,\ell)$ in the summation are $j\in[k_*]$, $0\leq|\alpha_1|\leq \brj$ and  $0\leq \ell\leq 2(\brj-|\alpha_1|)$. As a result, we get
\begin{align}
\label{eq:distinguishable_dependent_3}
\sum_{j,\alpha_1,\ell}\tau_{\alpha_1,\ell}(j)X^{\alpha_1}\cdot\frac{\partial^{|\alpha_1|+\ell}f}{\partial h_1^{|\alpha_1|+\ell}}(Y|(\aj)^{\top}X+\bj,\sj)+\tau g_{0}(Y|X)=0,
\end{align}
for almost surely $(X,Y)$. Since $p_{G_*}$ is distinguishable from $g_0$, it follows from Definition~\ref{definition:distinguishability} that $\tau_{\alpha_1,\ell}(j)X^{\alpha_1}=\tau=0$, for any $j\in[k_*]$, $0\leq|\alpha_1|\leq \brj$ and  $0\leq \ell\leq 2(\brj-|\alpha_1|)$ for almost surely $X$. This result indicates that $\tau_{\alpha_1,\ell}(j)=\tau=0$, which contradicts to the fact that at least one among $\tau_{\alpha_1,\ell}(j),\tau$ is non-zero. Hence, we obtain the local inequality in equation~\eqref{eq:distinguishable_dependent_local}.

As a consequence, there exists some $\varepsilon'>0$ such that
\begin{align*}
    \inf_{\substack{\lambda\in[0,1],G\in\Ocal_{k,\xi}(\Theta):\\ D_1((\lambda, G),(\lambdastar, G_*))\leq\varepsilon'}}V(p_{\lambda, G},p_{\lambdastar,\Gstar})/D_1((\lambda, G),(\lambdastar,\Gstar))>0.
\end{align*}
\textbf{Global inequality:} Thus, it remains to prove the following global inequality:
\begin{align*}
    \inf_{\substack{\lambda\in[0,1],G\in\Ocal_{k,\xi}(\Theta):\\ D_1((\lambda, G),(\lambdastar, G_*))>\varepsilon'}}V(p_{\lambda, G},p_{\lambdastar,\Gstar})/D_1((\lambda, G),(\lambdastar,\Gstar))>0.
\end{align*}
Assume by contrary that it is not the case. Then, there exist some sequences $G'_n\in\Ocal_{k,\xi}(\Theta)$ and $\lambda'_n\in[0,1]$ such that 
\begin{align*}
    V(p_{\lambda'_n,G'_n},p_{\lambdastar, G_*})/D_1((\lambda'_n,G'_n),(\lambdastar, G_*))\to 0,\\
    D_1((\lambda'_n,G'_n),(\lambdastar, G_*))>\varepsilon'.
\end{align*}
As a result, we get $ V(p_{\lambda'_n,G'_n},p_{\lambdastar, G_*})\to 0$. Note that $\Theta$ and $[0,1]$ are bounded sets,  then we can find a subsequence of $G'_n$ and a subsequence of $\lambda'_n$ such that $G'_n\to G'$ and $\lambda'_n\to\lambda'$, where $G'\in\Ocal_{k,\xi}(\Theta)$ and $\lambda'\in[0,1]$. By replacing $G'_n$ and $\lambda'_n$ with these subsequences, we get that $D_1((\lambda'_n,G'_n),(\lambdastar, G_*))>\varepsilon'$. Moreover, by the Fatou's lemma, we obtain that
\begin{align*}
    0=\lim_{n\to\infty}2V(p_{\lambda'_n,G'_n},p_{\lambdastar, G_*})&\geq \int\liminf_{n\to\infty}\left|p_{\lambda'_n,G'_n}(X,Y)-p_{\lambdastar, G_*}(X,Y)\right|\dint (X,Y)\\
    &=\int\left|p_{\lambda',G'}(X,Y)-p_{\lambdastar, G_*}(X,Y)\right|\dint(X,Y)\geq 0,
\end{align*}
which indicates that $p_{\lambda',G'}(X,Y)=p_{\lambdastar, G_*}(X,Y)$ for almost surely $(X,Y)$. According to Proposition~\ref{prop:identifiable}, the deviated Gaussian mixture of experts is identifiable when $p_{G_*}$ is distinguishable from $g_0$. Thus, it follows that $(\lambda',G')\equiv (\lambdastar, G_*)$. This contradicts to the previous claim that $D_1((\lambda',G'),(\lambdastar, G_*))>\varepsilon'>0$. Hence, the proof is completed.

\subsection{Proof of Theorem~\ref{theorem:partial_dependent}}
\label{appendix:partial_dependent}
Similar to the proof of Theorem~\ref{theorem:distinguishable_dependent}, we need to prove the following claim:
\begin{align*}
    \inf_{\substack{\lambda\in[0,1],G\in\Ocal_{k,\xi}(\Theta)}}\dfrac{V(p_{\lambda, G},p_{\lambdastar, G_*})}{D_{2}((\lambda, G),(\lambdastar, G_*))}>0.
\end{align*}
\textbf{Local inequality.} Firstly, we will demonstrate the local version of the above inequality:
\begin{align}   
    \label{eq:partial_dependent_local}
    \lim_{\varepsilon\to 0}\inf_{\substack{\lambda\in[0,1],G\in\Ocal_{k,\xi}(\Theta),\\ D_{2}((\lambda, G),(\lambdastar, G_*))\leq\varepsilon}}\dfrac{V(p_{\lambda, G},p_{\lambdastar, G_*})}{D_{2}((\lambda, G),(\lambdastar, G_*))}>0.
\end{align}
Assume by contrary that the above inequality does not hold, then there exist sequences $\lambdan\in[0,1]$ and $G_n=\sum_{i=1}^{k_n}p^n_i\delta_{(\ain,\bin,\sigmain)}\in\Ocal_{k,\xi}(\Theta)$ such that
\begin{align*}
    \begin{cases}
        D_{2n}:=D_{2}((\lambdan, G_n),(\lambdastar, G_*))\to 0,\\
        V(p_{\lambdan, G_n},p_{\lambdastar, G_*})/D_{2n}\to 0.
    \end{cases}
\end{align*}
\textbf{Case 1:} $\lambdan\leq\lambdastar$ for infinitely $n\in\mathbb{N}$. WLOG, we assume that $\lambdan\leq\lambdastar$ for all $n\in\mathbb{N}$. 

In this case, we have $D_{2n}=D_{1}((\lambdan, G_n),(\lambdastar, G_*))$, for any $n\in\mathbb{N}$. Note that $k_n\leq k$, thus, we can replace $G_n$ with one of its subsequences such that $k_n$ does not vary with $n$. Therefore, we assume that $k_n=k$ for all $n$. In addition, the Voronoi cells $\mathcal{A}_j=\mathcal{A}_j^n$ does not change with $n$ for all $j\in[k_*]$. Next, we decompose the quantity $p_{\lambdan, G_n}(X,Y)-p_{\lambdastar, G_*}(X,Y)$ as follows:
\begin{align*}
    p_{\lambdan, G_n}(X,Y)-p_{\lambdastar, G_*}(X,Y)&=(\lambdastar-\lambdan)\sum_{j=\Bar{k}+1}^{k_0}p^0_jf(Y|(a^0_j)^{\top}X+b^0_j,\sigma^0_j)\Bar{f}(X)\\
    &+\sum_{j:|\mathcal{A}_j|>1}\sum_{i\in \mathcal{A}_j}\lambdan p^n_i[f(Y|(\ain)^{\top}X+\bin,\sigmain)-f(Y|(\aj)^{\top}X+\bj,\sj)]\Bar{f}(X)\\
    &+\sum_{j:|\mathcal{A}_j|=1}\sum_{i\in \mathcal{A}_j}\lambdan p^n_i[f(Y|(\ain)^{\top}X+\bin,\sigmain)-f(Y|(\aj)^{\top}X+\bj,\sj)]\Bar{f}(X)\\
    &+\sum_{j=1}^{k_*}\left(\sum_{i\in\mathcal{A}_j}\lambdan p^n_i-\Bar{p}^*_j(\lambdan)\right)f(Y|(\aj)^{\top}X+\bj,\sj)\Bar{f}(X)\\
    &:=C_n+A_{n,1}+A_{n,2}+B_n,
\end{align*}
where we define $\Bar{p}^*_j(\lambdan):=\begin{cases}
    \lambdastar p^*_j+(\lambdan-\lambdastar)p^0_j, \quad j\in[\Bar{k}]\\
    \lambdastar p^*_j, \hspace{2.55cm} j\in[k_*]\setminus[\kbar]
\end{cases}$.

By applying the Taylor expansions as in Appendix~\ref{appendix:distinguishable_dependent}, we are able to show that $A_{n,1}/D_{2n}$, $A_{n,2}/D_{2n}$, $B_{n}/D_{2n}$ and $C_{n}/D_{2n}$ can be written as linear combinations of elements of the following set
\begin{align}
    \label{eq:set_H2}
    \mathcal{H}_2:=\Bigg\{X^{\alpha_1}\cdot\frac{\partial^{|\alpha_1|+\ell}f}{\partial h_1^{|\alpha_1|+\ell}}&(Y|(\aj)^{\top}X+\bj,\sj)\Bar{f}(X), \ f(Y|(a^0_{j'})^{\top}X+b^0_{j'},\sigma^0_{j'})\Bar{f}(X):j\in[k_*],\nonumber\\ 
&j'\in[k_0]\setminus[\kbar],\ 0\leq |\alpha_1|\leq \brj, \ 0\leq \ell\leq 2(\brj-|\alpha_1|)\Bigg\}.
\end{align}
Furthermore, not all the coefficients in these representations go to zero as $n$ tends to infinity.

Subsequently, by following the same arguments for deriving equation~\eqref{eq:distinguishable_dependent_3}, we can find some constants $\tau_{\alpha_1,\ell}(j)$ and $\tau(j')$, where $j\in[k_*]$, $0\leq |\alpha_1|\leq\Bar{r}(|\mathcal{A}_j|)$, $0\leq \ell\leq 2(\brj-|\alpha_1|)$ and $j'\in[k_0]\setminus[\kbar]$, such that at least one among them is non-zero and
\begin{align}
\label{eq:H2_equation}
\sum_{j,\alpha_1,\ell}\tau_{\alpha_1,\ell}(j)X^{\alpha_1}\cdot\frac{\partial^{|\alpha_1|+\ell}f}{\partial h_1^{|\alpha_1|+\ell}}(Y|(\aj)^{\top}X+\bj,\sj)+\sum_{j'=\Bar{k}+1}^{k_0}\tau(j')f(Y|(a^0_{j'})^{\top}X+b^0_{j'},\sigma^0_{j'}) =0,
\end{align}
for almost surely $(X,Y)$. Now, we demonstrate that the set $\mathcal{H}_2$ is linearly independent with respect to $X$ and $Y$, or equivalently, $\tau_{\alpha_1,\ell}(j)=\tau(j')=0$, for any  $j\in[k_*]$, $0\leq |\alpha_1|\leq\Bar{r}(|\mathcal{A}_j|)$, $0\leq \ell\leq 2(\brj-|\alpha_1|)$ and $j'\in[k_0]\setminus[\kbar]$. Indeed, equation~\eqref{eq:H2_equation} can be rewritten as
\begin{align}
    \label{eq:new_H2_equation}
    &\sum_{j=1}^{k_*}\sum_{v=0}^{2\brj}\left(\sum_{|\alpha_1|+\ell=v}\tau_{\alpha_1,\ell}(j)X^{\alpha_1}\right)\frac{\partial^{v}f}{\partial h_1^{v}}(Y|(\aj)^{\top}X+\bj,\sj)+\sum_{j'=\Bar{k}+1}^{k_0}\tau(j')f(Y|(a^0_{j'})^{\top}X+b^0_{j'},\sigma^0_{j'}) =0,
\end{align}
for almost surely $X$ and $Y$. It is worth noting that $(\aj,\bj,\sj)$ and $(a^0_{j'},b^0_{j'},\sigma^0_{j'})$ are distinct components for $j\in[k_*]$ and $j\in[k_0]\setminus[\kbar]$, therefore, $((\aj)^{\top}X+\bj,\sj)$ and $((a^0_{j'})^{\top}X+b^0_{j'},\sigma^0_{j'})$ are distinct pairs for almost surely $X\in\mathcal{X}$. This implies that $\frac{\partial^{v}f}{\partial h_1^{v}}(Y|(\aj)^{\top}X+\bj,\sj)$ and $f(Y|(a^0_{j'})^{\top}X+b^0_{j'},\sigma^0_{j'})$ are linearly independent with respect to $Y$ for $0\leq v\leq 2\brj$ for any $j\in[k_*]$ and $j'\in[k_0]\setminus[\kbar]$. Then, it follows from equation~\eqref{eq:new_H2_equation} that $\tau(j')=0$ for any $j'\in[k_0]\setminus[\kbar]$ and $\sum_{|\alpha_1|+\ell=u}\tau_{\alpha_1,\ell}(j)X^{\alpha_1}=0$ for any $j\in[k_*]$, $0\leq v\leq 2\brj$ for almost surely $X$. Note that this is a polynomial of $X\in\mathcal{X}$, which is a bounded subset of $\mathbb{R}^d$, we deduce that $\tau_{\alpha_1,\ell}(j)=0$ for all $|\alpha_1|+\ell=v$, $j\in[k_*]$ and $0\leq v\leq 2\brj$. This contradicts the previous claim that at least one among $\tau_{\alpha_1,\ell}(j),\tau(j')$ is different from zero. 

Thus, we obtain the local inequality in equation~\eqref{eq:partial_dependent_local} for this case.

\textbf{Case 2:} $\lambdan>\lambdastar$ for infinitely $n\in\mathbb{N}$. WLOG, we assume that $\lambdan>\lambdastar$ for all $n\in\mathbb{N}$.

\textbf{Case 2.1:} $k\leq k_*+k_0-\Bar{k}-1$

In this case, the discrepancy $D_{2n}$ reduces to $D_{1}((\lambdan, G_n),(\lambdastar, G_*))$. Therefore, the local inequality for this case can be achieved analogously to that for Case 1.

\textbf{Case 2.2:} $k\geq \kstar+k_0-\kbar$

In this case, the discrepancy $D_{2n}$ equals to $D_3(G_n,\overline{G}_*(\lambdan))$, which was defined in equation~\eqref{eq:D_partial_dependent}. Additionally, we have
\begin{align*}
    &p_{\lambdan, G_n}(X,Y)-p_{\lambdastar, G_*}(X,Y)=\lambdan\Bigg\{\sum_{j=1}^{k_*}\sum_{i\in \mathcal{A}_j} p^n_if(Y|(\ain)^{\top}X+\bin,\sigmain)\\
    &-\Big[\Big(1-\frac{\lambdastar}{\lambdan}\Big)\sum_{j=1}^{k_0}p^0_jf(Y|(a^0_j)^{\top}X+b^0_j,\sigma^0_j)+\frac{\lambdastar}{\lambdan}\sum_{j=1}^{\kstar}p^*_jf(Y|(\aj)^{\top}X+\bj,\sj)\Big]\Bigg\}\Bar{f}(X)\\
    &=~\lambdan\Big[p_{G_n}(X,Y)-p_{\overline{G}_*(\lambdan)}(X,Y)\Big].
\end{align*}
Recall that 
\begin{align*}
    0=\lim_{n\to\infty}\dfrac{2V(p_{\lambdan, G_n},p_{\lambdastar, G_*})}{D_{2n}}&=\lim_{n\to\infty}\dfrac{\int|p_{\lambdan, G_n}(X,Y)-p_{\lambdastar, G_*}(X,Y)|d(X,Y)}{D_3(G_n,\overline{G}_*(\lambdan))}\\
    &=\lim_{n\to\infty}\lambdan\cdot\dfrac{\int|p_{G_n}(X,Y)-p_{\overline{G}_*(\lambdan)}(X,Y)|d(X,Y)}{D_3(G_n,\overline{G}_*(\lambdan))}\\
    &=\lim_{n\to\infty}\lambdan\cdot\dfrac{2V(p_{G_n},p_{\overline{G}_*(\lambdan)})}{D_3(G_n,\overline{G}_*(\lambdan))}
\end{align*}
Since $\lambdan>\lambdastar>0$ for all $n\in\mathbb{N}$, we get $V(p_{G_n},p_{\overline{G}_*(\lambdan)})/D_3(G_n,\overline{G}_*(\lambdan))\to 0$ as $n\to\infty$. For the sake of presentation, we represent the mixing measure $\overline{G}_*(\lambdan)=\Big(1-\frac{\lambdastar}{\lambdan}\Big)G_0+\frac{\lambdastar}{\lambdan}G_*$ as
\begin{align*}
    \overline{G}_*(\lambdan)=\sum_{j=1}^{k_*+k_0-\Bar{k}}(p^n_j)'\delta_{(a'_{j},b'_{j},\sigma'_{j})}\in\mathcal{E}_{\kstar+k_0-\kbar}(\Theta).
\end{align*}
Next, let us define Voronoi cells used for this case as follows:
\begin{align*}
    \mathcal{B}^n_j=\mathcal{B}_j(G_n)=\{i\in[k_n]:\|\theta^n_{i}-\theta'_{j}\|\leq\|\theta^n_{i}-\theta'_{\ell}\|, \ \forall \ell\neq j\}, 
\end{align*}
where $\theta^n_{i}=(\ain,\bin,\sigmain)$ and $\theta'_{j}=(a'_{i},b'_{i},\sigma'_{i})$ for any $j\in[\kstar+k_0-\kbar]$. Since $k_n\leq k$, there exists a subsequence of $G_n$ such that $k_n$ does not vary with $n$. Thus, by replacing $G_n$ with this subsequence, we can assume that $k_n=k$ for all $n$. In addition, $\mathcal{B}_j=\mathcal{B}^n_j$ does not change with $n$ for all $j\in[k_*+k_0-\Bar{k}]$. Then, we can rewrite the difference $p_{G_n}(X,Y)-p_{\overline{G}_*(\lambdan)}(X,Y)$ as follows:
\begin{align*}
   p_{G_n}(X,Y)-p_{\overline{G}_*(\lambdan)}(X,Y)=&\sum_{j:|\mathcal{B}_j|>1|}\sum_{i\in\mathcal{B}_j}p^n_i\Big[f(Y|(\ain)^{\top}X+\bin,\sigmain)-f(Y|(a'_{j})^{\top}X+b'_{j},\sigma'_{j})\Big]\\
    +&\sum_{j:|\mathcal{B}_j|=1|}\sum_{i\in\mathcal{B}_j}p^n_i\Big[f(Y|(\ain)^{\top}X+\bin,\sigmain)-f(Y|(a'_{j})^{\top}X+b'_{j},\sigma'_{j})\Big]\\
    +&\sum_{j=1}^{k_*+k_0-\Bar{k}}\Big(\sum_{i\in\mathcal{B}_j}p^n_i-p'_j\Big)f(Y|(a'_{j})^{\top}X+b'_{j},\sigma'_{j})
\end{align*}
By abuse of notation, we denote three terms in the above summation as $A_{n,1}$, $A_{n,2}$ and $B_n$, respectively. By invoking the Taylor expansions as in Appendix~\ref{appendix:distinguishable_dependent}, we get that $A_{n,1}$, $A_{n,2}$ and $B_n$ can be treated as linear combinations of elements of the following set:
\begin{align*}
    \mathcal{H}:=\Bigg\{X^{\alpha_1}\cdot\frac{\partial^{|\alpha_1|+\ell}f}{\partial h_1^{|\alpha_1|+\ell}}(Y|(\apj)^{\top}X+\bpj,\spj)\Bar{f}(X):\ j\in[k_*+k_0-\kbar],\ 0\leq |\alpha_1|\leq \brbj,\nonumber\\ 
 \ 0\leq \ell\leq 2(\brbj-|\alpha_1|)\Bigg\}.
\end{align*}
Moreover, not all the coefficients in the representations of $A_{n,1}/D_{2n}$, $A_{n,2}/D_{2n}$ and $B_{n}/D_{2n}$ approach zero as $n\to\infty$. Additionally, we can utilize the same arguments for deriving equation~\eqref{eq:distinguishable_dependent_3} to deduce that there exist some constants $\tau_{\alpha_1,\ell}(j)$, where $j\in[\kstar+k_0-\kbar]$, $0\leq |\alpha_1|\leq \brbj$ and $0\leq \ell\leq 2(\brbj-|\alpha_1|)$ that satisfy at least one among them is different from zero and 
\begin{align*}
    \sum_{j,\alpha_1,\ell}\tau_{\alpha_1,\ell}(j)X^{\alpha_1}\cdot\frac{\partial^{|\alpha_1|+\ell}f}{\partial h_1^{|\alpha_1|+\ell}}(Y|(\apj)^{\top}X+\bpj,\spj)=0,
\end{align*}
for almost surely $(X,Y)$. Since $\mathcal{H}$ is a linearly independent set, which can be proved in a similar way as for the set $\mathcal{H}_2$ in Case 1, the above equation leads to $\tau_{\alpha_1,\ell}(j)$ for any $j\in[\kstar+k_0-\kbar]$, $0\leq |\alpha_1|\leq \brbj$ and $0\leq \ell\leq 2(\brbj-|\alpha_1|)$. This contradicts with the result that at least one among $\tau_{\alpha_1,\ell}(j)$ is non-zero. Hence, we achieve the local inequality in equation~\eqref{eq:partial_dependent_local}. 

As a consequence, there exists a positive constant $\varepsilon'$ that satisfies
\begin{align*}
    \inf_{\substack{\lambda\in[0,1],G\in\Ocal_{k,\xi}(\Theta),\\ D_{2}((\lambda, G),(\lambdastar, G_*))\leq\varepsilon}}\dfrac{V(p_{\lambda, G},p_{\lambdastar, G_*})}{D_{2}((\lambda, G),(\lambdastar, G_*))}>0.
\end{align*}

\textbf{Global inequality.} Now, it suffices to show that
\begin{align}
\label{eq:partial_independent_global}
\inf_{\substack{\lambda\in[0,1],G\in\Ocal_{k,\xi}(\Theta),\\D_{2}((\lambda, G),(\lambdastar, G_*))>\varepsilon'}}\dfrac{V(p_{\lambda, G},p_{\lambdastar, G_*})}{D_{2}((\lambda, G),(\lambdastar, G_*))}>0.
\end{align}
Assume by contrary that the above claim does not hold, then there exist sequences $({\lambda}^{\prime}_{n})\subset[0,1]$ and $(G^{\prime}_{n})\subset\Ocal_{k,\xi}(\Theta)$ that satisfy
\begin{align*}
    \begin{cases}
        D_{2}(({\lambda}^{\prime}_{n}, G^{\prime}_{n}),(\lambdastar, G_*))>\varepsilon^{\prime},\\
        V(p_{{\lambda}^{\prime}_{n}, G^{\prime}_{n}},p_{\lambdastar, G_*})/D_{2}(({\lambda}^{\prime}_{n}, G^{\prime}_{n}),(\lambdastar, G_*))\to 0,
    \end{cases}
\end{align*}
which leads to the fact that $V(p_{{\lambda}^{\prime}_{n}, G^{\prime}_{n}},p_{\lambdastar, G_*})\to 0$ as $n\to\infty$.

\textbf{Case 1:} $\lambda^{\prime}_{n}\leq\lambdastar$ for infinitely $n\in\mathbb{N}$. WLOG, we assume that $\lambda^{\prime}_{n}\leq\lambdastar$ for all $n\in\mathbb{N}$.

In this case, we have $D_{2}(({\lambda}^{\prime}_{n}, G^{\prime}_{n}),(\lambdastar, G_*))=D_1(({\lambda}^{\prime}_{n}, G^{\prime}_{n}),(\lambdastar, G_*))>\varepsilon^{\prime}$. Since the sets $\Theta$ and $[0,1]$ are bounded, we can find a subsequence of ${G}'_n$ and a subsequence of ${\lambda}'_n$ such that ${G}'_n\to {G}'$ and ${\lambda}'_n\to{\lambda}'$, where ${G}'\in\Ocal_{k,\xi}(\Theta)$ and ${\lambda}'\in[0,1]$. By replacing ${G}'_n$ and ${\lambda}'_n$ with those subsequences, we get that $D_1((\lambda^{\prime}, G^{\prime}),(\lambdastar, G_*))>\varepsilon^{\prime}$. On the other hand, the result that $V(p_{{\lambda}^{\prime}_{n}, G^{\prime}_{n}},p_{\lambdastar, G_*})\to 0$ as $n\to\infty$ implies that $V(p_{{\lambda}^{\prime}, G^{\prime}},p_{\lambdastar, G_*})=0$, which leads to $$p_{{\lambda}^{\prime}, G^{\prime}}(X,Y)=p_{\lambdastar, G_*}(X,Y),$$ for almost surely $(X,Y)$. Since $\lambda^{\prime}_n\leq\lambdastar$, we get $\lambda^{\prime}\leq\lambdastar$. It is worth noting that if $\lambda^{\prime}<\lambdastar$, $\overline{G}_{*}(\lambda^{\prime})$ is not valid mixing measure. Therefore, we obtain  $(\lambda^{\prime},G^{\prime})\equiv(\lambdastar, G_*)$ in this scenario. If $\lambda^{\prime}=\lambdastar$, then $\overline{G}_{*}(\lambda^{\prime})\equiv G_*$ and the above identifiability equation also admits $(\lambda^{\prime},G^{\prime})\equiv(\lambdastar, G_*)$ as a solution. Thus, it follows that $D_1((\lambda^{\prime}, G^{\prime}),(\lambdastar, G_*))=0$, which contradicts the result that $D_1((\lambda^{\prime}, G^{\prime}),(\lambdastar, G_*))>\varepsilon^{\prime}>0$. Hence, the global inequality~\eqref{eq:partial_independent_global} holds true in this case.

\textbf{Case 2:} $\lambda^{\prime}_{n}>\lambdastar$ for infinitely $n\in\mathbb{N}$. WLOG, we assume that $\lambda^{\prime}_{n}>\lambdastar$ for all $n\in\mathbb{N}$.

\textbf{Case 2.1:} $k\leq \kstar+k_0-\kbar-1$

In this case, we also have $D_{2}(({\lambda}^{\prime}_{n}, G^{\prime}_{n}),(\lambdastar, G_*))=D_1(({\lambda}^{\prime}_{n}, G^{\prime}_{n}),(\lambdastar, G_*))>\varepsilon^{\prime}$. Similar to Case 1, we get $p_{{\lambda}^{\prime}, G^{\prime}}(X,Y)=p_{\lambdastar, G_*}(X,Y)$ for almost surely $(X,Y)$, where $\lambda'\in[0,1]$ and $G'\in\Ocal_{k,\xi}(\Theta)$ are the limits of $\lambda'_n$ and $G'_n$ as $n$ goes to infinity, respectively. 

Under this setting, the previous identifiability equation admits either $(\lambda',G')\equiv (\lambdastar, G_*)$ or $(\lambda',G')\equiv(\lambda',\overline{G}_*(\lambda'))$ for any $\lambda'\in[0,1]$ as a solution as mentioned in Section~\ref{sec:distinguishable}. However, as $\overline{G}_*(\lambda')$ has $k_*+k_0-\kbar$ components, which is higher than that of $G'$ which has no more than $k$ components. As a result, we obtain that $(\lambda^{\prime},G^{\prime})\equiv(\lambdastar, G_*)$, leading to $D_1((\lambda^{\prime}, G^{\prime}),(\lambdastar, G_*))=0$, which is a contradiction to the fact that $D_1(({\lambda}^{\prime}, G^{\prime}),(\lambdastar, G_*))>\varepsilon^{\prime}>0$.

\textbf{Case 2.2:} $k\geq \kstar+k_0-\kbar$

In this case, we have $D_{2}(({\lambda}^{\prime}_{n}, G^{\prime}_{n}),(\lambdastar, G_*))=D_3( G^{\prime}_{n},\overline{G}(\lambda^{\prime}_{n}))>\varepsilon^{\prime}$. Similar to Case 1, we may replace $G'_n$ and $\lambda'_n$ with their subsequences whose limits are $G'\in\Ocal_{k,\xi}(\Theta)$ and $\lambda'\in[0,1]$, respectively. Then, we get $D_3(G^{\prime},\overline{G}(\lambda^{\prime}))>\varepsilon^{\prime}$. Additionally, we also achieve the identifiability equation $p_{{\lambda}^{\prime}, G^{\prime}}(X,Y)=p_{\lambdastar, G_*}(X,Y)$ for almost surely $(X,Y)$. Note that in this case, $G'$ has more components than $G_*$, then the previous equation admits only $(\lambda^{\prime}, \overline{G}(\lambda^{\prime}))$ as a solution. Therefore, we obtain that $D_3(G^{\prime},\overline{G}(\lambda^{\prime}))=0$, which contradicts the result that $D_3(G^{\prime},\overline{G}(\lambda^{\prime}))>\varepsilon^{\prime}>0$.

Hence, we reach the global inequality in equation~\eqref{eq:partial_independent_global}, and the proof is totally completed.

%%%%%%%%%%%%%%%%%%%%%%%%%%%%%%%%%%%%%%%

\section{PROOF OF AUXILIARY RESULTS}
\label{appendix:proof_auxiliary_results}
In this appendix, we provide proofs for Proposition~\ref{prop:identifiable} and Proposition~\ref{prop:mle_estimation} in Appendix~\ref{appendix:identifiable} and Appendix~\ref{appendix:prop:mle_estimation}, respectively.
\subsection{Proof of Proposition~\ref{prop:identifiable}}
\label{appendix:identifiable}
Firstly, we suppose that $G=\sum_{i=1}^{k}p_i\delta_{(a_i,b_i,\sigma_i)}$ and $G'=\sum_{i=1}^{k'}p'_i\delta_{(a'_i,b'_i,\sigma'_i)}$ are two mixing measures such that the equation $p_{\lambda, G}(X,Y)=p_{\lambda', G'}(X,Y)$ holds for almost surely $(X,Y)\in\mathcal{X}\times\mathcal{Y}$. This equation can be expanded as
\begin{align*}
    (\lambda'-\lambda)g_0(Y|X)+\sum_{i=1}^{k}\lambda p_if(Y|a_{i}^{\top}X+b_{i},\sigma_{i})-\sum_{i=1}^{k'}\lambda' p'_if(Y|(a'_i)^{\top}X+b'_i,\sigma'_i)=0,
\end{align*}
for almost surely $(X,Y)$. Now, assume that $G$ and $G'$ share only $\ell$ components in common, where $0\leq \ell \leq\min\{k,k'\}$, e.g. $(a_i,b_i,\sigma_i)=(a'_i,b'_i,\sigma'_i)$ for any $i\in[\ell]$. Then, we rewrite the above equation as
\begin{align*}
    (\lambda'-\lambda)g_0(Y|X)+\sum_{i=1}^{\ell}(\lambda p_i-\lambda'p'_{i})f(Y|a_{i}^{\top}X+b_{i},\sigma_{i})+\sum_{i=\ell+1}^{k}\lambda p_if(Y|a_{i}^{\top}X+b_{i},\sigma_{i})\\
    -\sum_{i=\ell+1}^{k'}\lambda' p'_if(Y|(a'_i)^{\top}X+b'_i,\sigma'_i)=0,
\end{align*}
for almost surely $(X,Y)$. Next, we consider a mixing measure $G_{\ell}$ which has $k+k'-\ell$ components $(a_1,b_1,\sigma_1),\ldots,(a_{k},b_{k},\sigma_{k})$, $(a'_{\ell+1},b'_{\ell+1},\sigma'_{\ell+1}),\ldots,(a'_{k'},b'_{k'},\sigma'_{k'})$. Since $p_{G_{\ell}}$ is distinguishable from $g_0$, we see that if either $k\neq k'$ or $0\leq\ell<\min\{k,k'\}$, then there exists an index $i$ such that $p_i=0$ and/or $p'_i=0$, which contradicts to the fact that $p_i,p'_i>0$. As a result, we must have $k=k'$ and $\ell=k$, i.e. $(a_i,b_i,\sigma_i)=(a'_i,b'_i,\sigma'_i)$ for all $i\in[k]$. Given this result, the above equation is equivalent to
\begin{align*}
    (\lambda'-\lambda)g_0(Y|X)+\sum_{i=1}^{k}(\lambda p_i-\lambda'p'_i)f(Y|a_i^{\top}X+b_i,\sigma_i)=0,
\end{align*}
for almost surely $(X,Y)$. Note that $p_{G}$ is distinguishable from $g_0$, then we obtain that $\lambda-\lambda'=\lambda p_i-\lambda'p'_i=0$, which implies that $\lambda=\lambda'$ and $p_i=p'_i$ for any $i\in[k]$. As a result, it follows that $(\lambda,G)\equiv(\lambda',G')$.

Hence, the proof is completed.

\subsection{Proof of Proposition~\ref{prop:mle_estimation}}
\label{appendix:prop:mle_estimation}
First of all, let us introduce some necessary notations used throughout this appendix. In particular, we denote $\mathcal{P}_K([0,1]\times\Theta):=\{p_{\lambda, G}(X,Y):\lambda\in[0,1], \ 
G\in\mathcal{O}_{K,\xi}(\Theta)\}$. Additionally, we define
\begin{align*}
    \overline{\mathcal{P}}_K([0,1]\times\Theta)&=\{\overline{p}_{\lambda,G}=(p_{\lambda,G}+p_{\lambdastar,G_*})/2:(\lambda,G)\in[0,1]\times\Ocal_{k,\xi}(\Theta)\},\\
    \overline{\mathcal{P}}^{1/2}_{K}([0,1]\times\Theta)&=\{\Bar{p}^{1/2}_{\lambda, G}:\Bar{p}_{\lambda, G}\in\overline{\mathcal{P}}_K([0,1]\times\Theta)\}.
\end{align*}
To derive a convergence rate for the joint density estimators under the Hellinger distance, we require a condition on the complexity of the following class introduced in \cite{Vandegeer-2000}:
\begin{align*}
    \overline{\mathcal{P}}_K^{1/2}([0,1]\times\Theta,\varepsilon):=\{\Bar{p}^{1/2}_{\lambda, G}\in\overline{\mathcal{P}}^{1/2}_{K}(\Theta): \ h(\Bar{p}_{\lambda, G},p_{\lambdastar, G_*})\leq\varepsilon\},
\end{align*}
The complexity of this class can be captured by the following bracketing entropy integral
\begin{align*}
    \mathcal{J}_B(\varepsilon,\overline{\mathcal{P}}_K^{1/2}([0,1]\times\Theta,\varepsilon),m) =\int_{\varepsilon^2/2^{13}}^{\varepsilon}\sqrt{\log H_B(u,\overline{\mathcal{P}}_K^{1/2}([0,1]\times\Theta,\varepsilon),m)}du\vee\varepsilon,
\end{align*}
where $u\vee \varepsilon=\max\{u,\varepsilon\}$ and  $H_B(\varepsilon,\mathcal{P},m)$ represents the $\varepsilon$-bracketing entropy of a set $\mathcal{P}$ under the Lebesgue measure $m$ (readers are referred to \cite{Vandegeer-2000} for more detail about the definition of this term). Interestingly, we discover a connection between this quantity and the convergence of the density estimators as follows:
\begin{lemma} \label{lemma:density_estimation_Vandegeer}
Given a universal constant $J>0$, assume that we can find a natural number $N$, possibly depending on $\Theta$ and $k$, such that for all $n\geq N$ and $\varepsilon>\sqrt{\log n/n}$, the following holds:
\begin{align}\label{eq:convergence_condition}
    \mathcal{J}_B(\varepsilon,\overline{\mathcal{P}}_K^{1/2}([0,1]\times\Theta,\varepsilon),m)\leq J\sqrt{n}\varepsilon^2.
\end{align}
Then, there exists a positive constant $c$ that depends only on $\Theta$ such that for all $n\in\mathbb{N}$, we have
\begin{align}
\Prob\parenth{h(p_{\hat{\lambda}_n,\widehat{G}_{n}}, p_{\lambdastar, G_{*}}) > \delta} \leq c\exp\parenth{-\frac{n\delta^2}{c^2}}. \nonumber
\end{align}
\end{lemma}
Proof of Lemma~\ref{lemma:density_estimation_Vandegeer} is in Appendix~\ref{appendix:lemma:bracketing_control}. As a consequence, in order to guarantee that our estimators will converge, it is sufficient to satisfy the condition in equation~\eqref{eq:convergence_condition}. For that purpose, we need to introduce a result regarding the upper bounds of the covering number $N(\varepsilon,\mathcal{P}_k(\Theta,[0,1]),\|\cdot\|_{\infty})$ and the bracketing entropy $H_B(\varepsilon,\mathcal{P}_k(\Theta,[0,1]),h)$ of the metric space $\mathcal{P}_k(\Theta,[0,1])$ in the following lemma:

\begin{lemma}\label{lemma:bracketing_control}
Suppose that $\Theta_{1}$ and $\Theta_{2}$ are respectively two bounded subsets of $\mathbb{R}^{q_{1}}$ and $\mathbb{R}^{q_{2}}$. Then, for any $0 < \varepsilon < 1/2$, the following results hold
\begin{itemize}
\item[(i)] $\log N\parenth{\varepsilon, \mathcal{P}_{k}(\Theta\times [0,1]),\|.\|_{\infty}} \lesssim \log (1/\varepsilon)$,
\item[(ii)] $H_{B}(\varepsilon, \mathcal{P}_{k}([0,1]\times\Theta), h) \lesssim \log(1/\varepsilon)$. 
\end{itemize}
\end{lemma}
Proof of Lemma~\ref{lemma:bracketing_control} is in Appendix~\ref{appendix:lemma:density_estimation_Vandegeer}. Now, we are ready to provide the proof of Proposition~\ref{prop:mle_estimation} in Appendix~\ref{appendix:main_proof}. 
\subsubsection{Main Proof}
\label{appendix:main_proof}
Note that $\overline{\mathcal{P}}_k^{1/2}([0,1]\times\Theta,\delta)\subseteq\overline{\mathcal{P}}_k^{1/2}([0,1]\times\Theta)$, it follows from the definition of Hellinger distance that
\begin{align*}
    H_B(\delta,\overline{\mathcal{P}}_k^{1/2}([0,1]\times\Theta,\delta),\|\cdot\|_2)&\leq H_B(\delta,\overline{\mathcal{P}}_k^{1/2}([0,1]\times\Theta),\|\cdot\|_2)\\
    &=H_B\Big(\frac{\delta}{\sqrt{2}},\overline{\mathcal{P}}_k([0,1]\times\Theta),h\Big)\\
    &\leq H_B(\delta,\mathcal{P}_k([0,1]\times\Theta),h).
\end{align*}
According to part (ii) of Lemma~\ref{lemma:bracketing_control}, we find that
\begin{align*}
    H_B(\delta,\overline{\mathcal{P}}_k^{1/2}([0,1]\times\Theta,\delta),\|\cdot\|_2)\lesssim\log(1/\delta).
\end{align*}
Consequently, we deduce that
\begin{align*}
    \mathcal{J}_B(\varepsilon,\overline{\mathcal{P}}_k^{1/2}([0,1]\times\Theta,\delta),u)\lesssim\varepsilon[\log(2^{13}\varepsilon^2)]^{1/2}<n\varepsilon^2,
\end{align*}
for all $\varepsilon>\sqrt{\log n/n}$. By applying Lemma~\ref{lemma:density_estimation_Vandegeer} with $\delta=\sqrt{\log n/n}$, we obtain the desired result.

% For completeness, we respectively present the proofs for Lemma~\ref{lemma:density_estimation_Vandegeer} and Lemma~\ref{lemma:bracketing_control} below.
\subsubsection{Proof of Lemma~\ref{lemma:density_estimation_Vandegeer}}
\label{appendix:lemma:density_estimation_Vandegeer}
It follows from Lemma 4.1 and Lemma 4.2 in \cite{Vandegeer-2000} that
    \begin{align*}
        \frac{1}{16}h^2(p_{\hat{\lambda}_n,\widehat{G}_n},p_{\lambdastar, G_*})\leq h^2(\Bar{p}_{\hat{\lambda}_n,\widehat{G}_n},p_{\lambdastar, G_*})\leq \frac{1}{\sqrt{n}}\nu_n(\hat{\lambda}_n,\widehat{G}_n),
    \end{align*}
    where $\nu(\hat{\lambda}_n,\widehat{G}_n)$ is an empirical process defined as
    \begin{align*}
        \nu_n(\hat{\lambda}_n,\widehat{G}_n):=\sqrt{n}\int_{p_{\lambdastar,G_*}>0}\frac{1}{2}\log\Big( \frac{\overline{p}_{\hat{\lambda}_n,\widehat{G}_n}}{p_{\lambdastar,G_*}}\Big)\cdot[\overline{p}_{\hat{\lambda}_n,\widehat{G}_n}-p_{\lambdastar,G_*}] \dint (X,Y).
    \end{align*}
    Thus, for any $\delta>\delta_n:=\sqrt{\log n/n}$, we obtain
    \begin{align*}
        &\mathbb{P}_{\lambdastar, G_*}(h(p_{\hat{\lambda}_n,\widehat{G}_n},p_{\lambdastar, G_*})\geq\delta)\\
        &\leq\mathbb{P}_{\lambdastar, G_*}\Big(\nu_n(\hat{\lambda}_n,\widehat{G}_n)-\sqrt{n}h^2(p_{\hat{\lambda}_n,\widehat{G}_n,p_{\lambdastar, G_*}})\geq 0,h(p_{\hat{\lambda}_n,\widehat{G}_n},p_{\lambdastar, G_*})\geq\delta/4\Big)\\
        &\leq  \mathbb{P}_{\lambdastar, G_*}\Big(\sup_{\lambda,G:h(\Bar{p}_{\lambda, G},p_{\lambdastar, G_*})\geq\delta/4}[\nu_n(\lambda G)-\sqrt{n}h^2(\Bar{p}_{\lambda, G},p_{\lambdastar, G_*})]\geq 0\Big)\\
        &\leq \sum_{s=0}^S\mathbb{P}_{\lambdastar, G_*}\Big(\sup_{\lambda,G:2^s\delta/4\leq h(\Bar{p}_{\lambda, G},p_{\lambdastar, G_*})\leq 2^{s+1}\delta/4}|\nu_n(\lambda G)|\geq \sqrt{n}2^{2s}(\delta/4)^2\Big)\\
        &\leq  \sum_{s=0}^S\mathbb{P}_{\lambdastar, G_*}\Big(\sup_{\lambda,G:h(\Bar{p}_{\lambda, G},p_{\lambdastar, G_*})\leq 2^{s+1}\delta/4}|\nu_n(\lambda G)|\geq \sqrt{n}2^{2s}(\delta/4)^2\Big),
    \end{align*}
    where $S$ is the smallest number such that $2^S\delta/4>1$. Now, we proceed by recalling Theorem 5.11 in \cite{Vandegeer-2000} with adapted notations to our setting as follows:
    \begin{lemma}[Theorem 5.11 in \cite{Vandegeer-2000}]
        \label{theorem:increment_empirical_process}
        Let $R>0$, $k\geq 1$, and $\Ocal_{k,\xi}(\Theta)\subset\Ocal_{k,\xi}(\Theta)$ containing $G_*$. Let $C$ be sufficiently large, then for all $n\in\mathbb{N}$ and $C_0,C_1,t>0$ that satisfy
        \begin{itemize}
            \item[(i)] $t\leq 8\sqrt{n}R \vee C_1\sqrt{n}R^2/K$;
            \item[(ii)] $t\geq C^2(C_1+1)\left(R\wedge \int_{t/(2^6\sqrt{n})}^RH_B^{1/2}\Big(\frac{u}{\sqrt{2}},\overline{\mathcal{P}}_{K}^{1/2}([0,1]\times\Theta,R),\nu\Big)du\right)$,
        \end{itemize}
        we get
        \begin{align*}
            \mathbb{P}_{\lambdastar, G_*}\left(\sup_{\substack{G\in\Ocal_{k,\xi}(\Theta),h(\Bar{p}_{\lambda, G},p_{\lambdastar, G_*})\leq R}}|\nu_n(\lambda G)|\geq t\right)\leq C\exp\left[-\dfrac{t^2}{C^2(C_1+1)R^2}\right].
        \end{align*}
    \end{lemma}
    Proof of Lemma~\ref{theorem:increment_empirical_process} can be found in \cite{Vandegeer-2000}.
    
    Back to our proof, by choosing $R=2^{s+1}\delta$, $C_1=15$ and $t=\sqrt{n}2^{2s}(\delta/4)^2$, we can verify that condition (i) in Lemma~\ref{theorem:increment_empirical_process} is satisfied as $2^{s-1}\delta/4\leq 1$ for all $s\leq S$. Meanwhile, the condition (ii) is met as
    \begin{align*}
        \int_{t/(2^6\sqrt{n})}^RH_B^{1/2}\left(\frac{u}{\sqrt{2}},\overline{\mathcal{P}}_K^{1/2}([0,1]\times\Theta,R),\nu\right)du\vee 2^{s+1}\delta&=\sqrt{2}\int_{R^2/2^{13}}^{R/\sqrt{2}}H_B^{1/2}\left(u,\overline{\mathcal{P}}_K^{1/2}([0,1]\times\Theta,R),\nu\right)du\vee 2^{s+1}\delta\\
        &\leq 2\mathcal{J}_B(R,\overline{\mathcal{P}}_K^{1/2}([0,1]\times\Theta,R),\nu)\\
        &\leq 2J\sqrt{n}2^{2s+1}\delta^2=2^6Jt.
    \end{align*}
    Applying Lemma~\ref{theorem:increment_empirical_process}, we get that
    \begin{align*}
        \mathbb{P}_{\lambdastar, G_*}(h(p_{\hat{\lambda}_n,\widehat{G}_n},p_{\lambdastar, G_*})>\delta)\leq C\sum_{s=0}^{\infty}\exp\left(\dfrac{2^{2s}n\delta^2}{J^2 2^{14}}\right)\leq c\exp\left(-\dfrac{n\delta^2}{c}\right).
    \end{align*}
    Hence, the proof is completed.

\subsubsection{Proof of Lemma~\ref{lemma:bracketing_control}}
\label{appendix:lemma:bracketing_control}
\textbf{Part (i).} For any set $S$, we denote $\mathcal{E}_{\varepsilon}(S)$ an $\varepsilon$-net of $S$ if each element of $S$ is within $\varepsilon$ distance from some elements of $\mathcal{E}_{\varepsilon}(S)$. By definition of the covering number, we get $|\mathcal{E}_{\varepsilon}(S)|=N(\varepsilon,S,\|\cdot\|_{\infty})$. Let $\mathcal{P}_k(\Theta):=\{p_G:G\in\mathcal{O}_k(\Theta)\}$, where $p_G(X,Y):=\sum_{i=1}^kp_if(Y|a_{i}^{\top}X+b_{i},\sigma_{i})$.
According to [Lemma 6, \cite{ho2022gaussian}], we have
\begin{align*}
    \log|\mathcal{E}_{\varepsilon}(\mathcal{P}_k(\Theta))|=N(\varepsilon,\mathcal{P}_k(\Theta),\|\cdot\|_{\infty})\lesssim \log(1/\varepsilon).
\end{align*}
Let $\Ocal_{k,\xi}(\Theta):=\{\widetilde{G}:p_{\Tilde{G}}\in\mathcal{E}_{\varepsilon}(\mathcal{P}_k(\Theta))\}$ be the set of all latent mixing measures $G$ in the net $\mathcal{E}_{\varepsilon}(\mathcal{P}_k(\Theta))$. We will show that
\begin{align}
    \label{eq:e-net-inequality}
    \mathcal{E}_{\varepsilon}(\mathcal{P}_k([0,1]\times\Theta))\subseteq\{p_{\lambda, G}:\Tilde{\lambda}\in\mathcal{E}_{\varepsilon}([0,1]),\widetilde{G}\in\Ocal_{k,\xi}(\Theta)\}.
\end{align}
Indeed, for any $\lambda\in[0,1],G\in\mathcal{O}_k(\Theta)$, there exist $\Tilde{\lambda}\in\mathcal{E}_{\varepsilon}([0,1]),\widetilde{G}\in\Ocal_{k,\xi}(\Theta)$ such that $|\lambda-\Tilde{\lambda}|\leq\varepsilon$ and $\|p_G-p_{\widetilde{G}}\|_{\infty}\leq\varepsilon$, which leads to
\begin{align*}
    \|p_{\lambda, G}-p_{\Tilde{\lambda}\widetilde{G}}\|_{\infty}&\leq\|p_{\lambda, G}-p_{\Tilde{\lambda}G}\|_{\infty}+\|p_{\Tilde{\lambda} G}-p_{\Tilde{\lambda}\widetilde{G}}\|_{\infty}\\
    &=|\lambda-\Tilde{\lambda}|\|g_0-p_G\|_{\infty}+\Tilde{\lambda}\|p_G-p_{\widetilde{G}}\|_{\infty}\\
    &\leq \varepsilon(\|g_0\|_{\infty}+\|p_G\|_{\infty})+\varepsilon\\
    &\lesssim \varepsilon.
\end{align*}
Therefore, we obtain equation~\eqref{eq:e-net-inequality}. Putting the above results together with a note that $\log(|\mathcal{E}_{\varepsilon}([0,1])|)\leq  \log(1/\varepsilon)$, we have
\begin{align*}
    \log(N(\varepsilon,\mathcal{P}_k(\Theta\times [0,1]),\|\cdot\|_{\infty}))\leq \log(|\mathcal{E}_{\varepsilon}([0,1])|)+\log(|\mathcal{E}_{\varepsilon}(\mathcal{P}_k(\Theta))|)\lesssim \log(1/\varepsilon).
\end{align*}
Hence, we reach the conclusion in part (i). 

\textbf{Part (ii).} Firstly, let $\eta \leq \varepsilon$ be some positive number that we will chose later. Since $f$ is the density function of an univariate location-scale Gaussian distribution, we can verify for any $\abss{Y} \geq 2a$ and $X \in \mathcal{X}$ that
\begin{align}
f(Y|h_{1}(X,\theta_{1}),h_{2}(X,\theta_{2})) \leq \frac{1}{\sqrt{2\pi}\ell}\exp\parenth{-Y^2/(8u^2)}. \nonumber
\end{align}
Recall that $\log g_0(Y|X)\lesssim -Y^p$ and $g_0(Y|X)\leq M$ for some positive constants $M,p>0$. Let $q=\min\{p,2\}$, $C_2=\max\left\{M,\frac{1}{\sqrt{2\pi}\ell}\right\}$ and
\begin{align}
H(X,Y) = \begin{cases} C_1\exp\parenth{-Y^q}\overline{f}(X), & \text{for} \ \abss{Y} \geq 2a \\ C_2\overline{f}(X), & \text{for} \ \abss{Y} < 2a, \end{cases}
\end{align}
where $C_1>0$ is a constant depending on $\ell,g_0$.
Thus, it can be shown that $H(X,Y)$ is an envelope of $\mathcal{P}_{k}([0,1]\times\Theta)$. Subsequently, we denote by $g_{1},\ldots,g_{N}$ an $\eta$-net over $\mathcal{P}_{k}([0,1]\times\Theta)$. Then, we construct the brackets $[p_{i}^{L}(X,Y),p_{i}^{U}(X,Y)]$ as follows:
\begin{align}
p_{i}^{L}(X,Y) : = \max \{g_{i}(X,Y) - \eta, 0\}, \ p_{i}^{U}(X,Y) : = \max \{g_{i}(X,Y) + \eta, H(X,Y) \} \nonumber 
\end{align}
for $1 \leq i \leq N$. As a result, $\mathcal{P}_{k}([0,1]\times\Theta) \subset \cup_{i=1}^{N} [p_{i}^{L}(X,Y), p_{i}^{U}(X,Y)]$ and $p_{i}^{U}(X,Y) - p_{i}^{L}(X,Y) \leq \min \{2\eta, H(X,Y)\}$. It follows that ,
\begin{align}
&\int \parenth{p_{i}^{U}(X,Y) - p_{i}^{L}(X,Y)} d(X,Y)\nonumber\\
& \leq \int \limits_{\abss{Y} < 2a} \parenth{p_{i}^{U}(X,Y) - p_{i}^{L}(X,Y)} d(X,Y)  + \int \limits_{\abss{Y} \geq 2a} \parenth{p_{i}^{U}(X,Y) - p_{i}^{L}(X,Y)} d(X,Y) \nonumber \\
& \leq \int \limits_{\abss{Y} < 2a} 2\eta d(X,Y)  + \int \limits_{\abss{Y} \geq 2a} H(X,Y) d(X,Y) \leq c\eta,\nonumber \\
\end{align}
where $c$ is some positive universal constant. This implies that
\begin{align}
H_{B}(c\eta, \mathcal{P}_{k}([0,1]\times\Theta), \|.\|_{1}) \leq N \lesssim \log(1/\eta). \nonumber
\end{align}
By choosing $\eta = \varepsilon/c$, we have
\begin{align}
H_{B}(\varepsilon, \mathcal{P}_{k}([0,1]\times\Theta), \|.\|_{1}) \lesssim \log(1/\varepsilon). \nonumber
\end{align}
Due to the inequality $h^{2} \leq \|.\|_{1}$ between Hellinger distance and total variational distance, we reach the conclusion of bracketing entropy bound.

\section{PARAMETER ESTIMATION UNDER THE FULL OVERLAP REGIME} 
\label{appendix:full_overlap}
In this appendix, we study the convergence rates of parameter estimation in the deviated Gaussian mixture of experts under the full overlap regime, namely when the function $g_0(Y|X)$ takes the following form:
\begin{align}
    \label{eq:new_g0}
    g_0(Y|X)=p_{G_0}(Y|X):=\sum_{j=1}^{k_0}p^0_jf(Y|(a^0_j)^{\top}X+b^0_j,\sigma^0_j),
\end{align}
and $\kbar=k_0$, where $\kbar$ stands for the number of overlapped components of two mixing measures $G_0$ and $G_*$.

Under this regime, it is worth noting that if $G_*=G_0$, then the conditional density function $p_{\lambdastar G_*}(Y|X)$ is reduced to $p_{\lambdastar G_*}(Y|X)=(1-\lambdastar)p_{G_0}(Y|X)+\lambdastar p_{G_0}(Y|X)=p_{G_0}(Y|X)$, which coincides with the setting $\lambdastar=0$ that we will consider in Appendix~\ref{appendix:additional_results}. For that reason, we assume that $G_*\neq G_0$ throughout this appendix. 

\textbf{Identiability of the deviated Gaussian mixture of experts.} Similar to the partial overlap regime studied in Section~\ref{sec:non-distinguishable}, the deviated Gaussian mixture of experts under the full overlap regime is also not identifiable. Furthermore, it is even more challenging to solve the equation $p_{\lambda,G}(X,Y)=p_{\lambdastar,\Gstar}(X,Y)$ for almost surely $(X,Y)$ in this regime. In particular, we have to take into account the following set of mixing proportions $\lambda$ which make the term $\bGstar=\frac{\lambda-\lambdastar}{\lambda}G_0+\frac{\lambdastar}{\lambda}G_*$ defined in equation~\eqref{eq:G_bar_definition} a valid mixing measure:
\begin{align*}
    \mathcal{T}&:=\{\lambda\in(0,1]:(\lambdastar-\lambda)p^0_i\leq\lambdastar p^*_i,\forall i\in[k_0]\}.
\end{align*}
Here, the set $\mathcal{T}$ contains $\lambda\in(0,1]$ such that the weights associated with components of $\bGstar$ are non-negative, i.e. $\bar{p}^*_i(\lambda)\geq 0$ where
\begin{align*}
    \bar{p}^*_i(\lambda):=\begin{cases}
        [(\lambda-\lambdastar)p^0_i+\lambdastar p^*_i]/\lambda, \quad i\in[k_0],\\
        \lambdastar p^*_i/\lambda, \hspace{7.2em} i\in[k_*]\setminus[k_0].
    \end{cases}
\end{align*}
Subsequently, we solve the identifiability equation in two complement scenarios of $\lambda$ with respect to the set $\mathcal{T}$.

\textbf{When} $\lambda\in\mathcal{T}$: Since $\bGstar$ is valid mixing measure in this case, the identifiability equation can be rewritten as $\lambda[p_{G}(X,Y)-p_{\overline{G}_*(\lambda)}(X,Y)]=0$ for almost surely $(X,Y)$. Moreover, as $\bGstar$ has $k_*+k_0-\kbar=k_*$ components and $k>k_*$, that equation admits $(\lambda,\overline{G}_*(\lambda))$ as a solution for any $\lambda\in\mathcal{T}$. Additionally, it is worth noting that $(\lambdastar,G_*)$ is a special instance of $(\lambda,\bGstar)$ when $\lambda=\lambdastar\in\mathcal{T}$.

\textbf{When} $\lambda\in\mathcal{T}^c$: In this case, there are some components of $\bGstar$ having negative weights $\bar{p}^*_i(\lambda)<0$. Thus, it is necessary to inspect such components by considering the following set:
\begin{align*}
    I_{\lambda}&:=\{i\in[k_0]:(\lambdastar-\lambda)p^0_i>\lambdastar p^*_i\}, 
\end{align*}
which includes indices $i\in[k_0]$ such that $\bar{p}^*_i(\lambda)<0$. Here, we say that $I_{\lambda}$ is \textit{ratio-independent} if $|I_{\lambda}|=1$ or $p^0_i/p^*_i=p^0_j/p^*_j$ for all $i,j\in I_{\lambda}$ if $|I_{\lambda}|\geq 2$. An intuition behind this definition is to guarantee that all the terms $(\lambdastar-\lambda)p^0_i-\lambdastar p^*_i$, for $i\in I_{\lambda}$, can be arbitrary small simultaneously. Consequently, when $I_{\lambda}$ is a ratio-independent set, mixing measures of the following form are solutions of the identifiability equation:
\begin{align}\label{eq:G_tilde_definition}
    \widetilde{G}_*(\lambda)&:=\dfrac{1}{s(\lambda)}\left(\sum_{i\in I_{\lambda}^c}\Big[~p^*_i\lambdastar+(\lambda-\lambdastar)p^0_i\Big]\delta_{(a^0_i,b^0_i,\sigma^0_i)} +\sum_{i=k_0+1}^{k_*}\lambdastar p^*_i\delta_{(a^*_i,b^*_i,\sigma^*_i)}\right),
\end{align}
where $s(\lambda):=\sum_{i\in I_{\lambda}^c}\Big[p^*_i\lambdastar+(\lambda-\lambdastar)p^0_i\Big]+\sum_{i=k_0+1}^{k_*}\lambdastar p^*_i$ is a normalizing term. It can be seen from the above formulation that components of the mixing measure $\widetilde{G}_*(\lambda)$ are those of $\bGstar$ with positive weights.

\textbf{Voronoi loss function.} Now, we are ready to define the Voronoi loss function $D_4((\lambda, G),(\lambdastar, G_*))$ for the full overlap regime as follows: 
% if $I_{\lambda}$ is ratio-independent, then
% \begin{align*}
%     D_4((\lambda, G),(\lambdastar, G_*))&=~\ones_{\{\lambda\in\mathcal{T}^c\}}s(\lambda)D_3(G,\widetilde{G}_*(\lambda)) +\ones_{\{\lambda\in\mathcal{T}\}}D_3(G,\overline{G}_*(\lambda)).
% \end{align*}
% otherwise,
% \begin{align}
% \label{eq:D_full_dependent}
%         D_4((\lambda, G),(\lambdastar, G_*))&=\ones_{\{\lambda\in\mathcal{T}^c\}}\sum_{i\in I_{\lambda}}[-\lambda p'_i(\lambda)] +\ones_{\{\lambda\in\mathcal{T}\}}D_3(G,\overline{G}_*(\lambda)),
% \end{align}
\begin{align}
    \label{eq:D_full_dependent}
    D_4((\lambda, G),(\lambdastar, G_*))=\begin{cases}
        \ones_{\{\lambda\in\mathcal{T}^c\}}s(\lambda)D_3(G,\widetilde{G}_*(\lambda)) +\ones_{\{\lambda\in\mathcal{T}\}}D_3(G,\overline{G}_*(\lambda)), \quad \text{if $I_{\lambda}$ is ratio-independent};\\
        \textbf{}\\
        \ones_{\{\lambda\in\mathcal{T}^c\}}\sum_{i\in I_{\lambda}}[-\lambda p'_i(\lambda)] +\ones_{\{\lambda\in\mathcal{T}\}}D_3(G,\overline{G}_*(\lambda)), \hspace{1.8em} \text{otherwise},
    \end{cases}
\end{align}
where the loss function $D_3$ is given in Section~\ref{sec:non-distinguishable}.

Given the above loss of interest, we establish the convergence rate of the MLE under the full overlap regime in the following theorem.
\begin{theorem}
\label{theorem:full_dependent}
Assume that $\lambdastar\in(0,1]$ is unknown, and let $g_{0}$ take the form in equation~\eqref{eq:new_g0} with $\bar{k}=k_0$. Then, we achieve that $V(p_{\lambda,G},p_{\lambdastar,G_*})\gtrsim D_4((\lambda,G),(\lambdastar,G_*))$ for any $(\lambda,G)\in[0,1]\times\Ocal_k(\Theta)$. This bound together with Proposition~\ref{prop:mle_estimation} indicate that 
\begin{align*}
    \mathbb{P}(D_4((\hat{\lambda}_n, \widehat{G}_n),(\lambdastar, G_*))>C_4\sqrt{\log (n)/n})\lesssim n^{-c_4},
\end{align*}
where $C_4>0$ is a constant depending on $g_0,\lambdastar$, $G_*,\Theta$, while the constant $c_4$ depends only on $\Theta$.
\end{theorem}
Proof of Theorem~\ref{theorem:full_dependent} is deferred to Appendix~\ref{appendix:full_dependent}. When $\hat{\lambda}_n\in\mathcal{T}$, the formulation of $D_4((\lambda, G),(\lambdastar, G_*))$ is simplified to that of $D_3(G,\overline{G}_*(\lambda))$. Therefore, the convergence behavior of the MLE in this case resembles the results in Theorem~\ref{theorem:partial_dependent}, which will not be repeated here. The difference between this theorem and its previous counterparts occurs only when $\hat{\lambda}_n\in\mathcal{T}^c$. In particular, if $I_{\hat{\lambda}_n}$ is a ratio-independent set, then the MLE $\widehat{G}_n$ converges to $\widetilde{G}_*(\hat{\lambda}_n)$ at a substantially slower rate than $\Otilde(n^{-1/2})$ as it depends on the convergence rate of $s(\hat{\lambda}_n)$ to zero. By contrast, if the set $I_{\hat{\lambda}_n}$ is not ratio-independent, then the discrepancy $D_4((\lambda, G),(\lambdastar, G_*))$ will not vanish as $n$ tends to infinity. Thus, we cannot deduce any conclusions regarding the convergence rates of the MLE in this case. An underlying reason for this phenomenon is that the terms $(\lambdastar-\hat{\lambda}_n)p^0_i-\lambdastar p^*_i$ for $i\in {I}_{\hat{\lambda}_n}$ cannot approach zero simultaneously as $I_{\hat{\lambda}_n}$ is not ratio-independent.

\subsection{Proof of Theorem~\ref{theorem:full_dependent}}
    
% \subsection{}
\label{appendix:full_dependent}
% Recall that since the setting $G_*=G_0$ is equivalent to the case when $\lambdastar=0$, which will be presented in Appendix~\ref{appendix:additional_results}, we assume that $G_*\neq G_0$ in this proof. 
Similar to previous proofs, we need to demonstrate the following inequality:
\begin{align}
    \label{eq:full_dependent_1}
    \inf_{\lambda\in[0,1],G\in\Ocal_{k,\xi}(\Theta)} V(p_{\lambda, G},p_{\lambdastar,\Gstar})/D_4((\lambda, G),(\lambdastar,\Gstar))>0.
\end{align}
\textbf{Local inequality.} Firstly, we will derive the local version of the above inequality:
\begin{align}
    \label{eq:full_dependent_local}
    \lim_{\varepsilon\to 0}\inf_{\substack{\lambda\in[0,1],G\in\Ocal_{k,\xi}(\Theta):\\ D_4((\lambda, G),(\lambdastar,\Gstar))\leq\varepsilon}}V(p_{\lambda, G},p_{\lambdastar,\Gstar})/D_4((\lambda, G),(\lambdastar,\Gstar))>0.
\end{align}
Assume that the above claim does not hold true, then there exist a sequence of mixing measures $G_n=\sum_{i=1}^{k_n}p^n_i\delta_{(\ain,\bin,\sigmain)}\in \Ocal_{k,\xi}(\Theta)$ and a sequence of mixing proportions $\lambdan\in[0,1]$ such that
\begin{align*}
\begin{cases}
D_{4n}:=D_4((\lambda_n,G_n),(\lambdastar,\Gstar))\to 0,\\
V(p_{\lambdaGn},p_{\lambdaGstar})/D_{4n}\to 0,
\end{cases}
\end{align*}
as $n\to\infty$. Under the setting of Theorem~\ref{theorem:full_dependent}, we have $\thetastar_j=\theta^0_j$ for all $j\in[k_0]$ and
\begin{align}
    &p_{\lambdan, G_n}(X,Y)-p_{\lambdastar, G_*}(X,Y)\nonumber\\
    \label{eq:full_dependent_88}
    &=\lambdan\sum_{i=1}^{k_n}p^n_if(Y|(\ain)^{\top}X+\bin,\sigmain)\Bar{f}(X)-\sum_{j=1}^{\kstar}\Bar{p}^*_j(\lambdan)f(Y|(\aj)^{\top}X+\bj,\sj)\Bar{f}(X),
\end{align}
where $\Bar{p}^*_j(\lambdan):=\begin{cases}
    \lambdastar p^*_j+(\lambdan-\lambdastar)p^0_j,\quad j\in[k_0]\\
    \lambdastar p^*_j, \quad k_0+1\leq j\leq\kstar
\end{cases}$.

Next, we will show that $\lim\inf \lambdan$ is bounded below by some positive constant. Assume that this claim is not true, then $\lambdan\to 0$ as $n\to\infty$. Note that
\begin{align*}
V(p_{\lambda,G_n},p_{\lambda, G_*})=\dfrac{V(p_{\lambda,G_n},p_{\lambda,G_*})}{D_{4n}}\times D_{4n}\to 0.
\end{align*}
Then, by the Fatou's lemma, we get that $p_{\lambdan, G_n}(X,Y)-p_{\lambdastar, G_*}(X,Y)\to 0$ as $n\to\infty$ for almost surely $(X,Y)$. Since $\lambdan\to 0$ and the density $f(Y|(\ain)^{\top}X+\bin,\sigmain)$ can be upper bounded by a function which is independent of $n$ for almost surely $(X,Y)$ (see the proof of part (ii) of Lemma~\ref{lemma:bracketing_control} for more detail), we deduce that $$\lambdan\sum_{i=1}^{k_n}p^n_if(Y|h_1(X,\thetan_{1i}),h_2(X,\thetan_{2i}))\to 0.$$ It follows that $\sum_{j=1}^{\kstar}\Bar{p}^*_j(\lambdan)f(Y|h_1(X,\thetastar_{1j}),h_2(X,\thetastar_{2j}))\to 0$ as $n\to\infty$, which leads to the fact that $\Bar{p}^*_j(\lambdan)\to 0$ for all $j\in[\kstar]$. This means that $p^*_i=p^0_i$ when $i\in[k_0]$ and $p^*_i=0$ otherwise. Thus, we obtain $G_*\equiv G_0$, which is a contradiction to the assumption that $G_*\neq G_0$. Therefore, $\lim\inf \lambdan$ is bounded below by some positive constant.

Subsequently, we consider two main scenarios of $\lambdan$ based on the set $\mathcal{T}$ mentioned in Section~\ref{appendix:full_overlap}, i.e.
\begin{align*}
    \mathcal{T}&:=\{\lambda\in(0,1]:(\lambdastar-\lambda)p^0_i\leq\lambdastar p^*_i,\forall i\in[k_0]\}.
\end{align*}
\textbf{Case 1:} $\lambdan\in\mathcal{T}$ for infinitely $n\in\mathbb{N}$. WLOG, we assume that $\lambdan\in\mathcal{T}$ for all $n\in\mathbb{N}$. 

In this case, we have $D_{4n}=D_3(G_n,\overline{G}_*(\lambdan))$, and the difference $p_{\lambdan, G_n}(X,Y)-p_{\lambdastar, G_*}(X,Y)$ can be written as 
\begin{align}
    &p_{\lambdan, G_n}(X,Y)-p_{\lambdastar, G_*}(X,Y)=~\lambdan\Bigg\{\sum_{j=1}^{k_*}\sum_{i\in \mathcal{A}_j} p^n_if(Y|(\ain)^{\top}X+\bin,\sigmain)\nonumber\\
    &\qquad -\Big[\Big(1-\frac{\lambdastar}{\lambdan}\Big)\sum_{j=1}^{k_0}p^0_jf(Y|(a^0_j)^{\top}X+b^0_j,\sigma^0_j)+\frac{\lambdastar}{\lambdan}\sum_{j=1}^{\kstar}p^*_jf(Y|(\aj)^{\top}X+\bj,\sj)\Big]\Bigg\}\Bar{f}(X)\nonumber\\
    \label{eq:full_dependent_99}
    =&~\lambdan\Big[p_{G_n}(X,Y)-p_{\overline{G}_*(\lambdan)}(X,Y)\Big].
\end{align}
Recall that under the full overlap regime, we have $\Bar{k}=k_0$, which leads to $k\geq k_*= k_*+k_0-\Bar{k}$. Moreover, since $\lambdan\in\mathcal{T}$, we get that $\overline{G}_{*}(\lambdan)$ is a valid mixing measure. Thus, by employing arguments utilized in Case 2.2 in Appendix~\ref{appendix:partial_dependent}, we obtain the local inequality in equation~\eqref{eq:full_dependent_local} for this case.

\textbf{Case 2:} $\lambdan\not\in\mathcal{T}$ for infinitely $n\in\mathbb{N}$. WLOG, we assume that $\lambdan\not\in\mathcal{T}$ for all $n\in\mathbb{N}$. 

For each $n\in\mathbb{N}$, since $\lambdan\not\in\mathcal{T}$, there exists an index $i\in[k_0]$ such that $\lambdastar p^*_i-(\lambdastar-\lambdan)p^0_i<0$. In other words, the set $I_{\lambdan}:=\{i\in[k_0]:(\lambdastar-\lambdan)p^0_i>\lambdastar p^*_i\}$ is not empty and has at least one element. In addition, we also have that $\Bar{p}^*_j(\lambdan)<0$ for any $j\in I_{\lambdan}$ in this case.

Next, we will consider two different settings of the set $I_{\lambdan}$ as follows:

\textbf{Case 2.1:} $I_{\lambdan}$ is not ratio-independent. 

From the formulation of metric $D_3$ in equation~\eqref{eq:D_full_dependent}, we have $D_{4n}=\sum_{j\in I_{\lambdan}}[-\Bar{p}^*_j(\lambdan)]\to 0$ in this case. 
% which implies that $V(p_{\lambdan, G_n},p_{\lambdastar, G_*})=V(p_{\lambdan, G_n},p_{\lambdastar, G_*})/D_{4n}\to 0$. Given this result, by means of Fatou's lemma, we obtain that $p_{\lambdan, G_n}(X,Y)-p_{\lambdastar, G_*}(X,Y)\to 0$ as $n\to\infty$ for almost surely $(X,Y)$. 
Recall that we have $-\Bar{p}^*_j(\lambdan)>0$ for all $j\in I_{\lambdan}$, then it follows that $\Bar{p}^*_j(\lambdan)\to 0$ as $n\to\infty$ for all $j\in I_{\lambdan}$. This leads to the fact that $p^*_i/p^0_i=p^*_j/p^0_j$ for all $i,j\in I_{\lambdan}$, which is a contradiction to the assumption that $I_{\lambdan}$ is not ratio-independent. Therefore, we obtain the local inequality in equation~\eqref{eq:full_dependent_local} for this case.

\textbf{Case 2.2:} $I_{\lambdan}$ is ratio-independent.

In this case, we have $D_{4n}=s(\lambdan)D_3(G_n,\widetilde{G}_*(\lambdan))$. Next, we will demonstrate that $s(\lambdan)\not\to 0$ as $n\to\infty$. Assume by contrary that $s(\lambdan)\to 0$, then $p^*_j=0$ for all $j>k_0$ and $\Bar{p}^*_j(\lambdan)=\lambdastar p^*_j+(\lambdan-\lambdastar)p^0_j\to 0$ for all $j\not\in I_{\lambdan}$. Note that
\begin{align*}
    V(p_{\lambdan, G_n},p_{\lambdastar, G_*})=\dfrac{V(p_{\lambdan, G_n},p_{\lambdastar, G_*})}{D_{4n}}\times D_{4n}\to 0.
\end{align*}
Therefore, by means of Fatou's lemma, we get that $p_{\lambdan, G_n}(X,Y)-p_{\lambdastar, G_*}(X,Y)\to 0$ when $n\to\infty$. Recall that
\begin{align*}
    &p_{\lambdan, G_n}(X,Y)-p_{\lambdastar, G_*}(X,Y)\\
    =&\sum_{j\in I_{\lambdan}}(-\Bar{p}^*_j(\lambdan))f(Y|(\aj)^{\top}X+\bj,\sj)\Bar{f}(X)+\Big[\lambdan\sum_{i=1}^Kp^n_if(Y|(\ain)^{\top}X+\bin,\sigmain)\Bar{f}(X)\\
    -&\sum_{j\in I_{\lambdan}^c}\Bar{p}^*_j(\lambdan)f(Y|(\aj)^{\top}X+\bj,\sj)\Bar{f}(X)-\sum_{j=k_0+1}^{\kstar}\Bar{p}^*_j(\lambdan)f(Y|(\aj)^{\top}X+\bj,\sj)\Bar{f}(X)\Big],
\end{align*}
we get $\Bar{p}^*_j\to 0$ for all $j\in I_{\lambdan}$ and $\lambdan\to 0$, which is a contradiction to the result that $\lim\inf \lambdan$ is bounded below by a positive constant. Thus, $s(\lambdan)\not\to 0$ as $n\to\infty$.

From the definition of $\widetilde{G}_*(\lambdan)$, we can rewrite it as $\widetilde{G}_*(\lambdan):=\sum_{j\in \mathcal{J}_{\lambdan}}\frac{\Bar{p}^*_j(\lambdan)}{s(\lambdan)}\delta_{(\aj,\bj,\sj)}$, where $$\mathcal{J}_{\lambdan}:=I_{\lambdan}^c\cup\{k_0+1,\ldots,k_*\}.$$ Next, we will use the following Voronoi cells to study the discrepancy $D_3(G_n,\widetilde{G}_*(\lambdan))$:
\begin{align*}   
    \mathcal{C}_j^n=\mathcal{C}_j(G_n)=\{i\in[k_n]:\|\theta^n_i-\theta^*_j\|\leq \|\theta^n_i-\theta^*_\ell\|, \ \forall \ell\neq j\},
\end{align*}
for any $\forall j\in\mathcal{J}_{\lambdan}$, where $\theta^n_i:=(\ain,\bin,\sigmain)$ and $\theta^*_j=(\aj,\bj,\sj)$.

As $k_n\leq k$ for all $n$, there exists a subsequence of $G_n$ such that $k_n$ does not change with $n$. Thus, by replacing $G_n$ with this subsequence, we assume that $k_n=k$ for all $n$. Additionally, $\mathcal{C}_j=\mathcal{C}_j^n$ does not change with $n$ for all $j\in[k_*]$, either. Then, we rewrite the difference $p_{\lambdan, G_n}(X,Y)-p_{\lambdastar, G_*}(X,Y)$ as follows:
\begin{align*}
    p_{\lambdan, G_n}(X,Y)&-p_{\lambdastar, G_*}(X,Y)=\sum_{j\in I_{\lambdan}}[-\Bar{p}^*_j(\lambdan)]f(Y|(\aj)^{\top}X+\bj,\sj)\Bar{f}(X)\\
    &+\sum_{j:|\mathcal{C}_j|>1}\sum_{i\in\mathcal{C}_j}\lambdan p^n_i[f(Y|(\ain)^{\top}X+\bin,\sigmain)-f(Y|(\aj)^{\top}X+\bj,\sj)]\Bar{f}(X)\\
    &+\sum_{j:|\mathcal{C}_j|=1}\sum_{i\in\mathcal{C}_j}\lambdan p^n_i[f(Y|(\ain)^{\top}X+\bin,\sigmain)-f(Y|(\aj)^{\top}X+\bj,\sj)]\Bar{f}(X)\\
    &+\sum_{j\in\mathcal{J}_{\lambdan}}\left(\sum_{i\in\mathcal{C}_j}\lambdan p^n_i-\Bar{p}^*_j(\lambdan)\right)f(Y|(\aj)^{\top}X+\bj,\sj)\Bar{f}(X)\\
    &:=C_n+A_{n,1}+A_{n,2}+B_n.
\end{align*}
For each $j\in\mathcal{J}_{\lambdan}:|\mathcal{C}_j|>1$, by applying the Taylor expansion up to order $\brcj$ as in Appendix~\ref{appendix:distinguishable_dependent}, we can rewrite $A_{n,1}$ as
\begin{align*}
    A_{n,1}=\sum_{j:|\mathcal{C}_j|>1}\sum_{|\alpha_1|=0}^{\brcj}\sum_{\ell=0}^{2(\brcj-|\alpha_1|)}E^n_{\alpha_1,\ell}(j)X^{\alpha_1}\cdot\frac{\partial^{|\alpha_1|+\ell}f}{\partial h_1^{|\alpha_1|+\ell}}(Y|(\aj)^{\top}X+\bj,\sj)\Bar{f}(X) 
    + R_5(X,Y),
\end{align*}
where $R_5(X,Y)$ is a Taylor remainder such that $R_5(X,Y)/D_{4n}$, and
\begin{align}
    \label{eq:En_definition}
    E^n_{\alpha_1,\ell}(j):=\sum_{i\in\mathcal{C}_j}\sum_{\substack{\alpha_2+2\alpha_3=\ell \\ \alpha_2+\alpha_3\geq 1-|\alpha_1|}}\frac{\lambdan p^n_i}{2^{\alpha_3}\alpha!}\cdot(\daijn)^{\alpha_1}(\dbijn)^{\alpha_2}(\dsijn)^{\alpha_3},
\end{align}
for any $j\in\mathcal{J}_{\lambdan}:|\mathcal{C}_j|>1$, $0\leq |\alpha_1|\leq \brcj$ and $0\leq \ell\leq 2(\brcj-|\alpha_1|)$.

On the other hand, by means of Taylor expansion up to the first order, we can decompose $A_{n,2}$ as
\begin{align*}
    A_{n,2}=\sum_{j:|\mathcal{C}_j|=1}\sum_{|\alpha_1|=0}^{1}\sum_{\ell=0}^{2(1-|\alpha_1|)}E^n_{\alpha_1,\ell}(j)X^{\alpha_1}\cdot\frac{\partial^{|\alpha_1|+\ell}f}{\partial h_1^{|\alpha_1|+\ell}}(Y|(\aj)^{\top}X+\bj,\sj)\Bar{f}(X)+R_6(X,Y),
\end{align*}
where $R_6(X,Y)$ is a Taylor remainder term such that $R_6(X,Y)/D_{1n}\to 0$ as $n\to\infty$, and $E_{\alpha_1,\ell}(j)$ is defined similarly as in equation~\eqref{eq:En_definition} but for $j\in\mathcal{J}_{\lambdan}:|\mathcal{C}_j|=1$, $0\leq|\alpha_1|\leq 1$ and $0\leq\ell\leq 2(\brcj-|\alpha_1|)$. Additionally, we also utilize the notation $E^n_{\alpha_1,\ell}(j)$ to denote the coefficients in $C_n$ as $E^n_{\zerod,0}(j):=-\Bar{p}^*_j(\lambdan)$ for any $j\in I_{\lambdan}$, and those in $B_n$ as
\begin{align*}
    E^n_{\zerod,0}(j):=\sum_{i\in\mathcal{C}_j}\lambdan p^n_i-\Bar{p}^*_j(\lambdan),
\end{align*}
for any $j\in\mathcal{J}_{\lambdan}$.
Therefore, $A_{n,1}$, $A_{n,2}$, $B_n$ and $C_n$ can be viewed as linear combinations of elements of the following set:
\begin{align}
\label{eq:set_H3}
\mathcal{H}_3 : = \biggr\{X^{\alpha_1}\cdot\frac{\partial^{|\alpha_1|+\ell}f}{\partial h_1^{|\alpha_1|+\ell}}(Y|(\aj)^{\top}X+\bj,\sj)\Bar{f}(X): j\in\mathcal{J}_{\lambdan}, \ 0 \leq |\alpha_1| \leq \Bar{r}(|\mathcal{C}_j|),
0\leq \ell\leq 2(\brcj-|\alpha_1|),  \biggr\}. 
\end{align}
Assume that all the coefficients in the formulations of $A_{n,1}/D_{4n}$, $A_{n,2}/D_{4n}$, $B_n/D_{4n}$ and $C_n/D_{4n}$ go to zero as $n\to\infty$. Now, we consider the following quantity:
\begin{align}
    1&=\frac{D_{4n}}{D_{4n}}=\frac{s(\lambdan)D_3(G_n,\widetilde{G}_*(\lambdan))}{D_{4n}}\nonumber\\
    &=\frac{s(\lambdan)\sum_{j\in\mathcal{J}_{\lambdan}}|\sum_{i\in\mathcal{C}_j}p^n_i-\Bar{p}^*_j(\lambdan)/s(\lambdan)|}{D_{4n}}\nonumber\\
    &+\frac{s(\lambdan)\sum_{j:|\mathcal{C}_j|=1}\sum_{i\in\mathcal{C}_j}p^n_{i}(\|\daijn\|+|\dbijn|+|\dsijn|)}{D_{4n}}\nonumber\\    &+\frac{s(\lambdan)\sum_{j:|\mathcal{C}_j|>1}\sum_{i\in\mathcal{C}_j}p^n_{i}(\|\daijn\|^{2}+|\dbijn|^{\brcj}+|\dsijn|^{\brcj/2})}{D_{4n}}
    \label{eq:full_dependent_9}
\end{align}
For $j\in\mathcal{J}_{\lambdan}:|\mathcal{C}_j|>1$, we take summation of the limits of $E_{\alpha_1,0}(j)$, where $\alpha_1\in\{2e_1,2e_2,\ldots,2e_{d}\}$ with $e_u:=(0,\ldots,0,\underbrace{1}_{\textit{u-th}},0,\ldots,0)$, and obtain that
\begin{align}
    \label{eq:full_dependent_5}
    \frac{1}{D_{4n}}\cdot\sum_{j\in\mathcal{J}_{\lambdan}:|\mathcal{C}_j|>1}\sum_{i\in\mathcal{C}_j}\lambdan p^n_i\|\daijn\|^2\to 0,
\end{align}
For $j\in\mathcal{J}_{\lambdan}$ such that $|\mathcal{C}_j|=1$, we combine the limits of $E_{\zerod,1}(j)/D_{4n}$, $E_{\zerod,2}(j)/D_{4n}$ and $E_{\alpha_1,0}(j)/D_{4n}$ for any $\alpha_1\in\{e_1,e_2,\ldots,e_d\}$, then 
\begin{align*}
    \frac{1}{D_{4n}}\cdot\sum_{j\in\mathcal{J}_{\lambdan}:|\mathcal{C}_j|=1}\sum_{i\in\mathcal{C}_j}\lambdan p^n_i\Big(\|\daijn\|_1+|\dbijn|+|\dsijn|\Big)\to 0.
\end{align*}
Due to the topological equivalence between $1$-norm and $2$-norm, we receive
\begin{align}
    \label{eq:coeff_A_n2}
    \frac{1}{D_{4n}}\cdot\sum_{j\in\mathcal{J}_{\lambdan}:|\mathcal{C}_j|=1}\sum_{i\in\mathcal{C}_j}\lambdan p^n_i\Big(\|\daijn\|+|\dbijn|+|\dsijn|\Big)\to 0.
\end{align}
Since $s(\lambdan)\not\to 0$, it follows from equations~\eqref{eq:full_dependent_5} and \eqref{eq:coeff_A_n2} that
\begin{align}
    \label{eq:limit_1}
    &\frac{s(\lambdan)}{D_{4n}}\cdot\sum_{j\in\mathcal{J}_{\lambdan}:|\mathcal{C}_j|>1}\sum_{i\in\mathcal{C}_j}\lambdan p^n_i\|\daijn\|^2+\frac{s(\lambdan)}{D_{4n}}\cdot\sum_{j\in\mathcal{J}_{\lambdan}:|\mathcal{C}_j|=1}\sum_{i\in\mathcal{C}_j}\lambdan p^n_i\Big(\|\daijn\|_1+|\dbijn|+|\dsijn|\Big)\to0.
\end{align}
By taking the summation of the limits of $|E_{\zerod,0}(j)|/D_{4n}$ for $j\in\mathcal{J}_{\lambdan}$, we get that 
\begin{align*}
\dfrac{1}{D_{4n}}\cdot{\sum_{j\in\mathcal{J}_{\lambdan}}\left|\sum_{i\in\mathcal{C}_j}\lambdan p^n_{i}-\bar{p}^*_j(\lambdan)\right|}\to 0.
\end{align*}
From the above hypothesis, we take the summation of all the coefficients in the representation of $C_n/D_{4n}$ and get that
\begin{align*}
    \frac{1}{D_{4n}}\cdot\sum_{j\in I_{\lambdan}}-\Bar{p}^*_j(\lambdan)\to 0.
\end{align*}
Then, we have
\begin{align}
    0&\leq \frac{s(\lambdan)|\sum_{i\in\mathcal{C}_j}p^n_i-\Bar{p}^*_j(\lambdan)/s(\lambdan)|}{D_{4n}}=\frac{\sum_{i\in\mathcal{J}_{\lambdan}}|s(\lambdan)\sum_{i\in\mathcal{C}_j}p^n_{i}-\Bar{p}^*_j(\lambdan)|}{D_{4n}}\nonumber\\
    &\leq\frac{\sum_{j\in\mathcal{J}_{\lambdan}}|\lambdan\sum_{i\in\mathcal{C}_j}p^n_{i}-\Bar{p}^*_j(\lambdan)|}{D_{4n}}+\left(\sum_{j\in\mathcal{J}_{\lambdan}}\sum_{i\in\mathcal{C}_j}p^n_i\right)\frac{\sum_{j\in I_{\lambdan}}-\Bar{p}^*_j(\lambdan)}{D_{4n}}\to 0,\nonumber
\end{align}
which leads to 
\begin{align}
    \label{eq:limit_2}
    \frac{1}{D_{4n}}\cdot{s(\lambdan)\left|\sum_{i\in\mathcal{C}_j}p^n_i-\Bar{p}^*_j(\lambdan)/s(\lambdan)\right|}\to0.
\end{align}
By plugging in the limits in equations~\eqref{eq:limit_1} and \eqref{eq:limit_2} into the equation\eqref{eq:full_dependent_9}, we deduce that
\begin{align*}
\frac{1}{D_{4n}}\cdot s(\lambdan)\sum_{j:|\mathcal{C}_j|>1}\sum_{i\in\mathcal{C}_j}p^n_{i}\left(|\dbijn|^{\Bar{r}(|\mathcal{C}_j|)}+|\dsijn|^{\Bar{r}(|\mathcal{C}_j|)/2}\right)\to 1.
\end{align*}
Therefore, we can find an index $j^*\in\mathcal{J}_{\lambdan}$ such that $|\mathcal{C}_j|>1$ satisfies
\begin{align*}
\frac{1}{D_{4n}}\cdot s(\lambdan)\sum_{i\in\mathcal{C}_{j^*}}p^n_{i}\left(|(\Delta b^n_{ij^*})^{(1)}|^{\Bar{r}(|\mathcal{C}_{j^*}|)}+|\Delta \sigma^n_{ij^*}|^{\Bar{r}(|\mathcal{C}_{j^*}|)/2}\right)\not\to 0.
\end{align*}
WLOG, we assume that $j^*=1$. From the hypothesis, as $E_{\zerod,\ell}(1)/D_{4n}\to 0$ as $n\to\infty$ for any $1\leq \ell\leq \brcone$, we have
\begin{align*}
   \dfrac{\sum_{i\in\mathcal{C}_{1}}p^n_{i}\sum_{\substack{\alpha_2+2\alpha_3=\ell}}\dfrac{(\dbione)^{\alpha_2}(\dsione)^{\alpha_3}}{2^{\alpha_{3}}{\alpha_2}!{\alpha_{3}}!}}{s(\lambdan)\sum_{i\in\mathcal{C}_{1}}p^n_{i}\left(|\dbione|^{\Bar{r}(|\mathcal{C}_{1}|)}+|\dsijn|^{\Bar{r}(|\mathcal{C}_{1}|)/2}\right)}\to 0,
\end{align*}
for any $1\leq \ell\leq \brcone$. Recall that $s(\lambdan)\not\to0$, then
\begin{align}\label{eq:limit_ratio}
   \dfrac{\sum_{i\in\mathcal{C}_{1}}p^n_{i}\sum_{\substack{\alpha_2+2\alpha_3=\ell}}\dfrac{(\dbione)^{\alpha_2}(\dsione)^{\alpha_3}}{2^{\alpha_{3}}{\alpha_2}!{\alpha_{3}}!}}{\sum_{i\in\mathcal{C}_{1}}p^n_{i}\left(|\dbione|^{\Bar{r}(|\mathcal{C}_{1}|)}+|\dsijn|^{\Bar{r}(|\mathcal{C}_{1}|)/2}\right)}\to 0.
\end{align}
Subsequently, we denote
\begin{align*}
    \overline{M}_n=\max\{|\dbione|,|\dsijn|^{1/2}:i\in\mathcal{C}_1\}, \quad \Bar{p}_n=\max_{i\in\mathcal{C}_{1}}p^n_{i}.
\end{align*}
Since the sequence $p^n_i/\overline{p}_n$ is bounded, we can substitute it by its subsequence which admits a non-negative limit $s_i^2=\lim_{n\to\infty}p^n_{i}/\overline{p}_n$. Furthermore, as $p^n_{i}\geq \xi>0$ for all $i\in\mathcal{C}_{1}$, at least one among the limit $s_i^2$ is equal to $1$. Similarly, let $(\dbione)/\overline{M}_n\to t_{1i}$ and $(\dsione)/(2\overline{M}_n^2)\to t_{2i}$ as $n\to\infty$ for any $i\in\mathcal{C}_{1}$. Then, at least one among $t_{1i}$ and $t_{2i}$ for $i\in\mathcal{C}_{1}$ is equal to either $1$ or $-1$.

Then, we divide both the numerator and the denominator of the ratio in equation~\eqref{eq:limit_ratio} by $\overline{p}_n\overline{M}_n^{\ell}$, and obtain the following system of polynomial equations:
\begin{align*}
    \sum_{i\in\mathcal{C}_{1}}\sum_{\substack{\alpha_2+2\alpha_{3}=\ell}}\dfrac{s_i^2~t_{1i}^{\alpha_2}~t_{2i}^{\alpha_{3}}}{{\alpha_2}!~{\alpha_{3}}!}=0, \quad \forall \ell=1,2,\ldots,\brcone.
\end{align*}
It follows from the definition of $\Bar{r}(|\mathcal{C}_j|)$ that this system of polynomial equations will not admit any non-trivial solutions $(s_i,t_{1i},t_{2i})_{i\in\mathcal{C}_j}$, which is a contradiction to the fact that $s_i>0$ for all $i\in\mathcal{C}_{1}$.

Consequently, at least one among the coefficients in the representations of $A_{n,1}/D_{4n}$, $A_{n,2}/D_{4n}$, $B_n/D_{4n}$ and $C_n/D_{4n}$ does not go to zero as $n\to\infty$. Let us denote by $m_n$ the maximum of the absolute values of those aforementioned coefficients, i.e.
\begin{align*}
m_n=\max_{\substack{j\in[k_*], 0\leq|\alpha_1|\leq\Bar{r}(|\mathcal{C}_j|),\\ 0\leq\ell\leq 2(\brcj-|\alpha_1|)}}\left\{\dfrac{|E^n_{\alpha_1,\ell}(j)|}{D_{4n}}\right\}.
\end{align*}
Additionally, we define
\begin{align*}
    E^n_{\alpha_1,\ell}(j)/m_n\to\tau_{\alpha_1,\ell}(j)
\end{align*}
as $n\to\infty$ for all $j\in[k_*]$, $0\leq |\alpha_1|\leq\Bar{r}(|\mathcal{C}_j|)$, $0\leq\ell\leq 2(\brcj-|\alpha_1|)$. Here, at least one among $\tau_{\alpha_1,\ell}(j)$ is non-zero. By applying the Fatou's lemma, we get
\begin{align*}
    0=\lim_{n\to\infty}\frac{1}{m_n}\frac{2V(p_{\lambdan, G_n},p_{\lambdastar, G_*})}{D_{4n}}\geq\int\liminf_{n\to\infty}\frac{1}{m_n}\dfrac{|p_{\lambdan, G_n}(X,Y)-p_{\lambdastar, G_*}(X,Y)|}{D_{4n}}\dint(X,Y)\geq 0.
\end{align*}
Note that
\begin{align*}
    \frac{1}{m_n}\dfrac{p_{\lambdan, G_n}(X,Y)-p_{\lambdastar, G_*}(X,Y)}{D_{4n}}\to\sum_{j,\alpha_1,\ell}\tau_{\alpha_1,\ell}(j)X^{\alpha_1}\cdot\frac{\partial^{|\alpha_1|+\ell}f}{\partial h_1^{|\alpha_1|+\ell}}(Y|(\aj)^{\top}X+\bj,\sj)\Bar{f}(X).
\end{align*}
As a result, we get
\begin{align}
\label{eq:full_dependent_11}
\sum_{j,\alpha_1,\ell}\tau_{\alpha_1,\ell}(j)X^{\alpha_1}\cdot\frac{\partial^{|\alpha_1|+\ell}f}{\partial h_1^{|\alpha_1|+\ell}}(Y|(\aj)^{\top}X+\bj,\sj)\Bar{f}(X)=0,
\end{align}
By employing similar arguments for showing the set $\mathcal{H}_2$ is linearly independent as in Appendix~\ref{appendix:partial_dependent}, we can demonstrate that $\mathcal{H}_3$ defined in equation~\eqref{eq:set_H3} is also a linearly independent set. Thus, equation~\eqref{eq:full_dependent_11} indicates that
\begin{align*}
    \sum_{j,\alpha_1,\ell}\tau_{\alpha_1,\ell}(j)X^{\alpha_1}=0,
\end{align*}
for all $j\in[k_*]$ and $0\leq|\alpha_1|\leq\Bar{r}(|\mathcal{C}_j|)$ and $0\leq\ell\leq 2(\brcj-|\alpha_1|)$. As the left hand side of the above equation is a polynomial of $X\in\mathcal{X}$, which is a bounded set of $\mathbb{R}^d$. Then, $\tau_{\alpha_1,\ell}(j)=0$ for all $j\in[k_*]$, $0\leq |\alpha_1|\leq\Bar{r}(|\mathcal{C}_j|)$ and $0\leq\ell\leq 2(\brcj-|\alpha_1|)$. This is a contradiction to the fact that at least one among $\tau_{\alpha_1,\ell}(j)$ is different from 0. Therefore, we reach the local inequality in equation~\eqref{eq:full_dependent_local}, which means that there exists a positive constant $\varepsilon_0$ such that 
\begin{align*}
\inf_{\substack{\lambda\in[0,1],G\in\Ocal_{k,\xi}(\Theta):\\ D_4((\lambda, G),(\lambdastar,\Gstar))\leq\varepsilon_0}}V(p_{\lambda, G},p_{\lambdastar,\Gstar})/D_4((\lambda, G),(\lambdastar,\Gstar))>0.    
\end{align*}
\textbf{Global inequality.} Thus, it is sufficient to demonstrate that
\begin{align}
    \inf_{\substack{\lambda\in[0,1],G\in\Ocal_{k,\xi}(\Theta):\\ D_4((\lambda, G),(\lambdastar,\Gstar))>\varepsilon_0}}V(p_{\lambda, G},p_{\lambdastar,\Gstar})/D_4((\lambda, G),(\lambdastar,\Gstar))>0.
\end{align}
Suppose that the above inequality does not hold, then there exist sequences ${\lambda}^{\prime}_{n}\in[0,1]$ and $G^{\prime}_{n}\in\Ocal_{k,\xi}(\Theta)$ such that
\begin{align*}
    \begin{cases}
        D_4((\lambda'_n,G'_n),(\lambdastar,G_*))>\varepsilon_0\\
        V(p_{{\lambda}^{\prime}_{n}, G^{\prime}_{n}},p_{\lambdastar, G_*})/D_4((\lambda'_n,G'_n),(\lambdastar,G_*))\to 0,
    \end{cases}
\end{align*}
which implies that $V(p_{{\lambda}^{\prime}_{n}, G^{\prime}_{n}},p_{\lambdastar, G_*})\to 0$ as $n\to\infty$. Note that the sets $\Theta$ and $[0,1]$ are bounded, we can find a subsequence of ${G}'_n$ and a subsequence of ${\lambda}'_n$ such that ${G}'_n\to {G}'$ and ${\lambda}'_n\to{\lambda}'$, where ${G}'\in\Ocal_{k,\xi}(\Theta)$ and ${\lambda}'\in[0,1]$. By replacing ${G}'_n$ and ${\lambda}'_n$ with their subsequences, we get that $D_4(({\lambda}',{G}'),(\lambdastar, G_*))>\varepsilon_0$. By the Fatou's lemma, we obtain that
\begin{align*}
    0=\lim_{n\to\infty}2V(p_{{\lambda}^{\prime}_{n}, G^{\prime}_{n}},p_{\lambdastar, G_*})&\geq \int\liminf_{n\to\infty}\left|p_{{\lambda}'_n,{G}'_n}(X,Y)-p_{\lambdastar, G_*}(X,Y)\right|d(X,Y)\\
    &=\int\left|p_{{\lambda}',{G}'}(X,Y)-p_{\lambdastar, G_*}(X,Y)\right|d(X,Y)\geq 0,
\end{align*}
which indicates that $p_{{\lambda}',{G}'}(X,Y)=p_{\lambdastar, G_*}(X,Y)$ for almost surely $(X,Y)$.

\textbf{Case 1:} $\lambda'_n\in\mathcal{T}$ for infinitely $n\in\mathbb{N}$. WLOG, we assume that $\lambda'_n\in\mathcal{T}$ for all $n\in\mathbb{N}$. 

In this case, $D_4((\lambda', G'),(\lambdastar, G_*))=D_3(G',\overline{G}_*(\lambda'))>\varepsilon_0$. It follows from equation~\eqref{eq:full_dependent_99} that
\begin{align*}
    0=p_{{\lambda}',{G}'}(X,Y)-p_{\lambdastar, G_*}(X,Y)=\lambda'[p_{G'}(X,Y)-p_{\overline{G}_*(\lambda')}(X,Y)],
\end{align*}
Since $\liminf \lambda'_n$ is lower bounded by a positive constant, then $\lambda'>0$. Combining this with the above result, we get that $p_{G}(X,Y)=p_{\overline{G}_*(\lambda')}(X,Y)$ for almost surely $(X,Y)$. Due to the identifiability of the Gaussian mixture of experts \cite{ho2022gaussian}, we obtain that $G'=\overline{G}_*(\lambda')$. This means that $D_3(G',\overline{G}_*(\lambda'))=0$, which is a contradiction to the fact that $D_3(G',\overline{G}_*(\lambda'))>\varepsilon_0>0$.

\textbf{Case 2:} $\lambda'_n\in\mathcal{T}^c$ for infinitely $n\in\mathbb{N}$. WLOG, we assume that $\lambda'_n\in\mathcal{T}^c$ for all $n\in\mathbb{N}$. 

\textbf{Case 2.1:} $I_{\lambda'_n}$ is not ratio-independent

In this case, $D_4((\lambda', G'),(\lambdastar, G_*))=\sum_{j\in I_{\lambda'}}-\Bar{p}^*_j(\lambda')>\varepsilon_0$. It follows from equation~\eqref{eq:full_dependent_88} that
\begin{align*}
    0&=p_{\lambda', G'}(X,Y)-p_{\lambdastar, G_*}(X,Y)=\sum_{j\in I_{\lambda'}}-\Bar{p}^*_j(\lambda')f(Y|(\aj)^{\top}X+\bj,\sj)\Bar{f}(X)\\
    &\qquad+\Big[\sum_{i=1}^{k'}\lambda' p'_if(Y|(a'_i)^{\top}X+b'_i,\sigma'_i)\Bar{f}(X)-\sum_{j\in\mathcal{J}_{\lambda'}}\Bar{p}^*_j(\lambda')f(Y|(\aj)^{\top}X+\bj,\sj)\Bar{f}(X)\Big]\\
    &=\sum_{j\in I_{\lambda'}}-\Bar{p}^*_j(\lambda')f(Y|(\aj)^{\top}X+\bj,\sj)\Bar{f}(X)+[p_{G'}(X,Y)-p_{\widetilde{G}_*(\lambda')}(X,Y)]\\
    &=\sum_{j\in I_{\lambda'}}-\Bar{p}^*_j(\lambda')f(Y|(\aj)^{\top}X+\bj,\sj)\Bar{f}(X).
\end{align*}
Recall that $-\Bar{p}^*_j(\lambda')>0$ for all $j\in I_{\lambda'}$. Thus, $\Bar{p}^*_j(\lambda')= 0$ as $n\to\infty$ for all $j\in I_{\lambda'}$. This leads to the fact that $p^*_i/p^0_i=p^*_j/p^0_j=(\lambdastar-\lambda')/\lambdastar$ for all $i,j\in I_{\lambda'}$, which is a contradiction to the fact that $I_{\lambda'}$ is not ratio-independent, which follows from the ratio-independece of $I_{\lambda'_n}$.

\textbf{Case 2.2:} $I_{\lambda'_n}$ is ratio-independent.

In this case, $D_4((\lambda', G'),(\lambdastar, G_*))=D_3(G',\widetilde{G}_*(\lambda'))$. The result $p_{{\lambda}',{G}'}(X,Y)=p_{\lambdastar, G_*}(X,Y)$ for almost surely $(X,Y)$ indicates that $G'=\widetilde{G}_*(\lambda')$. Then, we have $D_3(G',\widetilde{G}_*(\lambda'_n))=0$, which contradicts to the fact that $D_3(G',\widetilde{G}_*(\lambda'_n))>\varepsilon_0>0$.

Hence, the proof is completed.

%%%%%%%%%%%%%%%%%%%%%%%%%%%%%%%%%%%%
\section{PARAMETER ESTIMATION WITH VANISHING MIXING PROPORTION}
\label{appendix:additional_results}
In this appendix, we resume the discussion about parameter estimation rates under the deviated Gaussian mixture of experts when the mixing proportion vanishes, that is, $\lambdastar=0$. For that purpose, we consider the distinguishable and non-distinguishable settings in Appendix~\ref{appendix:vanish_distinguishable} and Appendix~\ref{appendix:vanish_non_distinguishable}, respectively.
\subsection{Distinguishable Settings}
\label{appendix:vanish_distinguishable}
First of all, we explore the convergence behavior of parameter estimation under the distinguishable settings.
\begin{theorem}
    \label{theorem:distinguishable_vanishing}
  Assume that the distinguishability condition in Definition~\ref{definition:distinguishability} holds and $\lambdastar=0$. Then, the Total Variation lower bound $V(p_{\lambda,G},p_{\lambdastar,G_*})\gtrsim \lambda$ holds for any $(\lambda,G)\in[0,1]\times\mathcal{O}_k(\Theta)$. This bound together with Proposition~\ref{prop:mle_estimation} suggest that we can find a positive constant $C_{5}$ that depends only on $g_0,\lambdastar,\Theta$ such that
\begin{align*}
    \mathbb{P}(\hat{\lambda}_n>C_{5}\sqrt{\log (n)/n})\lesssim n^{-c_5},
\end{align*}
where $c_{5}$ is a positive constant depending only on $\Theta$.
\end{theorem}
When $\lambdastar=0$, the mixture part $p_{G_*}$ is no longer involved in the formulation of the true conditional density function $p_{\lambdastar,G_*}(Y|X)$. Moreover, since $p_{G_*}$ is distinguishable from the known function $g_0$, then we are not able to access the convergence behavior of the MLE $\widehat{G}_n$. Nevertheless, Theorem~\ref{theorem:distinguishable_vanishing} indicates that the mixing proportion estimation $\hat{\lambda}_n$ converges to $\lambdastar=0$ at a parametric rate of order $\mathcal{O}(n^{-1/2})$.
\begin{proof}[Proof of Theorem~\ref{theorem:distinguishable_vanishing}]
    From the result of Proposition~\ref{prop:mle_estimation}, it is sufficient to show that $V(p_{\hat{\lambda}_n, \widehat{G}_n},p_{\lambda^*, G_*})\gtrsim \hat{\lambda}_n$. When $\hat{\lambda}_n=0$, this problem becomes trivial. Therefore, we will consider only the case when $\hat{\lambda}_n>0$, in which the problem turns into proving that
\begin{align*}
    \inf_{\lambda\in(0,1],G\in\Ocal_{k,\xi}(\Theta)}\dfrac{V(p_{\lambda, G},p_{\lambda^* ,G_*})}{\lambda}>0.
\end{align*}
Assume that the above inequality does not hold, which implies that there exist sequences $\lambdan\in(0,1]$ and ${G}_n=\sum_{i=1}^{k_n}{p}^n_i\delta_{(\ain,\bin,\sigmain)}\in\Ocal_{k,\xi}(\Theta)$ such that $V(p_{{\lambda}_n,{G}_n},p_{\lambda^*,G_*})/{\lambda}_n\to 0$ as $n\to\infty$. Since $\Theta$ is a compact set, we can find a subsequence of $G_n$ such that $G_n\to \widetilde{G}$, where  $\widetilde{G}:=\sum_{i=1}^{\widetilde{k}}\widetilde{p}_i\delta_{(\widetilde{a}_{i},\widetilde{b}_{i},\widetilde{\sigma}_{i})}\in\Ocal_{k,\xi}(\Theta)$. By replacing $G_n$ with this subsequence and applying the Fatou's lemma with a note that $\lambdastar=0$, we get 
\begin{align*}
    \lim_{n\to\infty}\dfrac{2V(p_{{\lambda}_n,{G}_n},p_{\lambda^*,G_*})}{{\lambda}_n}&\geq\int\liminf_{n\to\infty}\left|\sum_{i=1}^{k_n}{p}^n_if(Y|(\ain)^{\top}X+\bin,\sigmain)-g_0(Y|X)\right|\Bar{f}(X)d(X,Y).
\end{align*}
It follows from the hypothesis $V(p_{{\lambda}_n,{G}_n},p_{\lambda^*,G_*})/{\lambda}_n\to 0$ that $ \sum_{i=1}^{\widetilde{k}}\widetilde{p}_if(Y|(\widetilde{a}_{i})^{\top}X+\widetilde{b}_{i},\widetilde{\sigma}_{i})-g_0(Y|X)=0$,
for almost surely $(X,Y)$. This contradicts the assumption that  $p_{\widetilde{G}}$ is distinguishable from $g_0$. Hence, we reach the conclusion of this part.
\end{proof}

\subsection{Non-distinguishable Settings}
\label{appendix:vanish_non_distinguishable}
We now draw our attention to parameter estimation rates under the non-distinguishable settings when the mixing proportion vanishes, namely when the function $g_0$ takes the following form:
\begin{align}
    \label{eq:new_new_g0}
    g_0(Y|X)=p_{G_0}(Y|X):=\sum_{j=1}^{k_0}p^0_jf(Y|(a^0_j)^{\top}X+b^0_j,\sigma^0_j),
\end{align}
where $k_0\in[k_*]$.

Since $\lambdastar=0$, the mixture part $p_{G_*}$ is not involved in the formulation of the true conditional density $p_{\lambdastar,G_*}(Y|X)$. As a result, we do not have any interaction between two functions $g_0(Y|X)$ and $p_{G_*}(Y|X)$. Therefore, it is unnecessary to divide the non-distinguishable setting into partial overlap regime and full overlap regime. Instead, we establish the parameter estimations under the general non-distinguishable settings in the following theorem:
\begin{theorem}
\label{theorem:dependent_lambda_0}
Suppose that the function $g_0$ takes the form in equation~\eqref{eq:new_new_g0} and $\lambdastar=0$.  
Then, the Total Variation lower bound $V(p_{\lambda,G},p_{\lambdastar,G_*})\gtrsim \lambda D_3(G,G_0)$ holds for any $(\lambda,G)\in[0,1]\times\mathcal{O}_k(\Theta)$. This bound together with Proposition~\ref{prop:mle_estimation} indicates that there exists a positive constants $C_{6}$ depending on $g_0,\lambdastar,\Theta$ such that
\begin{align*}
    \mathbb{P}(\hat{\lambda}_nD_3(\widehat{G}_n,G_0)>C_{6}\sqrt{\log (n)/n})\lesssim n^{-c_6},
\end{align*}
where $c_{6}$ is a constant that depends only on $\Theta$.
\end{theorem}
Different from the results of all previous theorems, the MLE $\widehat{G}_n$ converges to the mixing measure $G_0$ rather than $G_*$ under the loss function $D_3$ due to the disappearance of the mixture part $p_{G_*}(Y|X)$ in the conditional density $p_{\lambdastar,G_*}(Y|X)$. Moreover, the rate of that convergence depends on the vanishing rate of $\hat{\lambda}_n$, therefore, it is no better than the parametric rate of order $\mathcal{O}(n^{-1/2})$.
\begin{proof}[Proof of Theorem~\ref{theorem:dependent_lambda_0}]
    Note that the problem is trivial when $\hat{\lambda}_n=0$, therefore, we consider only the case when $\hat{\lambda}_n > 0$. From Proposition~\ref{prop:mle_estimation}, it is sufficient to show that 
    \begin{align}
        \label{eq:original_to_prove}
        \inf_{\lambda\in(0,1], G\in\Ocal_{k,\xi}(\Theta)}\dfrac{V(p_{\lambda, G},p_{\lambdastar, G_*})}{\lambda D_3(G,G_0)}>0.
    \end{align}
    Since $\lambdastar=0$, we get that 
    \begin{align*}
        p_{\lambda, G}(X,Y)-p_{\lambdastar, G_*}(X,Y)&=(1-\lambda)g_0(Y|X)\Bar{f}(X)+\lambda p_{G}(X,Y)-g_0(Y|X)\Bar{f}(X)\\
        &=\lambda~[p_{G}(X,Y)-g_0(Y|X)\Bar{f}(X)]\\
        &=\lambda~[p_{G}(X,Y)-p_{G_0}(X,Y)].
    \end{align*}
    As a result, equation~\eqref{eq:original_to_prove} becomes $\inf_{G\in\Ocal_{k,\xi}(\Theta)}{V(p_{ G},p_{ G_0})}/{ D_3(G,G_0)}>0$.

    \textbf{Local inequality}: We first prove that
    \begin{align*}
        \lim_{\varepsilon\to 0}\inf_{\substack{G\in\Ocal_{k,\xi}(\Theta):D_3(G,G_0)\leq\varepsilon}}\dfrac{V(p_{ G},p_{ G_0})}{ D_3(G,G_0)}>0.    
    \end{align*}
    Assume by contrary that the above claim is not true. Then, there exists a sequence $G_n=\sum_{i=1}^{k_n}p^n_i\delta_{(\ain,\bin,\sigmain)}\in\Ocal_{k,\xi}(\Theta)$ such that as $n\to\infty$, we have
    \begin{align*}
        \begin{cases}
            D_3(G_n,G_0)\to 0,\\
            V(p_{G_n},p_{G_0})/D_3(G_n,G_0)\to 0.
        \end{cases}
    \end{align*}
    By employing arguments (with adapted notations) used in Case 2.2 in Appendix~\ref{appendix:partial_dependent} for showing contradiction to the fact that $V(p_{G_n},p_{\overline{G}_*(\lambdan)})\to 0$ as $n\to\infty$, we also get a contradiction here. Consequently, there exists a positive constant $\varepsilon_0$ such that
    \begin{align*}
        \inf_{\substack{G\in\Ocal_{k,\xi}(\Theta):D_3(G,G_0)\leq\varepsilon_0}}\dfrac{V(p_{ G},p_{ G_0})}{ D_3(G,G_0)}>0.
    \end{align*}
    \textbf{Global inequality}: From the above result, we only need to show that 
    \begin{align*}
        \inf_{\substack{G\in\Ocal_{k,\xi}(\Theta):D_3(G,G_0)>\varepsilon_0}}\dfrac{V(p_{ G},p_{ G_0})}{ D_3(G,G_0)}>0.
    \end{align*}
    Assume that the above inequality does not hold. Then, there exists a sequence $G'_n\in\Ocal_{k,\xi}(\Theta)$ satisfying $V(p_{G'_n},p_{G_0})/D_3(G'_n,G_0)\to 0$ as $n\to\infty$, whereas $D_3(G'_n,G_0)>\varepsilon_0$ for all $n\in\mathbb{N}$. Therefore, $V(p_{G'_n},p_{G_0})\to 0$ as $n\to\infty$.
    Note that $\Theta$ is a compact set, then we can find a subsequence of $G'_n$ such that $G'_n\to G'$ for some $G'\in\Ocal_{k,\xi}(\Theta)$. By replacing the sequence $G'_n$ by that subsequence, we obtain that $D_3(G',G_0)>\varepsilon_0$ as a result of $D_3(G'_n,G_0)>\varepsilon_0$ for all $n\in\mathbb{N}$. By Fatou's lemma, we get
    \begin{align*}
        0=\lim_{n\to\infty}V(p_{G'_n},p_{G_0})&\geq \dfrac{1}{2}\int\liminf_{n\to\infty}|p_{G'_n}(X,Y)-p_{G_0}(X,Y)|d(X,Y)\\
        &=\dfrac{1}{2}\int|p_{G'}(X,Y)-p_{G_0}(X,Y)|d(X,Y)\geq 0,
    \end{align*}
    which implies that $p_{G'}(X,Y)=p_{G_0}(X,Y)$ for almost surely $(X,Y)\in\mathcal{X}\times\mathcal{Y}$. Since the Gaussian mixture of experts is identifiable (cf. Proposition 3 in \cite{ho2022gaussian}), the previous equation indicates that $G'\equiv G_0$. This contradicts to the fact that $D_3(G',G_0)\geq\varepsilon>0$.

    Hence, the proof is completed.
\end{proof}

\end{document}